\documentclass{article}

\PassOptionsToPackage{sort&compress, numbers}{natbib}

\usepackage[preprint]{neurips_2026}

\usepackage[utf8]{inputenc}
\usepackage[T1]{fontenc}
\usepackage{hyperref}
\usepackage{url}
\usepackage{xurl}
\usepackage{booktabs}
\usepackage{amsfonts}
\usepackage{amsmath}
\usepackage{amssymb}
\usepackage{amsthm}
\usepackage{mathtools}
\usepackage{nicefrac}
\usepackage{microtype}
\usepackage[dvipsnames]{xcolor}
\usepackage{graphicx}
\usepackage{subcaption}
\usepackage{caption}
\usepackage{wrapfig}
\usepackage{float}
\usepackage{multirow}
\usepackage{makecell}
\usepackage{array,ragged2e}
\newcolumntype{P}[1]{>{\RaggedRight\arraybackslash}p{#1}}
\usepackage{tabularx}
\usepackage{longtable}
\usepackage{colortbl}
\usepackage{soul}
\usepackage{enumitem}
\usepackage{pifont}
\usepackage{bbding}
\usepackage{xspace}
\usepackage{afterpage}
\usepackage{fontawesome5}
\usepackage{listings}
\usepackage{algorithm}
\usepackage{algorithmic}
\usepackage{placeins}

\usepackage{tikz}
\usetikzlibrary{positioning, arrows.meta, calc, shapes.geometric, fit, backgrounds, decorations.pathreplacing}
\usepackage[most,skins,theorems]{tcolorbox}

\newtheorem{theorem}{Theorem}
\newtheorem{lemma}{Lemma}

\newtheorem{corollary}{Corollary}
\theoremstyle{definition}
\newtheorem{definition}{Definition}
\newtheorem{assumption}{Assumption}
\theoremstyle{remark}
\newtheorem{remark}{Remark}

\definecolor{darkblue}{rgb}{0, 0, 0.5}
\hypersetup{colorlinks=true, citecolor=darkblue, linkcolor=darkblue, urlcolor=darkblue}

\usepackage[capitalise,nameinlink]{cleveref}
\crefname{theorem}{Theorem}{Theorems}
\crefname{lemma}{Lemma}{Lemmas}
\crefname{proposition}{Proposition}{Propositions}
\crefname{corollary}{Corollary}{Corollaries}
\crefname{definition}{Definition}{Definitions}
\crefname{assumption}{Assumption}{Assumptions}
\crefname{remark}{Remark}{Remarks}
\crefname{claim}{Claim}{Claims}
\crefname{equation}{Eq.}{Eqs.}
\crefname{figure}{Figure}{Figures}
\crefname{table}{Table}{Tables}
\crefname{algorithm}{Algorithm}{Algorithms}
\crefname{section}{Section}{Sections}
\crefname{subsection}{Section}{Sections}
\crefname{subsubsection}{Section}{Sections}

\definecolor{lightorange}{HTML}{faa755}
\definecolor{lightblue}{RGB}{220,235,250}
\tcbset{
  takeawaysbox/.style={
    title=Takeaways,
    colback=lightblue!80,
    colframe=black,
    fonttitle=\bfseries\small,
    coltitle=white,
    colbacktitle=black,
    enhanced,
    attach boxed title to top left={xshift=2.5mm,yshift=-2.5mm},
    boxed title style={rounded corners, size=small, colframe=black, colback=black},
    width=\linewidth,
    arc=3.5mm
  }
}

\title{OPRD: On-Policy Representation Distillation}

\author{%
  Shenzhi Yang$^{1,2}$\thanks{Work in progress.} \quad
  Guangcheng Zhu$^{1,2}$ \quad
  Bowen Song$^{2}$ \quad
  Haobo Wang$^{1}$\thanks{Corresponding author.} \\[0.3em]
  \textbf{ Mingxuan Xia$^{1}$ \quad
  Xing Zheng$^{2}$ \quad
  Yingfan Ma$^{2}$ \quad 
  Zhongqi Chen$^{2}$ \quad } \\[0.3em]
  \textbf{ Weiqiang Wang$^{2}$ \quad
  Junbo Zhao$^{1,2}$ \quad
  Gang Chen$^{1}$} \\[0.6em]
  $^{1}$Zhejiang University \quad $^{2}$Ant Group
}

\begin{document}

\maketitle


\begin{abstract}
On-policy distillation (OPD) supervises the student exclusively in the \emph{output space} by matching next-token distributions.
This paradigm suffers from two limitations: (i)~a high-variance gradient estimator whose signal-to-noise ratio collapses as the student approaches the teacher, and (ii)~an LM-head information bottleneck that discards the teacher's intermediate hidden states.
We propose \textbf{On-Policy Representation Distillation (OPRD)}, the first method to lift on-policy distillation into the \emph{hidden-state space}.
OPRD aligns student and teacher representations across selected layers on the same on-policy rollouts, providing dense, deterministic, per-layer supervision while bypassing the LM head entirely.
Theoretically, OPRD provides a deterministic per-sample gradient, removing the token-level estimation variance that plagues OPD, and exposes structural information that any output-space objective necessarily discards.
Empirically, OPRD closes the student--teacher gap on competition mathematics benchmarks (AIME 2024, AIME 2025, AIMO) where every output-space baseline plateaus below the teacher, while training $1.44\times$ faster and using up to $54\%$ less memory.
We further extend OPRD to the \emph{cross-architecture} setting via \textbf{OPRD-Bridge}: by exploiting the observation that heterogeneous models share a low-rank representational structure, we construct a frozen projector pair that aligns representations across arbitrary depth/width mismatches, shifting the alignment from the output space (which depends on a shared vocabulary) to the representation space.
We validate OPRD-Bridge on both cross-architecture (Qwen3-4B $\to$ Qwen3-1.7B-Base) and cross-tokenizer (Phi-4-mini-reasoning $\to$ Qwen3-1.7B-Base) settings, demonstrating successful knowledge transfer even when the vocabulary-based alignment channel is unavailable. The code is available via \url{https://github.com/ShenzhiYang2000/OPRD}.
\end{abstract}

\section{Introduction}
\label{sec:intro}


On-policy distillation (OPD) has become a central building block in large language model (LLM) post-training.
By letting the student sample its own responses and then scoring each token against the teacher's conditional distribution, OPD provides a dense, token-level training signal that adapts to the student's current policy, avoiding the exposure bias inherent in training on static teacher outputs~\citep{bengio2015scheduled}.
Multiple production systems now rely on OPD as a primary post-training stage~\citep{yang2025qwen3,xiao2026mimo,zeng2026glm,deepseek2026deepseek}, positioning it alongside supervised fine-tuning and outcome-reward reinforcement learning.

Despite this momentum, the design space of OPD has remained surprisingly narrow.
Every variant proposed to date (sampled-token~\citep{xiao2026mimo,yang2026learning}, full-vocabulary, and top-$k$) differs only in \emph{how many output tokens} are evaluated per position, yet they all operate inside the same \emph{output space}: the divergence is computed over next-token probability distributions $p_t$ and $q_t$.
We argue that this output-only paradigm imposes two practical limitations that become increasingly damaging as training progresses.

\begin{figure}[!t]
  \centering
  \includegraphics[width=0.9\linewidth]{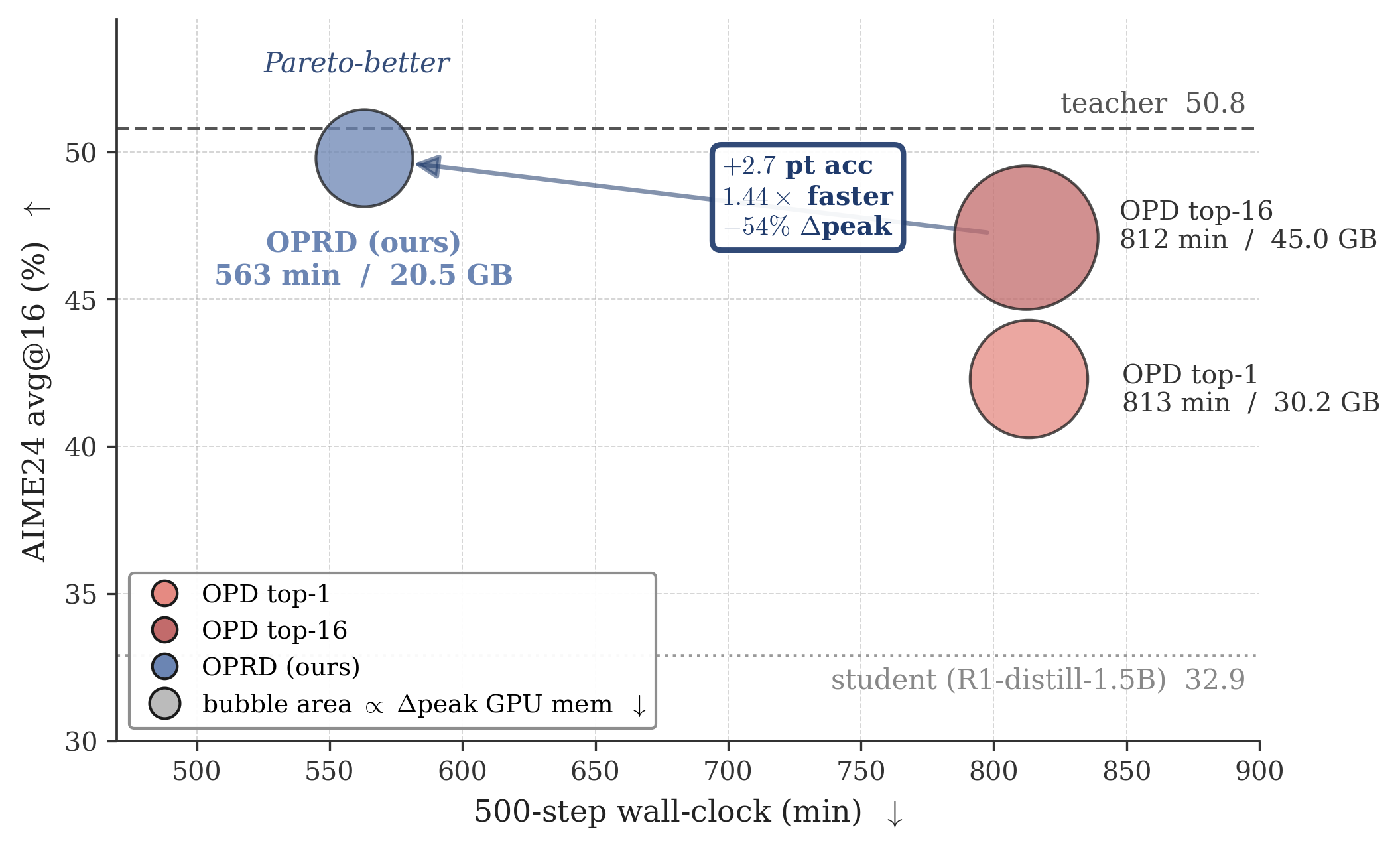}
  \caption{\textbf{OPRD is strictly Pareto-dominant on accuracy, training time, and GPU memory.}
    Each bubble is a method trained from the same R1-distill-1.5B student against JustRL-1.5B teacher for $500$ optimizer steps on $8\!\times\!$ A100 GPU (80G) FSDP (\S\ref{sec:experiments}).
    Axes carry the two ``compute'' costs (wall-clock $\downarrow$, AIME24 Avg@16 $\uparrow$); bubble area encodes the third cost (actor-update $\Delta$peak GPU memory $\downarrow$).
    OPRD (navy bubble) simultaneously dominates the strongest output-space baseline (OPD top-$16$) on all three axes: $2.7$\,pt accuracy gain, $1.44\times$ speed-up, and $54\%$ $\Delta$peak memory cut. Results on AIME25 and AIMO are qualitatively identical (\S\ref{sec:exp_main}).}
  \label{fig:teaser}
\end{figure}

\indent \textbf{Limitation 1: Variance dominates the late-stage signal.}\label{lim:variance}
Sampled-token OPD estimates each token-level reverse KL from a \emph{single} sample drawn from a vocabulary of size $|\mathcal{V}|$ (e.g., $\approx\!151$K for Qwen3).
The estimator is unbiased, but its variance shrinks much more slowly than the signal: letting $\delta \triangleq \|p_t - q_t\|$ measure the student--teacher gap, the expected gradient decays as $O(\delta^2)$ while the variance decays only as $O(\delta)$, so the signal-to-noise ratio collapses and training plateaus well below the teacher (\Cref{fig:training_curves}).
Top-$k$ OPD partially mitigates this by evaluating $k$ tokens per position, but trades sampling variance for truncation bias and still plateaus empirically.

\indent \textbf{Limitation 2: The output layer is an information bottleneck.}\label{lim:compute}
Every output-space variant treats the teacher as a \emph{black-box probability oracle}, querying only the LM-head output.
Yet the teacher has computed, at every position, an entire stack of $d$-dimensional hidden states encoding attention patterns, mid-layer reasoning state, and geometric structure.
Nearly all of this is destroyed by the LM-head projection $W_{\mathrm{head}}\!:\mathbb{R}^d \!\to\! \mathbb{R}^{|\mathcal{V}|}$ and the subsequent softmax: output distributions that agree to within an arbitrary tolerance can correspond to hidden states differing along entire affine subspaces of $\mathbb{R}^d$.
The student is therefore graded only on what survives this projection, and receives no signal about \emph{how} the teacher arrived at that distribution.
This is particularly wasteful in the on-policy regime, where the teacher's hidden states are already computed on every rollout but discarded before they reach the loss.

To overcome both limitations, we propose \textbf{On-Policy Representation Distillation (OPRD)}, the first method to lift on-policy distillation from the output space into the \emph{hidden-state space}.
On the same on-policy rollouts $(x, \hat{y})$ already used by standard OPD, OPRD aligns the student's intermediate hidden representations with the teacher's across selected transformer layers and response positions via a normalized mean-squared error objective.
A single design choice (supervising at the representation level rather than at the output level) simultaneously addresses both limitations.
\textbf{First, deterministic, low-variance gradients:} OPRD's MSE objective is a deterministic function of the rollout; its gradient carries \emph{zero} additional sampling variance, eliminating the late-stage signal-to-noise collapse of OPD by construction.
\textbf{Second, a richer supervision channel beyond logits:} OPRD taps the teacher at any subset of its $L$ intermediate layers, exposing (layers $\times$ positions $\times$ hidden-dim) scalars of structural supervision per sample, orders of magnitude more than the signal extracted at the output. The student is graded on the same intermediate representations the teacher actually computed, without filtering through the LM-head projection.
Both properties follow from a single conceptual shift: moving the supervision target from the output of the LM head to its input.
Because the loss path never materialises the $[B, T, |\mathcal{V}|]$ logits tensor, OPRD also reduces wall-clock time and peak GPU memory as a direct consequence of this design.
OPRD is a self-contained training objective that can be used on its own; it also composes additively with any OPD variant at essentially zero infrastructure cost.
Beyond the standard teacher--student setting studied here, we highlight two high-value scenarios where OPRD's advantages are especially pronounced.
\textbf{(1)~Multi-model RL merging.}
State-of-the-art RL pipelines increasingly merge multiple teacher or reward-model checkpoints into a single student.
In this setting, full-vocabulary OPD requires materialising a $[B, T, |\mathcal{V}|]$ logit tensor \emph{per teacher}, quickly exhausting GPU memory and demanding heavy infrastructure work~\citep{deepseek2026deepseek}.
Top-$k$ OPD alleviates memory but introduces a truncation bias (tail tokens are ignored) and still plateaus below the teacher empirically.
OPRD sidesteps both: its hidden-state loss never touches the vocabulary dimension, so memory and wall-clock scale with $d$ rather than $|\mathcal{V}|$, while its deterministic gradient avoids the variance trap entirely.
\textbf{(2)~On-policy self-distillation (OPSD).}
A growing line of work constructs the teacher from the student itself by injecting privileged information (e.g., ground-truth solutions, step-level verification signals) into the prompt.
Because teacher and student share exactly the same weights, the same-architecture requirement is satisfied by construction, and the hidden-state alignment signal is maximally informative.
OPRD can therefore serve as a drop-in replacement for the output-space reverse-KL in any OPSD pipeline, delivering lower variance and lower cost without any architectural modification.
We discuss both applications in detail in \S\ref{sec:discussion}.
The strict Pareto improvement over all output-space baselines is summarized in \Cref{fig:teaser}.

The above formulation, which we call \textbf{OPRD-Vanilla}, however, requires teacher and student to share the \emph{same} architecture: identical depth, identical hidden dimension, and compatible initialization.
The moment these conditions are violated, the method breaks down, since hidden states of different dimensionality cannot be directly compared, and even when dimensions happen to match across different architectures, the cosine similarity between corresponding hidden states is empirically zero.
This same-architecture constraint excludes precisely the most common distillation scenario, namely distilling a large, expensive teacher into a smaller, deployable student, which inherently involves heterogeneous model pairs.

A useful perspective on this problem comes from asking why \emph{output-space} distillation does not suffer from the same limitation.
The answer is that logit distillation possesses a natural \emph{bridge}: the shared vocabulary.
Both teacher and student project their hidden states through their respective LM heads into the same probability simplex, making outputs directly comparable regardless of internal architecture.
The vocabulary serves as a fixed, pre-established interface that neither model modifies during training, and knowledge flows through it.
This bridge works, but it has a critical limitation of its own: it requires both models to share the same vocabulary.
Moreover, it operates at only a single point in the network (the final layer) and is lossy, compressing all representational structure into a single next-token distribution.
The natural question, then, is whether one can construct an analogous bridge for \emph{representations}, one that carries richer information, operates at every layer, and does not depend on a shared vocabulary.

The Platonic Representation Hypothesis~\citep{huh2024platonic} provides theoretical grounds to believe such a bridge should exist.
It posits that strong neural networks, regardless of architecture, are converging toward a shared statistical model of reality in their representation spaces.
Empirically, models with higher downstream performance exhibit higher mutual representation alignment~\citep{kornblith2019similarity}, and a single linear projection suffices to bridge representations across architectures and even across modalities~\citep{moschella2022relative,merullo2022linearly}.
We confirm this directly: pairing a Qwen3-4B teacher (36 layers, $d\!=\!2560$) with a Qwen3-1.7B-Base student (28 layers, $d\!=\!2048$), a rank-8 linear subspace achieves \textbf{95\% cosine similarity} between the projected representations of the two models, despite their different depths and widths.
Notably, this similarity \emph{decreases} with higher rank, indicating that the shared structure concentrates in very few principal directions and that additional dimensions introduce noise rather than signal.

This finding leads directly to \textbf{OPRD-Bridge}.
We construct a pair of linear projectors, one per model, that map their respective hidden states into a shared low-rank subspace.
The teacher's projector is obtained via PCA of its hidden-state covariance (extracting the principal variation directions); the student's projector is trained to align with it.
Once both projectors converge, they are \emph{frozen} and serve as a fixed bridge through which representational structure flows from teacher to student during the main distillation phase, which optimizes only the student backbone.
The design mirrors logit distillation: just as the vocabulary is a frozen interface through which output-level knowledge flows, the projector pair $(P_T, P_S)$ is a frozen interface through which representation-level knowledge flows, but at every layer rather than only the last, and decoupled from the vocabulary.
We validate this capability empirically by distilling across completely disjoint tokenizers (Phi-4-mini-reasoning $\to$ Qwen3-1.7B), a setting where the vocabulary-based alignment channel is unavailable.

Our main contributions are as follows:
\begin{enumerate}[leftmargin=16pt, itemsep=2pt]
    \item \textbf{OPRD-Vanilla: representation-level on-policy distillation.} We propose On-Policy Representation Distillation, the first method to lift on-policy distillation from the output space into the hidden-state space. We provide a two-perspective theoretical analysis showing that OPRD (i)~yields a deterministic per-sample gradient, removing the token-level estimation variance of OPD's gradient estimator, and (ii)~exposes per-layer structural information that the LM-head projection necessarily discards.

    \item \textbf{OPRD-Bridge: cross-architecture extension.} We show that heterogeneous teacher--student pairs share a low-rank representational structure, and exploit this to construct a frozen projector pair that enables representation distillation across arbitrary architecture pairs (different depth, width, and potentially vocabulary), the first such method in the on-policy regime.

    \item \textbf{Empirical validation.} On competition mathematics benchmarks (AIME 2024, AIME 2025, AIMO), OPRD-Vanilla closes the student--teacher gap where every output-space baseline plateaus, while training $1.44\times$ faster and using up to $54\%$ less memory.
\end{enumerate}

\section{Background and Problem Setup}
\label{sec:preliminaries}

This section formalizes the on-policy distillation problem we build upon.
We introduce the necessary notation in \S\ref{sec:notation}, define the on-policy distillation framework in \S\ref{sec:opd_framework}, and catalogue the three output-space supervision granularities used in prior work in \S\ref{sec:opd_variants}.
We close in \S\ref{sec:opd_limitations} by isolating the common structural property of these variants that motivates our hidden-state approach in the next section.

\subsection{Notation}
\label{sec:notation}

We consider two autoregressive language models with a shared vocabulary $\mathcal{V}$: a \emph{student} $\pi_\theta$ with trainable parameters $\theta$, and a fixed \emph{teacher} $\pi_T$.
A training instance is a prompt $x = (x_1, \ldots, x_n)$ drawn from a prompt distribution $\mathcal{D}_x = \{x^{(i)}\}_{i=1}^N$; a model response is a token sequence $y = (y_1, \ldots, y_m)$ produced autoregressively.
For brevity we write the prefix up to step $t$ as $y_{<t} \triangleq (y_1, \ldots, y_{t-1})$, and use $\pi(\cdot \mid x, y_{<t})$ to denote either model's next-token distribution over $\mathcal{V}$ conditioned on $(x, y_{<t})$.
The notation $y \sim \pi_\theta(\cdot \mid x)$ refers to an autoregressive sample drawn from the student.

Both models follow the standard transformer template: a stack of self-attention blocks producing intermediate hidden states $h^{(l)} \in \mathbb{R}^{d}$ at each layer $l$ and position, followed by a language-model head $W_{\mathrm{head}} \in \mathbb{R}^{|\mathcal{V}| \times d}$ that maps the final hidden state to logits.
We use $L$ and $d$ when the two models share the same depth and width; when they differ (\S\ref{sec:oprd_bridge}), we write $L_S, L_T$ and $d_S, d_T$ explicitly.
We write $h_{\theta,t}^{(l)}$ and $h_{T,t}^{(l)}$ for the student and teacher hidden states at layer $l$ and response position $t$, with both networks evaluated on the same input sequence.

\subsection{The On-Policy Distillation Framework}
\label{sec:opd_framework}

\paragraph{Setup.}
On-policy distillation (OPD) departs from classical knowledge distillation by drawing the supervision distribution from the student rather than from a fixed dataset.
Concretely, at each training step the student first samples a response
$\hat{y} = (\hat{y}_1, \ldots, \hat{y}_T) \sim \pi_\theta(\cdot \mid x)$
of length $T \triangleq |\hat{y}|$, after which both models are evaluated on the student-generated prefixes.
For each position $t \in \{1, \ldots, T\}$ this yields a pair of next-token distributions over $\mathcal{V}$:
\begin{equation}
  p_t(v) \;\triangleq\; \pi_\theta(v \mid x, \hat{y}_{<t}),
  \qquad
  q_t(v) \;\triangleq\; \pi_T(v \mid x, \hat{y}_{<t}),
  \qquad v \in \mathcal{V}.
\end{equation}
The defining feature of OPD is that the teacher is queried on \emph{student-visited states}, namely prefixes that arise from the current policy, rather than on canonical teacher trajectories.
This eliminates the exposure-bias gap between training and inference distributions that plagues fixed-target distillation.

\paragraph{Objective.}
The canonical OPD objective minimizes the trajectory-level reverse KL divergence between the student and teacher policies on student rollouts.
By the chain rule for KL divergence, this trajectory-level quantity decomposes exactly into a sum of token-level reverse KL terms:
\begin{equation}
  \mathcal{L}_{\mathrm{OPD}}(\theta)
  \;=\;
  \mathbb{E}_{x \sim \mathcal{D}_x,\;
    \hat{y} \sim \pi_\theta(\cdot \mid x)}
  \!\left[
    \sum_{t=1}^{T} D_{\mathrm{KL}}(p_t \,\|\, q_t)
  \right],
  \label{eq:opd_exact_token_decomp}
\end{equation}
where the token-level reverse KL at position $t$ is
$D_{\mathrm{KL}}(p_t \,\|\, q_t) = \sum_{v \in \mathcal{V}} p_t(v) \log [p_t(v)/q_t(v)]$.
Eq.~\eqref{eq:opd_exact_token_decomp} is conceptually clean but computationally inconvenient: it requires summing over the full vocabulary $\mathcal{V}$ at every position, which is prohibitive for modern LLMs with $|\mathcal{V}|$ in the hundreds of thousands.
Practical implementations differ in how they approximate this sum, and we review the three dominant choices below.

\subsection{Three Output-Space Variants}
\label{sec:opd_variants}

We use a unified template to describe each variant: at each position $t$, define a token subset $S_t \subseteq \mathcal{V}$ and a per-position loss $\ell_t$ that depends only on $\{p_t(v), q_t(v) : v \in S_t\}$.
The three variants below correspond to different choices of $S_t$.

\paragraph{(a) Sampled-token OPD ($S_t = \{\hat{y}_t\}$).}
The most lightweight and by far the most widely adopted choice in production deployments~\citep{xiao2026mimo,yang2026learning}.
A single token $\hat{y}_t \sim p_t$ already drawn during rollout is reused as the supervision target, and the per-position loss takes the form of a log-ratio:
\begin{equation}
  \ell_t^{\mathrm{sample}}
  \;\triangleq\;
  \log p_t(\hat{y}_t) - \log q_t(\hat{y}_t),
  \qquad
  \mathcal{L}_{\mathrm{OPD}}^{\mathrm{sample}}(\theta)
  =
  \mathbb{E}_{x, \hat{y}}\!\left[\sum_{t=1}^{T} \ell_t^{\mathrm{sample}}\right].
  \label{eq:opd_sampled_obj}
\end{equation}
A straightforward calculation gives
$\mathbb{E}_{\hat{y}_t \sim p_t}[\ell_t^{\mathrm{sample}}] = D_{\mathrm{KL}}(p_t \,\|\, q_t)$,
so $\ell_t^{\mathrm{sample}}$ is an unbiased \emph{single-sample} estimator of the token-level reverse KL.
Memory cost is $O(BT)$ for batch size $B$ and response length $T$; teacher queries amount to one log-probability per token.

\paragraph{(b) Full-vocabulary OPD ($S_t = \mathcal{V}$).}
At the opposite extreme, one materializes the entire teacher distribution and computes the exact token-level KL at every position:
\begin{equation}
  \mathcal{L}_{\mathrm{OPD}}^{\mathrm{full}}(\theta)
  =
  \mathbb{E}_{x, \hat{y}}\!\left[
    \sum_{t=1}^{T}
    \sum_{v \in \mathcal{V}} p_t(v)\,\log\!\tfrac{p_t(v)}{q_t(v)}
  \right].
  \label{eq:opd_full_obj}
\end{equation}
The gradient signal is the densest possible, but the price is steep: storing teacher logits demands $O(BT|\mathcal{V}|)$ memory, which becomes infeasible for long-context training at modern vocabulary sizes. 

\paragraph{(c) Top-$k$ OPD ($S_t = \mathrm{TopK}(p_t, k)$).}
Top-$k$ OPD interpolates between the two extremes by restricting attention to the $k$ tokens that the \emph{student} ranks highest at position $t$, then computing a KL between renormalized distributions on this support:
\begin{equation}
  \mathcal{L}_{\mathrm{OPD}}^{\text{top-}k}(\theta)
  =
  \mathbb{E}_{x, \hat{y}}\!\left[
    \sum_{t=1}^{T}
    D_{\mathrm{KL}}\!\left(
      \bar{p}_t^{(S_t)} \,\big\|\, \bar{q}_t^{(S_t)}
    \right)
  \right],
  \quad
  \bar{p}_t^{(S_t)}(v) = \frac{p_t(v)\,\mathbf{1}[v \in S_t]}{\sum_{u \in S_t} p_t(u)},
  \label{eq:opd_topk_obj}
\end{equation}
and analogously for $\bar{q}_t^{(S_t)}$.
The hyperparameter $k$ trades supervision density against teacher-query cost, with $k=1$ recovering (a deterministic version of) sampled-token OPD and $k = |\mathcal{V}|$ recovering full-vocabulary OPD.
Typical implementations use $k \in [4, 64]$. We measure this cost empirically in \S\ref{sec:exp_efficiency}, where top-$16$ OPD's actor-update transient memory is more than $2\times$ larger than OPRD's at the same setting.

\begin{figure}[!t]       
    \centering              
    \includegraphics[width=1.0\textwidth]{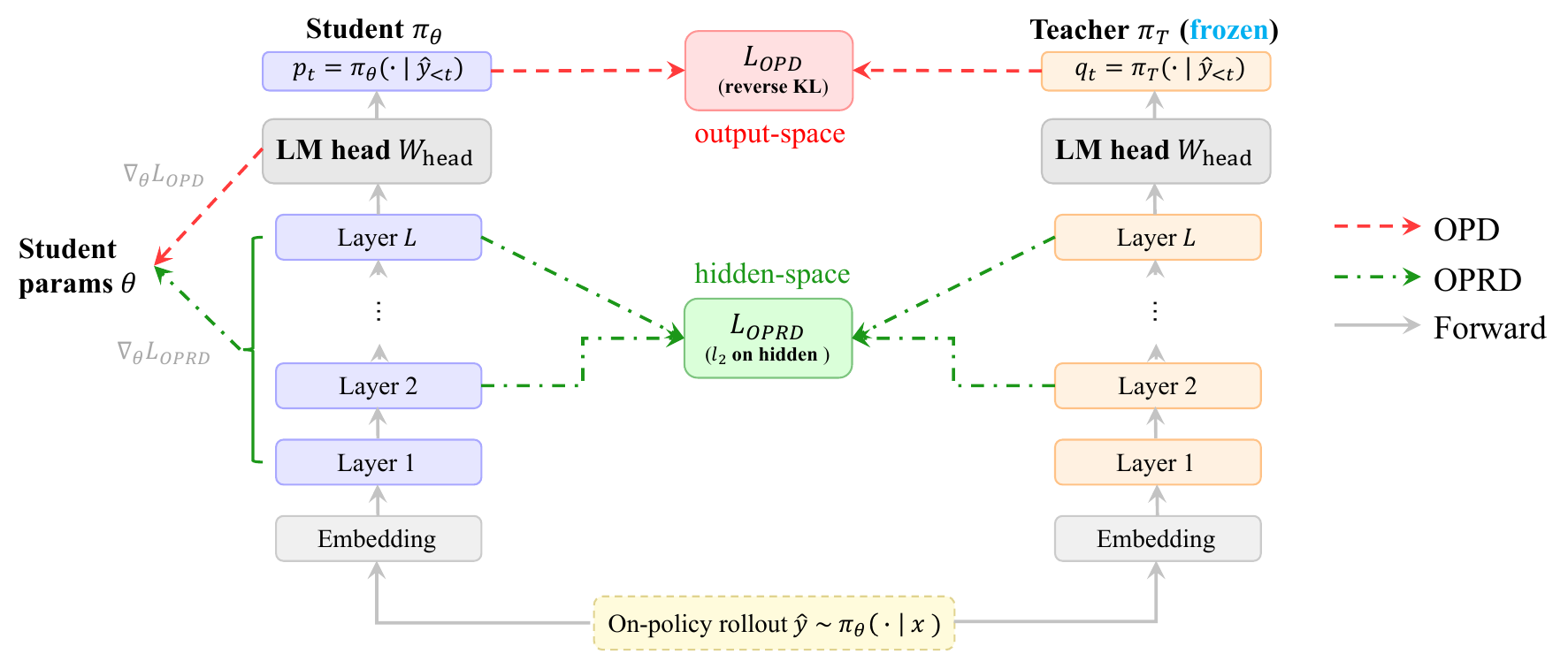}  
    \caption{%
        \textbf{Architecture of OPRD vs.\ output-space OPD.}
        Both methods share the same on-policy rollout $\hat{y} \sim \pi_\theta(\cdot\!\mid\!x)$, which is fed to the student (blue, trainable) and the teacher (orange, frozen).
        \textcolor{red!75}{\textbf{OPD}} extracts supervision \emph{after} the LM head, comparing output distributions $p_t$ and $q_t$ via reverse KL on a token subset.
        \textcolor{green!55!black}{\textbf{OPRD (ours)}} extracts supervision \emph{before} the LM head, comparing intermediate hidden states $h^{(l)}$ at selected layers via masked $\ell_2$ loss.
        OPRD is used on its own by default, but can optionally be combined with $\mathcal{L}_{\mathrm{OPD}}$ as $\mathcal{L}_{\mathrm{OPD}} + \mu\,\mathcal{L}_{\mathrm{OPRD}}$.
        OPRD taps the teacher \emph{before} the LM-head projection, exposing per-layer structural information that any output-space objective discards.
    }
    \label{fig:oprd_architecture}
\end{figure}

\subsection{A Shared Structural Limitation}
\label{sec:opd_limitations}

The three variants above span the full design space studied in prior work, yet they share a defining structural property:
\emph{the supervision signal is always a function of the next-token distributions $p_t$ and $q_t$ that the LM head produces.}
Equivalently, the only way teacher knowledge reaches the student is through the projection
$W_{\mathrm{head}}: \mathbb{R}^d \to \mathbb{R}^{|\mathcal{V}|}$
applied to the final hidden state.
The internal representations $\{h_{T,t}^{(l)}\}_{l < L}$, which are the very features that encode the teacher's intermediate reasoning, never enter the loss.
This output-only view has two immediate consequences that will become focal points of our analysis.
\emph{(i)~Statistical:} the most popular variant (sampled-token OPD) estimates each token-level KL from a single Monte Carlo draw, introducing variance that scales unfavorably with $|\mathcal{V}|$ and dominates the optimization signal once $p_t$ approaches $q_t$.
\emph{(ii)~Informational:} because $W_{\mathrm{head}}$ is low-rank ($d \ll |\mathcal{V}|$), the loss imposes only $d$ effective constraints per position regardless of $|S_t|$, leaving large directions of the hidden-state space unsupervised.
Our method, introduced next, attacks both issues by replacing $W_{\mathrm{head}}\!\circ\!h$ with $h$ itself as the alignment target.


\section{On-Policy Representation Distillation}
\label{sec:method}

We now present \textbf{On-Policy Representation Distillation (OPRD)}, a novel distillation framework that supervises the student in the hidden-state space on student-generated trajectories.
We define the method (\Cref{sec:oprd_def}) and state two theorems, one on gradient variance and one on the LM-head information bottleneck (\Cref{thm:variance},~\Cref{thm:bottleneck}), that explain why hidden-state supervision is a principled and effective complement to output-space distillation.

\begin{tcolorbox}[takeawaysbox]
\begin{itemize}[topsep=0pt, partopsep=0pt, leftmargin=12pt, itemsep=2pt]
    \item \textbf{Hidden-state supervision: richer and lower-variance.} OPRD supervises the student at intermediate layers, exposing structural information that the LM-head projection compresses away. The resulting MSE gradient is deterministic and carries (layers $\times$ positions $\times$ hidden-dim) scalars per sample, orders of magnitude denser than sampled-token OPD.
    \item \textbf{Efficient.} Operating entirely before the LM head, OPRD never materialises the $[B, T, |\mathcal{V}|]$ logits tensor, reducing wall-clock time by $1.44\times$ and cutting actor-update $\Delta$peak GPU memory by up to $54\%$.
    \item \textbf{Cross-architecture via a frozen low-rank bridge.} Teacher and student representations, despite living in different-dimensional spaces, share a low-rank structure that can be aligned to high cosine similarity through a low-rank projection. Once identified, this projector pair is frozen and serves as a fixed, architecture-agnostic interface through which supervision flows.
\end{itemize}
\end{tcolorbox}

\subsection{The OPRD Objective}
\label{sec:oprd_def}

The three OPD variants in \S\ref{sec:opd_variants} all operate in the \emph{output space} by matching next-token distributions $p_t$ and $q_t$.
OPRD instead supervises the student in the \emph{hidden-state space} on the same on-policy trajectories.
Intuitively, OPD asks the student to assign similar probabilities to tokens, whereas OPRD asks the student to produce similar internal representations at selected layers and positions.
We refer to this basic same-architecture formulation as \textbf{OPRD-Vanilla} (or simply OPRD when unambiguous); the cross-architecture extension, \textbf{OPRD-Bridge}, is presented in \S\ref{sec:oprd_bridge}.

Let $\mathcal{L}_{\mathrm{layer}} \subseteq \{1,\ldots,L\}$ be the set of distilled layers
(e.g.\ the last layer, all layers, or a parity subset such as even/odd layers), and let
$\mathcal{P}(\hat{y}) \subseteq \{1,\ldots,T\}$ be the set of supervised response positions
(e.g.\ all tokens, the first $k$ tokens, or the last $k$ tokens).
We use a position mask $m_t \in \{0,1\}$ to indicate whether $t \in \mathcal{P}(\hat{y})$; for short responses, positions beyond the valid length are masked out rather than padded into the loss.
OPRD minimizes a layer-averaged, position-masked mean-squared error between student and teacher representations:
\begin{equation}
  \mathcal{L}_{\mathrm{OPRD}}(\theta)
  =
  \mathbb{E}_{x \sim \mathcal{D}_x,\;
      \hat{y} \sim \pi_\theta(\cdot \mid x)}
  \left[
    \frac{1}{|\mathcal{L}_{\mathrm{layer}}|}
    \sum_{l \in \mathcal{L}_{\mathrm{layer}}}
    \frac{1}{\sum_{t=1}^{T}m_t}
    \sum_{t=1}^{T}
      m_t \,
      \frac{1}{d}
      \Bigl\|
        h_{\theta,t}^{(l)} - \mathrm{sg}\!\bigl(h_{T,t}^{(l)}\bigr)
      \Bigr\|_2^2
  \right],
  \label{eq:oprd_obj}
\end{equation}
where $\mathrm{sg}(\cdot)$ denotes the stop-gradient operator on the teacher representation and $d$ is the hidden dimension.
The $1/d$ factor normalizes the loss across architectures with different hidden sizes; the position averaging $1/\sum_t m_t$ makes the loss invariant to the choice of $|\mathcal{P}(\hat{y})|$.
The two design knobs ($\mathcal{L}_{\mathrm{layer}}$, $\mathcal{P}(\hat{y})$) offer flexibility along two axes: \emph{depth} of supervision (single-layer vs.\ multi-layer) and \emph{breadth} of supervision (single-position vs.\ all-position).
For long chain-of-thought (CoT) responses common in mathematical reasoning, we typically set $\mathcal{P}(\hat{y})$ to the last $k$ response tokens and $\mathcal{L}_{\mathrm{layer}}$ to all transformer layers, yielding dense layer-wise supervision on a compact suffix while keeping memory bounded.
We empirically study the effect of these design choices in \S\ref{sec:exp_mechanism}.
OPRD is a self-contained training objective and our main results (\S\ref{sec:experiments}) are reported in the OPRD-only setting.
For completeness, OPRD also composes additively with any output-space OPD variant as
\begin{equation}
  \mathcal{L}(\theta) = \mathcal{L}_{\mathrm{OPD}}(\theta) + \mu\,\mathcal{L}_{\mathrm{OPRD}}(\theta),
  \quad \mu \ge 0,
  \label{eq:oprd_opd_combined}
\end{equation}
at essentially zero infrastructure cost since both terms are computed on the same on-policy rollout and share a single teacher forward pass.

\subsection{Why OPRD Works}
\label{sec:oprd_theory}%
Two complementary properties, in one-to-one correspondence with the two limitations of \S\ref{sec:intro}, explain why hidden-state supervision is a principled complement to output-space OPD. We state both as informal theorems; precise statements and proofs are deferred to \Cref{sec:oprd_theorems_appendix}.

\begin{theorem}[Zero-variance gradient]
\label{thm:variance}%
\phantomsection\label{sec:oprd_variance}%
Let $g_{\mathrm{OPD}}$ and $g_{\mathrm{OPRD}}$ be the per-sample stochastic gradients of sampled-token OPD and OPRD~\eqref{eq:oprd_obj} on an on-policy rollout $\hat{y}\sim\pi_\theta(\cdot\!\mid\!x)$. Conditioned on $(x,\hat{y})$,
\begin{equation}
  \mathrm{Var}\!\left[g_{\mathrm{OPRD}}\,\big|\,x,\hat{y}\right] = 0,
  \qquad
  \mathrm{Var}\!\left[g_{\mathrm{OPD}}\,\big|\,x,\hat{y}\right]
  \;\propto\;
  \mathrm{Var}_{\hat{y}_t\sim p_t}\!\bigl[\log p_t(\hat{y}_t)-\log q_t(\hat{y}_t)\bigr],
  \label{eq:variance_gap}
\end{equation}
where the right-hand variance is over per-position token sampling.
\end{theorem}

The OPD variance in \eqref{eq:variance_gap} does \emph{not} vanish as $p_t\!\to\! q_t$, and through the score-function term $\nabla_\theta\log p_t(\hat{y}_t)$ it dominates the policy gradient late in training; this is the mechanism behind the late-stage stagnation of pure OPD (Limitation~1 in \S\ref{sec:intro}). OPRD adds zero conditional variance and therefore provides a stable optimization signal even after the output distribution has nearly converged.

\paragraph{Comparison with top-$k$: no truncation bias.}
Top-$k$ OPD eliminates the per-position sampling variance of sampled-token OPD by deterministically selecting the $k$ highest-probability tokens.
However, it introduces a \emph{truncation bias}: the loss is computed only over the student's top-$k$ support $S_t$, so any teacher probability mass outside $S_t$ is invisible to the gradient.
When the student's top-$k$ set does not coincide with the teacher's (a common situation early in training and on difficult tokens), the student receives no signal to shift probability toward the teacher's preferred tokens that lie outside $S_t$.
This bias is systematic and cannot be reduced by training longer or increasing the batch size.
OPRD avoids this problem entirely: its MSE objective $\|h_{\theta,t}^{(l)} - h_{T,t}^{(l)}\|_2^2$ operates on continuous $d$-dimensional vectors with no token selection or truncation step, so every dimension of the teacher's hidden state contributes to the gradient at every supervised position.

\begin{theorem}[Hidden-state information beyond the LM head]
\label{thm:bottleneck}%
\phantomsection\label{sec:oprd_complexity}%
Let $W_{\mathrm{head}}\!\in\!\mathbb{R}^{|\mathcal{V}|\times d}$ have singular values $\sigma_1\!\ge\!\cdots\!\ge\!\sigma_d\!>\!0$ with right-singular vectors $v_1,\ldots,v_d$, and define the \emph{effective null space} $\mathcal{N}_W \triangleq \{\Delta h\in\mathbb{R}^d : W_{\mathrm{head}}\,\Delta h \in \mathrm{span}\{\mathbf{1}\}\}$, i.e., the set of hidden-state perturbations whose image under $W_{\mathrm{head}}$ is an additive softmax-invariant shift. For any last-layer student/teacher hidden states $h_\theta,h_T\!\in\!\mathbb{R}^d$ and any output-space OPD loss $\ell_{\mathrm{out}}$ (sampled-token, top-$k$, or full-vocabulary reverse KL),
\begin{equation}
  \ell_{\mathrm{out}}(h_\theta,h_T) = 0
  \quad\text{whenever}\quad
  h_\theta - h_T \in \mathcal{N}_W,
  \label{eq:bottleneck_null}
\end{equation}
and along $h_\theta-h_T = \alpha v_d$ with $\|v_d\|=1$,
\begin{equation}
  \|h_\theta-h_T\|^2 / \ell_{\mathrm{out}}(h_\theta,h_T)\;\gtrsim\;(\sigma_1/\sigma_d)^2,
  \label{eq:bottleneck_gap}
\end{equation}
where $\gtrsim$ hides a constant depending only on $\ell_{\mathrm{out}}$ and the logit range (made precise in \Cref{sec:oprd_theorems_bottleneck}).
\end{theorem}

The ratio in \eqref{eq:bottleneck_gap} scales as $(\sigma_1/\sigma_d)^2$, which is typically very large for production LLMs due to the ill-conditioned singular spectrum of $W_{\mathrm{head}}$. This means hidden-state deviations along low-singular-value directions can be orders of magnitude larger than along top directions while producing the same output-space loss; moreover output-space OPD has no mechanism to constrain intermediate hidden states $h^{(l)}$ for $l\!<\!L$. OPRD~\eqref{eq:oprd_obj} penalizes exactly the directions in $\mathcal{N}_W$ and supervises any subset of intermediate layers, exposing (layers $\times$ positions $\times$ hidden-dim) scalars of structural information per sample that the LM-head projection necessarily compresses away (Limitation~2 in \S\ref{sec:intro}).

\subsection{OPRD-Bridge: Cross-Architecture Extension}
\label{sec:oprd_bridge}

The OPRD-Vanilla objective defined in \S\ref{sec:oprd_def} implicitly assumes that the student and teacher share a \emph{well-aligned representation space}: the MSE in Eq.~\eqref{eq:oprd_obj} directly compares $h_{\theta,t}^{(l)}$ and $h_{T,t}^{(l)}$ in the same $\mathbb{R}^d$, which is meaningful only when corresponding dimensions carry comparable semantics.
Sharing the same model architecture is a sufficient condition but not the fundamental one.
In the same-architecture experiments of \S\ref{sec:oprd_def}, the two models share a common origin: JustRL-1.5B (teacher) was obtained by applying RL to R1-Distill-1.5B (student), which perturbs the representation space only mildly; the two models therefore retain a high degree of representational alignment, and the full $\mathbb{R}^d$ channel acts as a high-quality all-pass filter that transmits every dimension of the teacher's hidden state faithfully.
Conversely, two models with identical architecture but trained from different initialisations or with different tokenizers would have misaligned representation spaces, making direct MSE comparison meaningless despite matching dimensionality.

The true bottleneck, therefore, is \emph{representational alignment}, not architectural identity.
When this alignment is absent (as is typically the case for models that differ in depth ($L_S \neq L_T$), width ($d_S \neq d_T$), or training history), the hidden states are incommensurable, and we need a mechanism to bridge the gap.
This section introduces \textbf{OPRD-Bridge}, which identifies and exploits a shared low-rank structure between heterogeneous models to construct such a bridge.

The key observation is that \emph{although the full representations are incomparable, their principal variation directions are not.}
We first demonstrate this empirically (\S\ref{sec:shared_structure}), then show how to exploit it (\S\ref{sec:bridge_construction},\,\ref{sec:bridge_distillation}), and finally explain why the resulting design takes the specific form it does (\S\ref{sec:design_philosophy}).

\begin{figure}[!t]       
    \centering              
    \includegraphics[width=0.9\textwidth]{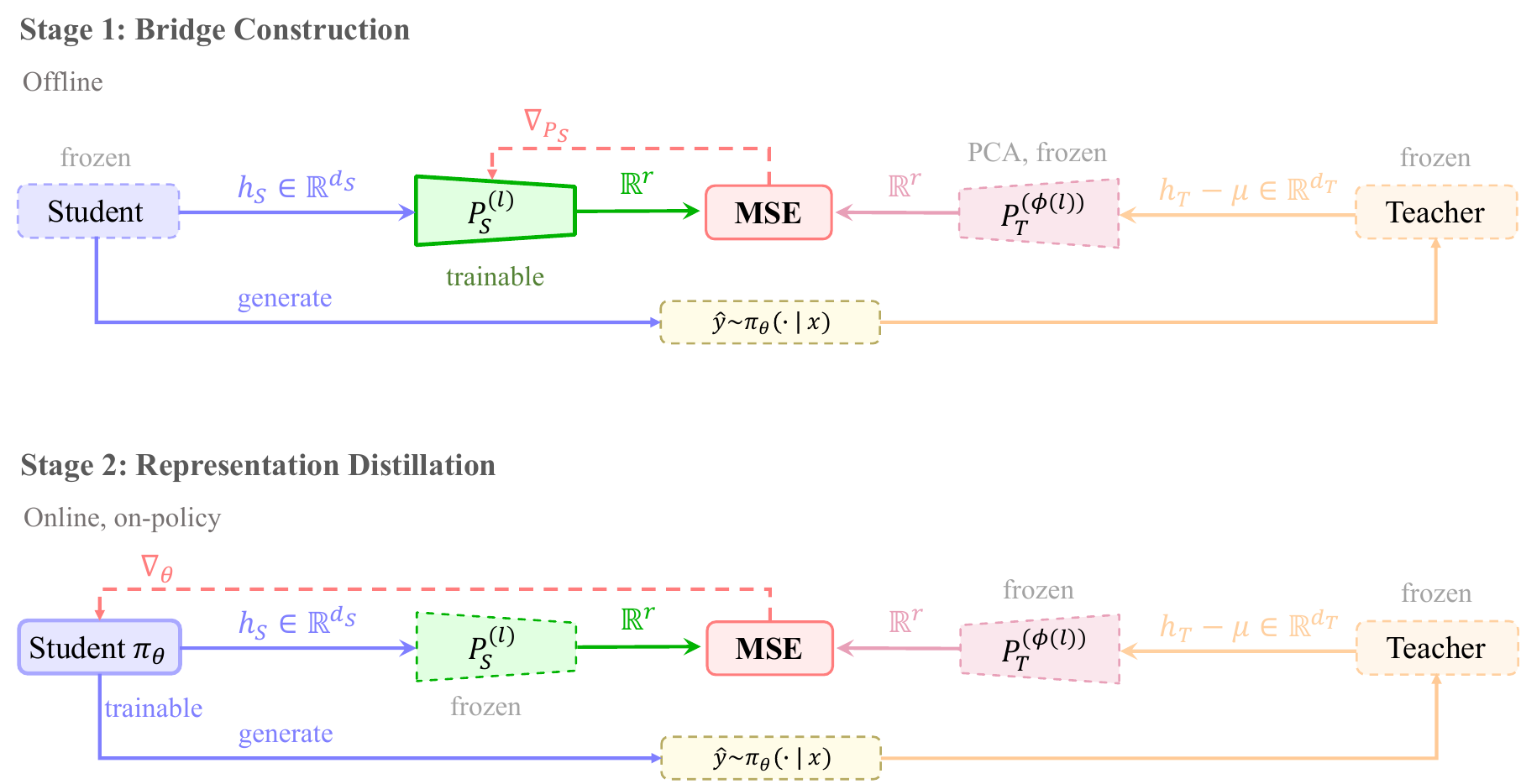}  
    \caption{%
        \textbf{Overview of OPRD-Bridge.}
        \emph{Top (Stage 1):} both models are frozen; $P_T$ is obtained via PCA of teacher hidden states and $P_S$ is trained to align with $P_T$ in $\mathbb{R}^r$.
        After convergence, both projectors are frozen.
        \emph{Bottom (Stage 2):} the bridge $(P_T, P_S)$ serves as a fixed interface; only the student backbone $\theta$ is updated via on-policy rollouts through the frozen bridge.
    }
    \label{fig:oprd_bridge_overview}
\end{figure}

\subsubsection{The Shared Low-Rank Structure}
\label{sec:shared_structure}

Before presenting the method, we establish the empirical fact that motivates it.
We take a Qwen3-4B teacher (36 layers, $d_T\!=\!2560$) and a Qwen3-1.7B student (28 layers, $d_S\!=\!2048$) as an illustrative example, sample 2K prompts from DAPO-Math-17K, generate responses with the student, and collect both models' hidden states on these responses.
We then perform PCA on the centered teacher hidden states at each layer, extracting the top-$r$ principal directions as a basis $P_T \in \mathbb{R}^{r \times d_T}$.
Next, we train a simple linear map $P_S \in \mathbb{R}^{r \times d_S}$ to minimize $\|P_S h_S - P_T(h_T - \mu_T)\|_2^2$ with both models frozen.

The result is striking.
At $r\!=\!8$, the cosine similarity between $P_S(h_S)$ and $P_T(h_T - \mu_T)$ reaches \textbf{95\%} averaged across all layer pairs.
Eight dimensions suffice to capture the shared structure between two models that, in their full representation spaces, show zero alignment.

Equally revealing: this similarity \emph{decreases} as $r$ increases.
At $r\!=\!32$ it drops to $94.1\%$, at $r\!=\!256$ to $87.9\%$, and continues declining to $77.0\%$ at full rank $r\!=\!2048$.
This is not a failure of the method; it is a signal about the data.
The shared structure between teacher and student is concentrated in very few principal directions.
Beyond those directions, the two models diverge, encoding architecture-specific information that is not shared and should not be distilled.

This observation has a direct methodological implication: \emph{the bridge should be low-rank by design, not as a computational shortcut, but because the shared structure itself is low-rank.}
Higher rank does not mean more information transfer; it means more noise.

\subsubsection{Stage 1: Bridge Construction}
\label{sec:bridge_construction}

The bridge is built with \emph{both model weights frozen}; only the projectors are involved.
We write $h_S$ for the student hidden states in this stage (emphasizing that the student backbone is fixed), reserving $h_\theta$ for Stage~2 where the backbone parameters $\theta$ are updated.

We first establish a layer correspondence.
When $L_S \neq L_T$, we pair each student layer $l_S$ with a teacher layer via proportional spacing:
\begin{equation}
  \phi(l_S) = \mathrm{round}\!\left(\frac{l_S - 1}{L_S - 1} \cdot (L_T - 1)\right) + 1,
  \label{eq:layer_mapping}
\end{equation}
so that the first layers align, the last layers align, and intermediate layers are evenly distributed.

For each teacher layer $l_T$, we compute the PCA basis from the centered hidden states:
\begin{equation}
  P_T^{(l_T)} = \mathrm{TopPrincipalDirections}\!\left(
    h_{T}^{(l_T)} - \mu_T^{(l_T)},\; r
  \right) \in \mathbb{R}^{r \times d_T}.
  \label{eq:teacher_pca}
\end{equation}
This is computed once and permanently frozen.
It captures the directions along which the teacher's representations vary most.
In practice, we obtain $P_T^{(l_T)}$ via PCA on the sample covariance rather than a thin SVD of the hidden-state matrix.
Stacking the teacher's response-token hidden states into $H \in \mathbb{R}^{n \times d_T}$, the two are equivalent up to centering: the top-$r$ right singular vectors of the centered matrix $H - \mathbf{1}\mu_T^\top$ coincide with the top-$r$ eigenvectors of the sample covariance $\Sigma = \frac{1}{n-1}(H - \mathbf{1}\mu_T^\top)^\top(H - \mathbf{1}\mu_T^\top) \in \mathbb{R}^{d_T \times d_T}$.
We therefore form $\Sigma$ and take its eigendecomposition, which is far cheaper than an SVD of $H$ in the typical regime $n \gg d_T$ (the cost is governed by the $d_T \times d_T$ eigensolve rather than the large row count $n$), and we subsample at most $16{,}384$ token rows to estimate $\Sigma$.
A single eigendecomposition per layer yields the full ordered basis; the rank-$r$ projector for any $r$ is then obtained by slicing the top-$r$ eigenvectors, without re-running the decomposition.

For each student layer $l_S$, we then train a linear projector $P_S^{(l_S)} \in \mathbb{R}^{r \times d_S}$ to minimize:
\begin{equation}
  \mathcal{L}_{\mathrm{bridge}}
  = \sum_{l_S} \sum_{t}
    \left\|
      P_S^{(l_S)} h_{S,t}^{(l_S)} - P_T^{(\phi(l_S))}\!\bigl(h_{T,t}^{(\phi(l_S))} - \mu_T^{(\phi(l_S))}\bigr)
    \right\|_2^2.
  \label{eq:bridge_train}
\end{equation}
After 20 epochs of training (with both models frozen), $P_S$ converges and is also frozen.
The bridge is now complete: a pair of fixed linear maps $(P_T, P_S)$ that project both architectures into a shared $r$-dimensional subspace where their representations are directly comparable.

\subsubsection{Stage 2: Bridge-Based Distillation}
\label{sec:bridge_distillation}

With the bridge frozen, we perform on-policy representation distillation.
The student generates rollouts, both models perform forward passes, and the loss is computed \emph{through the frozen bridge}:
\begin{equation}
  \ell^{(l)}(\theta; \hat{y})
  =
  \frac{1}{\sum_{t} m_t}
  \sum_{t=1}^{T}
    m_t \,
    \left\|
      P_S^{(l)}(h_{\theta,t}^{(l)}) - P_T^{(\phi(l))}\!\bigl(h_{T,t}^{(\phi(l))} - \mu_T^{(\phi(l))}\bigr)
    \right\|_2^2,
  \label{eq:per_layer_loss}
\end{equation}
where $m_t$ is the same position mask as in Eq.~\eqref{eq:oprd_obj} (defaulting to the last $k\!=\!2000$ tokens).
The full objective averages over layers:
\begin{equation}
  \mathcal{L}_{\text{OPRD-Bridge}}(\theta)
  =
  \mathbb{E}_{x,\, \hat{y} \sim \pi_\theta(\cdot \mid x)}
  \left[
    \frac{1}{|\mathcal{L}_{\mathrm{layer}}|}
    \sum_{l \in \mathcal{L}_{\mathrm{layer}}}
    \ell^{(l)}(\theta; \hat{y})
  \right].
  \label{eq:xoprd_obj}
\end{equation}
Only the student backbone $\theta$ receives gradients.
In practice, we L2-normalize both projected vectors before computing MSE, making the loss scale-invariant.
As with OPRD-Vanilla, the bridge objective composes additively with any output-space OPD variant via Eq.~\eqref{eq:oprd_opd_combined}.

\Cref{alg:oprd_bridge} summarizes the complete pipeline.

\begin{algorithm}[t]
\caption{OPRD-Bridge}
\label{alg:oprd_bridge}
\begin{algorithmic}[1]
\STATE \textbf{Input:} Student $\pi_\theta$ ($L_S$ layers, dim $d_S$), Teacher $\pi_T$ ($L_T$ layers, dim $d_T$), rank $r$, prompt set $\mathcal{D}_x$
\STATE
\STATE \textit{// Stage 1: Bridge Construction (both models frozen)}
\STATE Sample 2K prompts; generate responses $\hat{y} \sim \pi_\theta(\cdot \mid x)$
\STATE Compute layer mapping $\phi$ via Eq.~\eqref{eq:layer_mapping}
\FOR{each layer pair $(l_S, \phi(l_S))$}
  \STATE $P_T^{(\phi(l_S))} \leftarrow \text{Top-}r\text{ principal directions of } (H_T - \mu_T)$ \hfill \textit{// PCA, frozen}
  \STATE Initialize $P_S^{(l_S)} \in \mathbb{R}^{r \times d_S}$
\ENDFOR
\FOR{epoch $= 1, \ldots, 20$}
  \STATE Update all $P_S$ to minimize Eq.~\eqref{eq:bridge_train}
\ENDFOR
\STATE Freeze $P_S$ \hfill \textit{// Bridge complete}
\STATE
\STATE \textit{// Stage 2: Bridge-Based Distillation (bridge + teacher frozen)}
\FOR{each training step}
  \STATE Sample $x$; generate $\hat{y} \sim \pi_\theta(\cdot \mid x)$
  \STATE Forward both models; compute $\mathcal{L}_{\text{OPRD-Bridge}}$ via Eq.~\eqref{eq:xoprd_obj}
  \STATE Update $\theta \leftarrow \theta - \eta \nabla_\theta \mathcal{L}$ \hfill \textit{// Only student backbone}
\ENDFOR
\end{algorithmic}
\end{algorithm}

\subsubsection{Design Philosophy}
\label{sec:design_philosophy}

Every design choice in OPRD-Bridge follows from a single principle: \emph{the bridge should behave like the vocabulary in logit distillation}, a fixed, pre-established interface that neither model modifies during training.
The frozen, low-rank bridge plays two complementary roles:

\begin{itemize}[topsep=2pt, partopsep=0pt, leftmargin=14pt, itemsep=3pt]
    \item \textbf{A stable interface that preserves pre-existing alignment.}
    The projectors $(P_T, P_S)$ encode an alignment that \emph{already exists} between the two models' representations before any distillation takes place; Stage~1 merely discovers it and freezes it into a fixed pair of linear maps.
    This mirrors logit distillation, where the shared vocabulary provides a stable coordinate system that neither model modifies during training: if this common reference frame changed, the optimization target would be non-stationary and training would destabilize.
    Freezing the bridge in Stage~2 ensures the on-policy training signal flows through this stable channel rather than destroying it. If $P_S$ were trained jointly with the backbone, the projection and the target would shift simultaneously, erasing the very alignment the bridge was built to exploit; empirically, joint training indeed performs worse.
    \item \textbf{A low-pass filter that suppresses high-frequency noise.}
    The shared structure between heterogeneous models is intrinsically low-rank: the top-$r$ principal components correspond to the ``low-frequency'' modes where teacher and student are already highly aligned, while the orthogonal complement carries ``high-frequency'' architecture-specific variation that acts as noise for distillation.
    By projecting through a rank-$r$ bridge, we retain the well-aligned signal and discard the misaligned noise, which is precisely why Stage~2 training remains stable even under on-policy distribution shift.
    It is instructive to contrast this with the two all-pass alternatives.
    Full-vocabulary logit distillation transmits every dimension of the output space through the shared vocabulary, a high-quality all-pass channel, but at $O(|\mathcal{V}|)$ cost and without any denoising.
    OPRD-Vanilla likewise operates as an all-pass channel over the full $\mathbb{R}^d$ representation space; this works precisely \emph{because} same-architecture models with shared initialisation are already well-aligned in every direction, so no dimension is noise.
    When this alignment breaks down (different depth, width, or training history), the all-pass property becomes a liability: high-frequency directions carry misaligned noise, and a low-pass bridge is needed to filter it out.
\end{itemize}

\Cref{tab:bridge_comparison} makes the bridge--vocabulary analogy explicit.

\begin{table}[!t]
\centering\small
\caption{The bridge analogy: comparing the shared interface in output-space vs.\ representation-space distillation.}
\label{tab:bridge_comparison}
\begin{tabular}{lcc}
\toprule
& \textbf{Logit Distillation (OPD)} & \textbf{OPRD-Bridge (ours)} \\
\midrule
Bridge space & Vocabulary simplex $\Delta^{|\mathcal{V}|-1}$ & Low-rank subspace $\mathbb{R}^r$ \\
Bridge map (teacher) & LM head $W_T \in \mathbb{R}^{|\mathcal{V}| \times d_T}$ & PCA basis $P_T \in \mathbb{R}^{r \times d_T}$ \\
Bridge map (student) & LM head $W_S \in \mathbb{R}^{|\mathcal{V}| \times d_S}$ & Projector $P_S \in \mathbb{R}^{r \times d_S}$ \\
Frozen during training? & Yes & Yes \\
Alignment space & Vocabulary (output) & Representation \\
Supervision depth & Final layer only & Every layer \\
Information preserved & Next-token distribution & Per-layer geometric structure \\
\bottomrule
\end{tabular}
\end{table}

Why PCA for the teacher side?
Because the Platonic Representation Hypothesis predicts that strong models converge in their principal variation directions~\citep{huh2024platonic}, and PCA extracts exactly these directions.
It is also deterministic, data-efficient (2K prompts suffice), and introduces no trainable parameters on the teacher side.

Why such a low rank?
Because the shared structure \emph{is} low-rank: our pre-experiment shows 95\% alignment at $r\!=\!8$, reflecting the concentration of task-relevant information in very few directions.
The rank is therefore not a computational budget; it is a statement about the geometry of the shared representation, and the low-pass property described above is what makes the bridge a robust supervision channel.
We formalize this intuition in the following theorem (stated informally; the precise version with proof is in \Cref{sec:bridge_theory_appendix}).

\begin{theorem}[Optimality and bias--variance trade-off of the low-rank bridge, informal]
\label{thm:bridge_optimality}
Assume teacher hidden states are drawn from a distribution with covariance $\Sigma_T$ (eigenvalues $\lambda_1 \!\geq\! \cdots \!\geq\! \lambda_{d_T}$), and that the cross-model alignment in direction $i$ is $\rho_i \in [0,1]$ (decreasing in $i$: high-eigenvalue directions are better aligned).
Then:
\begin{enumerate}[leftmargin=18pt, itemsep=2pt, topsep=2pt]
    \item \textbf{(Rate--distortion optimality.)} Among all rank-$r$ linear encodings of the teacher representation, PCA minimizes the expected distillation error under Gaussian assumptions. The bridge is the information-theoretically optimal channel at capacity $r$.
    \item \textbf{(Bias--variance decomposition.)} The expected distillation error decomposes as
    \begin{equation}
      \mathcal{E}(r)
      = \underbrace{\textstyle\sum_{i > r} \lambda_i\, \rho_i^2}_{\mathrm{Bias:\ discarded\ aligned\ signal}}
      \;+\; \underbrace{\textstyle\sum_{i=1}^{r} \lambda_i\,(1 - \rho_i^2)}_{\mathrm{Variance:\ transmitted\ misaligned\ noise}}.
      \label{eq:bias_variance}
    \end{equation}
    When $\rho_i$ decreases with $i$ (empirically verified), there exists an optimal rank $r^*$ that minimizes $\mathcal{E}(r)$: below $r^*$, bias dominates (useful signal is discarded); above $r^*$, variance dominates (misaligned noise is transmitted).
\end{enumerate}
\end{theorem}

\noindent
The theorem explains three empirical observations simultaneously: (i)~why a very low rank suffices (the $\lambda_i$ decay fast and $\rho_i$ decay fast, so both bias and variance are small at $r^*\!\approx\!8$); (ii)~why increasing rank beyond the sweet spot \emph{hurts} alignment (variance term grows); and (iii)~why freezing the bridge is critical (a moving $P_S$ would invalidate the $\rho_i$ estimates on which the optimality rests).

Finally, we note that because the bridge performs alignment in representation space rather than in the output space, it is decoupled from the vocabulary: the teacher and student need not share a tokenizer.
This opens a path toward cross-tokenizer and potentially cross-modal distillation, where the vocabulary-based alignment channel is unavailable.

\section{Experiments}
\label{sec:experiments}

We evaluate both OPRD variants: OPRD-Vanilla in the same-architecture setting (\S\ref{sec:exp_vanilla}), and OPRD-Bridge in the cross-architecture setting (\S\ref{sec:exp_bridge}).

\subsection{OPRD-Vanilla: Same-Architecture Distillation}
\label{sec:exp_vanilla}

We evaluate OPRD-Vanilla on competition-level mathematical reasoning, against (i)~a frozen teacher and an unmodified student baseline, and (ii)~two strong on-policy distillation baselines that share the same on-policy rollout and teacher forward pass as OPRD but extract supervision from the LM-head output.
The experiments test the two predictions of \S\ref{sec:oprd_theory}: OPRD provides a lower-noise, structurally richer training signal than any output-space OPD variant, and should therefore close the student--teacher gap that pure OPD cannot.

\begin{figure}[!t]
  \centering
  \begin{subfigure}[t]{0.32\linewidth}
    \includegraphics[width=\linewidth]{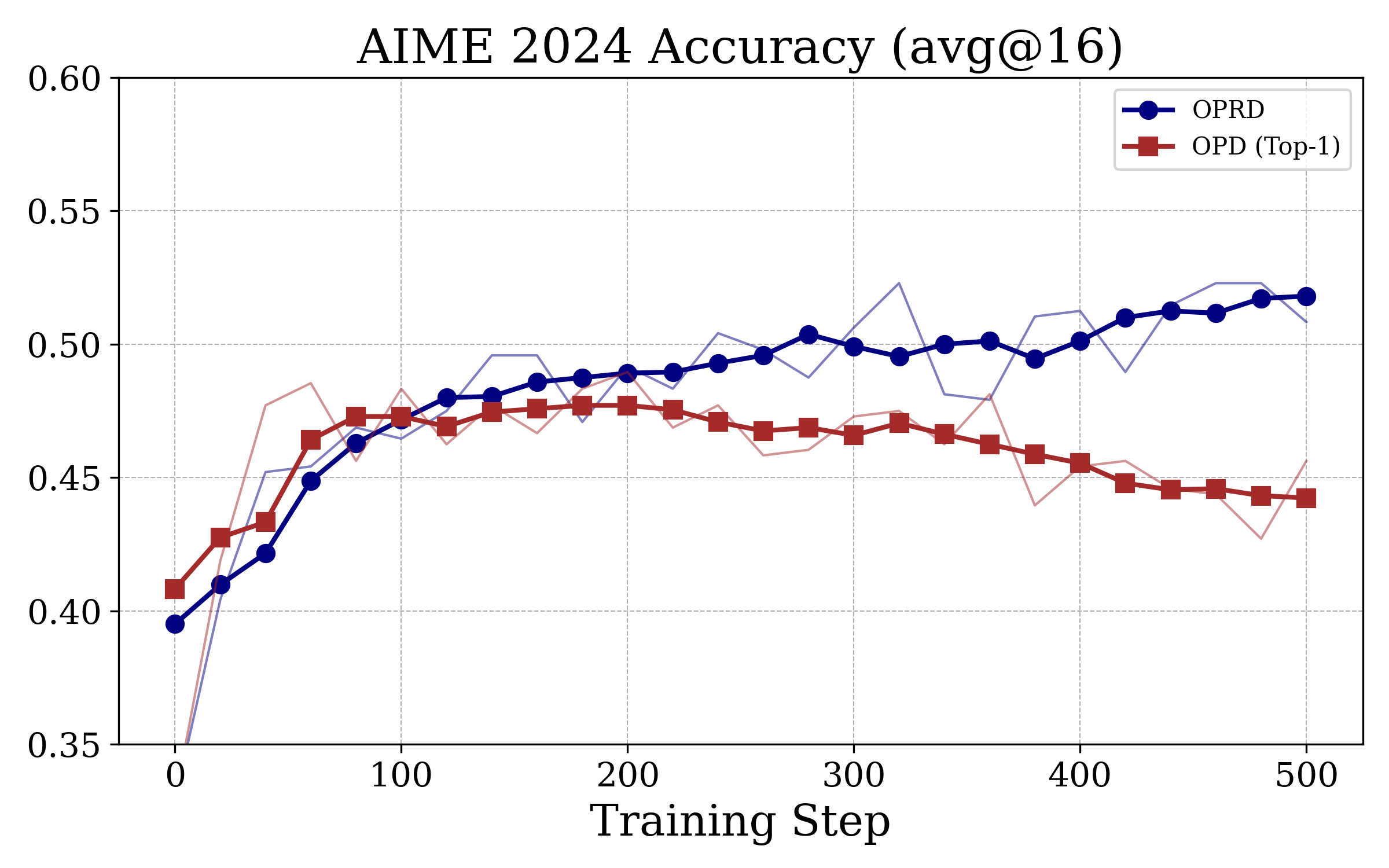}
    \caption{AIME24 \quad vs.\ OPD top-1}
    \label{fig:tc_aime24_top1}
  \end{subfigure}\hfill
  \begin{subfigure}[t]{0.32\linewidth}
    \includegraphics[width=\linewidth]{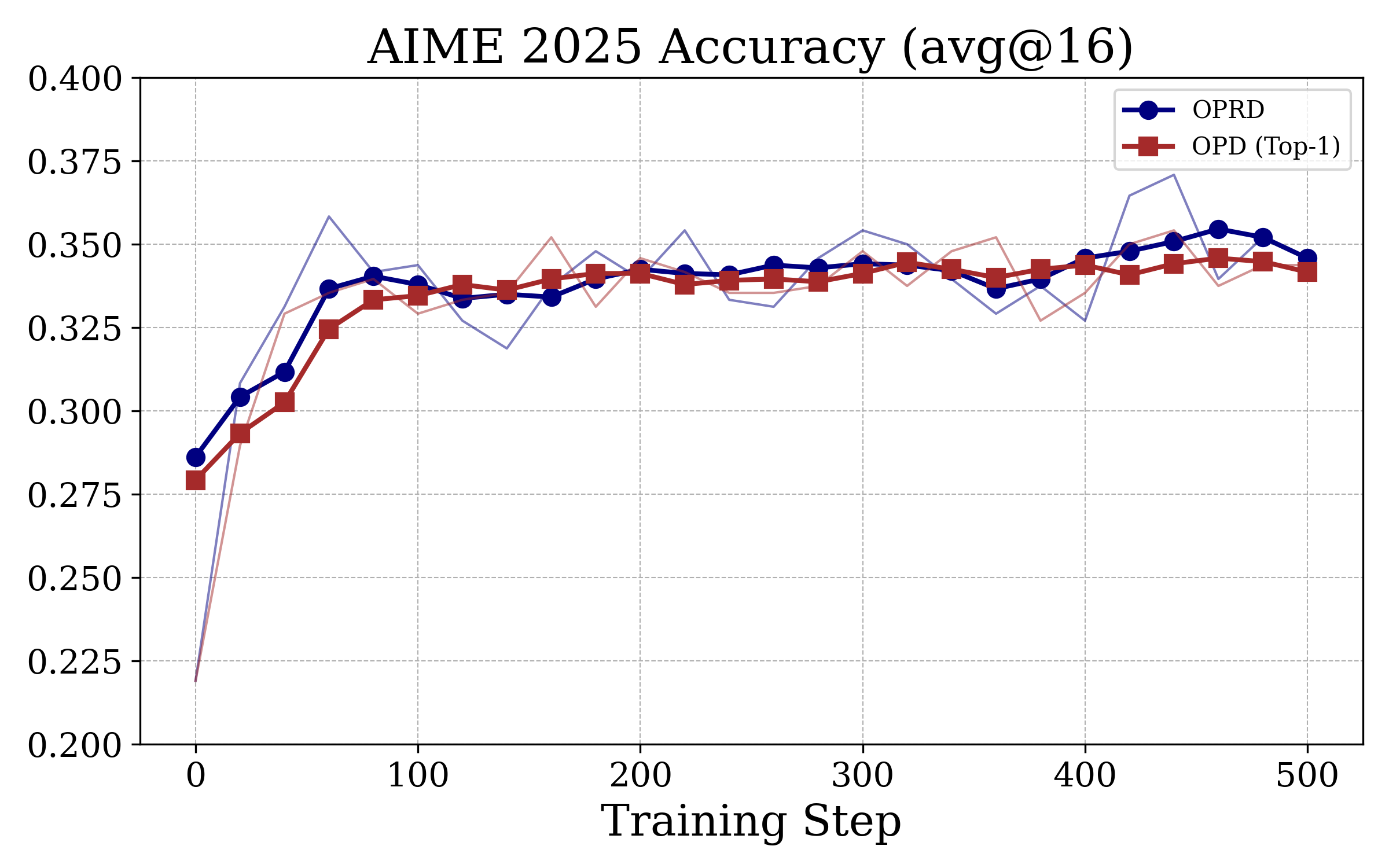}
    \caption{AIME25 \quad vs.\ OPD top-1}
    \label{fig:tc_aime25_top1}
  \end{subfigure}\hfill
  \begin{subfigure}[t]{0.32\linewidth}
    \includegraphics[width=\linewidth]{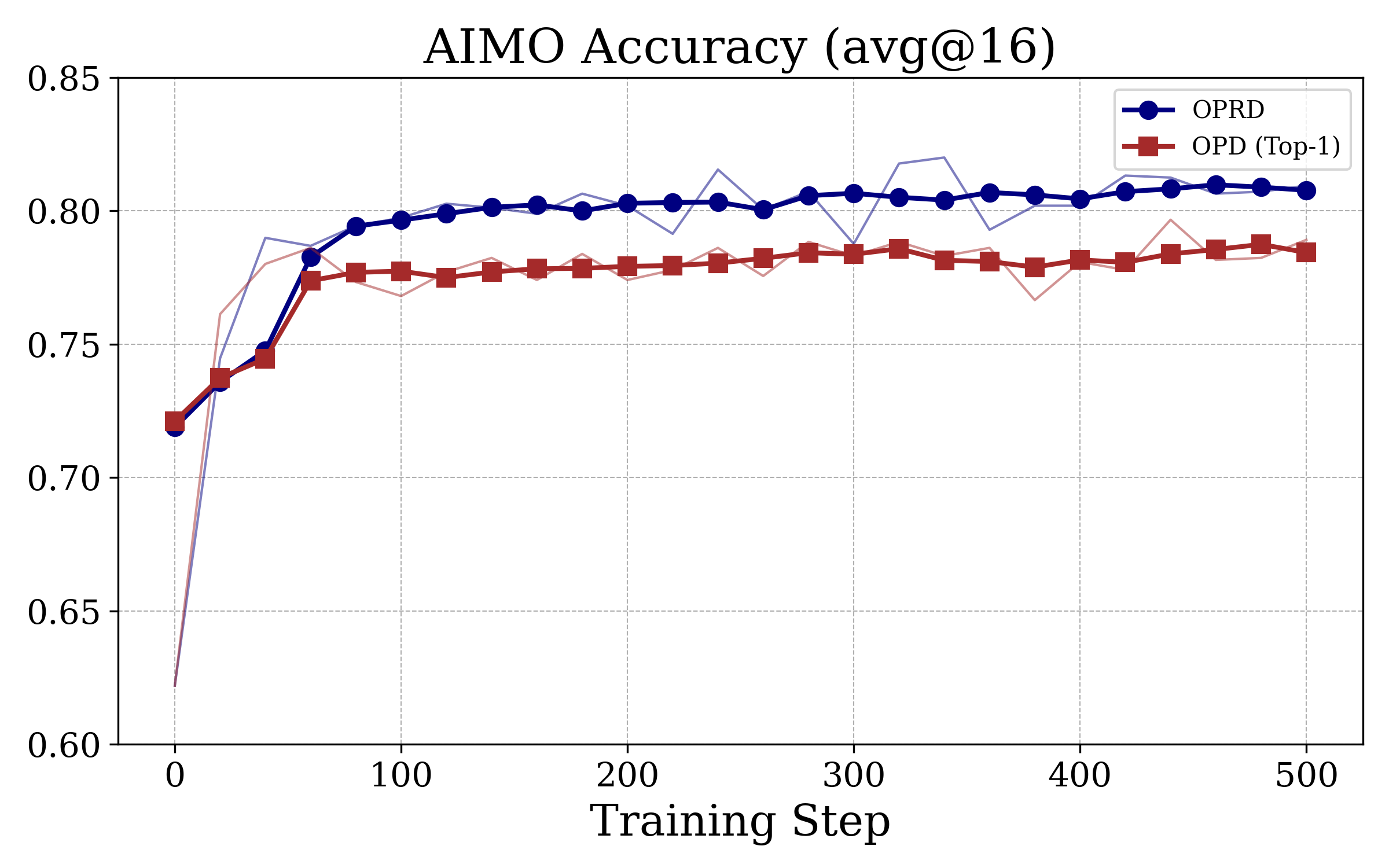}
    \caption{AIMO \quad vs.\ OPD top-1}
    \label{fig:tc_aimo_top1}
  \end{subfigure}\\[2pt]
  \begin{subfigure}[t]{0.32\linewidth}
    \includegraphics[width=\linewidth]{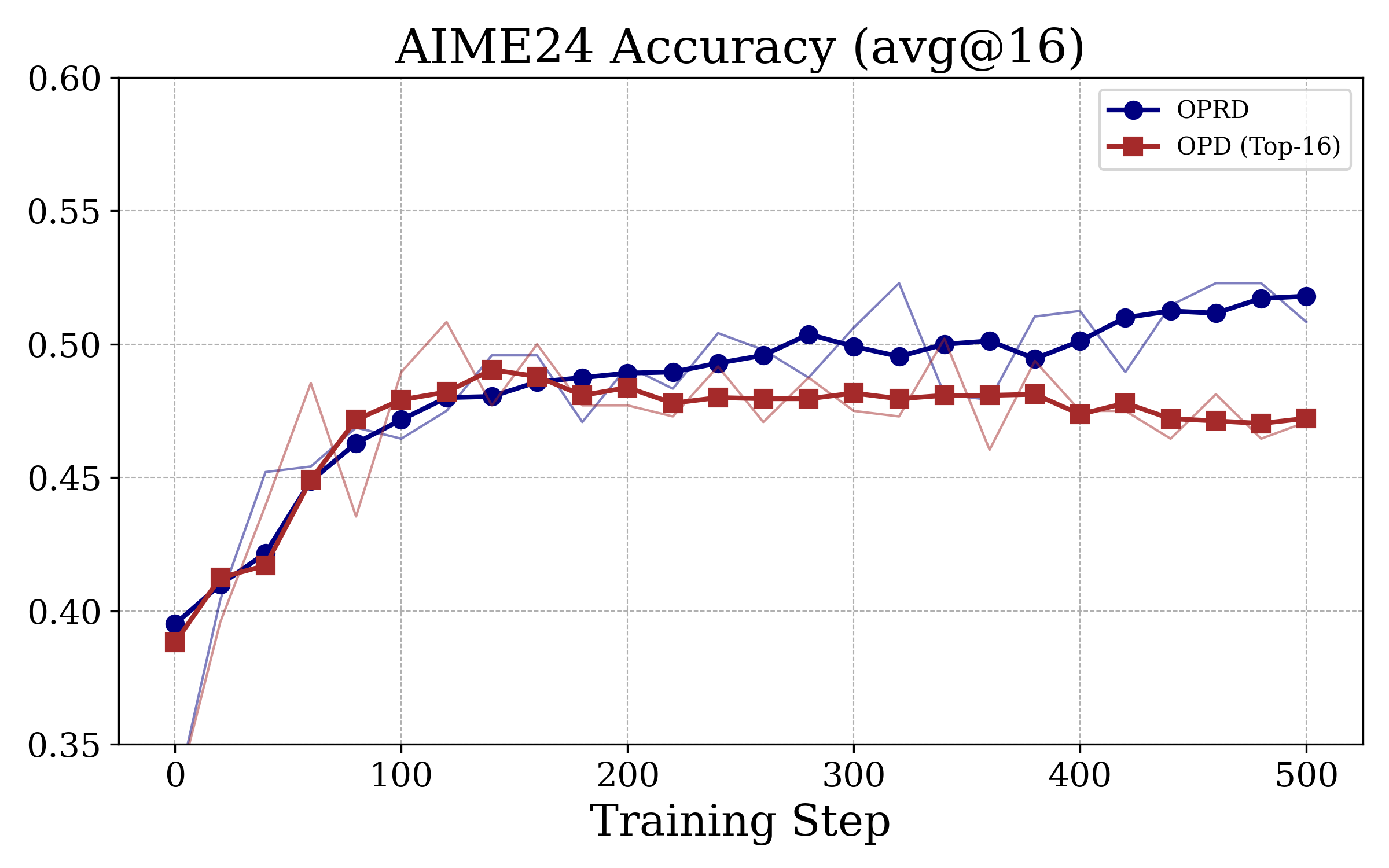}
    \caption{AIME24 \quad vs.\ OPD top-16}
    \label{fig:tc_aime24_top16}
  \end{subfigure}\hfill
  \begin{subfigure}[t]{0.32\linewidth}
    \includegraphics[width=\linewidth]{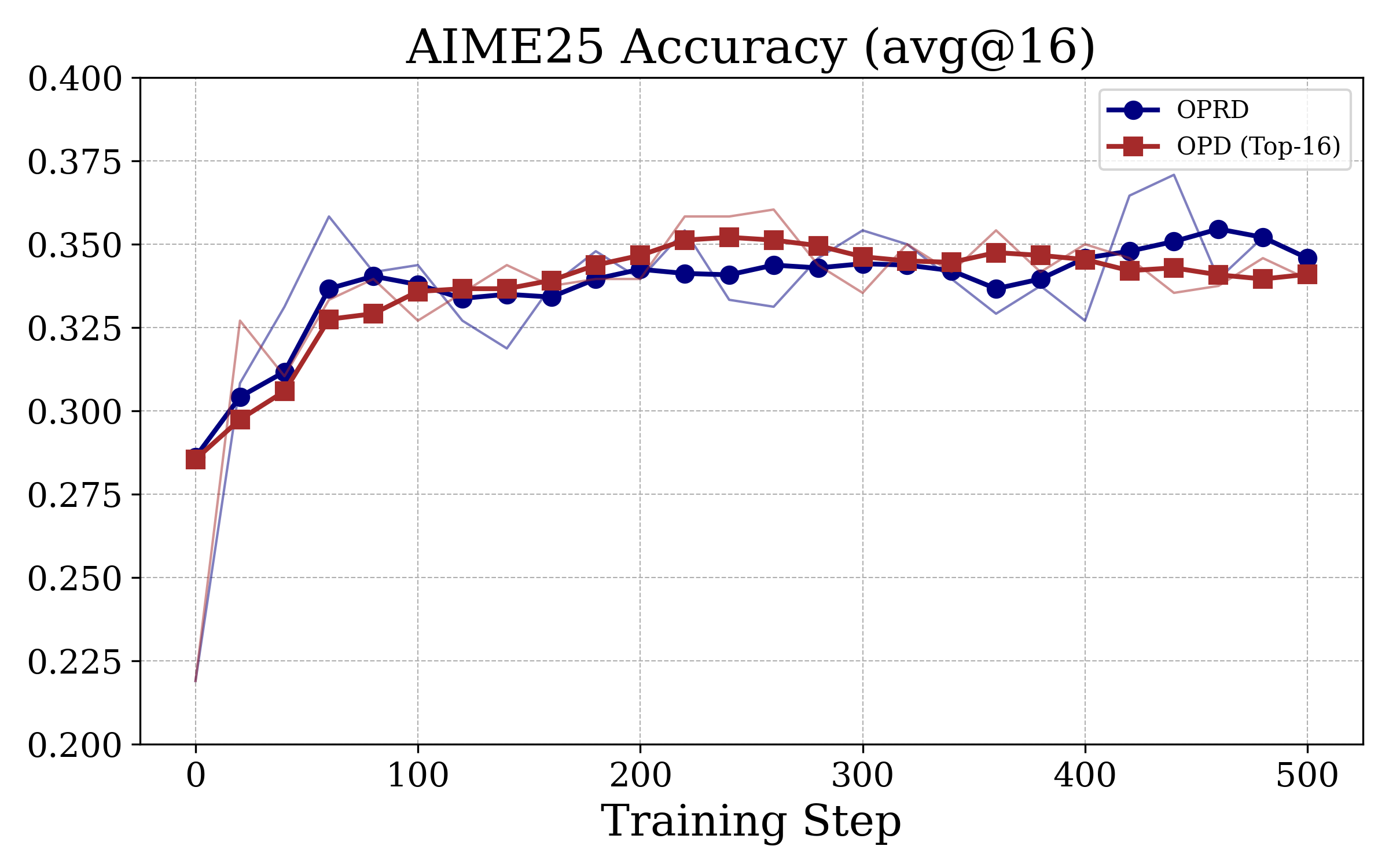}
    \caption{AIME25 \quad vs.\ OPD top-16}
    \label{fig:tc_aime25_top16}
  \end{subfigure}\hfill
  \begin{subfigure}[t]{0.32\linewidth}
    \includegraphics[width=\linewidth]{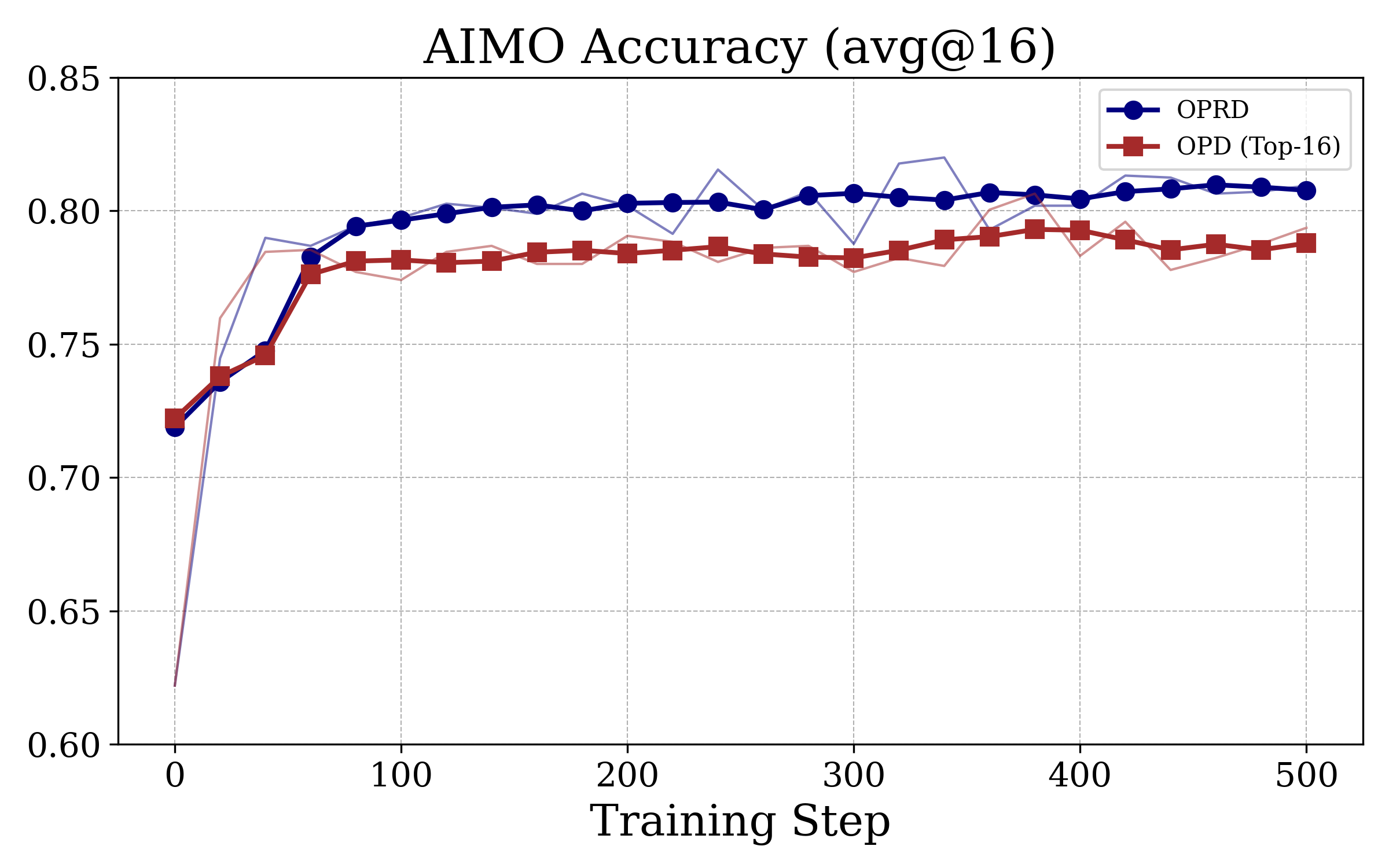}
    \caption{AIMO \quad vs.\ OPD top-16}
    \label{fig:tc_aimo_top16}
  \end{subfigure}
  \caption{%
    \textbf{Training dynamics of OPRD vs.\ OPD baselines} on AIME24 (left), AIME25 (middle), and AIMO (right).
    Top row: OPRD vs.\ OPD top-1 (sampled-token reverse KL); bottom row: OPRD vs.\ OPD top-16.
    Translucent line: raw Avg@16 at each evaluation step; solid line with markers: $5$-step centered rolling mean.
    Within each panel the two methods share the same student initialization, on-policy rollouts, teacher forward passes, and optimizer schedule, and differ only in where the supervision is extracted.
    On every benchmark OPRD continues to improve until it approaches the teacher's level, while both OPD variants plateau or oscillate, reflecting the late-stage stagnation of OPD and the deterministic-gradient advantage of OPRD predicted by \Cref{thm:variance}.
  }
  \label{fig:training_curves}
\end{figure}

\subsubsection{Experimental Setup}
\label{sec:exp_setup}

\paragraph{Models.}
Following ~\citep{li2026rethinking}, we use \texttt{JustRL-Deepseek-1.5B}~\citep{he2025justrl} (denoted \texttt{JustRL-1.5B}) as the (frozen) teacher and \texttt{DeepSeek-R1-Distill-Qwen-1.5B}~\citep{guo2025deepseek} (denoted \texttt{R1-distill-1.5B}) as the student.
Both models share the Qwen2.5-1.5B backbone ($L\!=\!28$ transformer layers, $d\!=\!1536$ hidden dimension, $|\mathcal{V}|\!\approx\!151$K vocabulary) and the same LM head $W_{\mathrm{head}}$, so OPRD-Vanilla's hidden-state targets are directly comparable across the two models ($d_S\!=\!d_T$).
The student starts from the public \texttt{R1-distill-1.5B} checkpoint, which already places it close to but well below the teacher in reasoning ability (\Cref{tab:main}).

\paragraph{Training data.}
On-policy prompts $x$ are drawn from DAPO-Math-17K~\citep{yu2026dapo}.
For each prompt the student samples $2$ responses $\hat{y}\sim\pi_\theta(\cdot\!\mid\!x)$ at temperature $1.0$ with a max generation length of $16{,}384$ tokens; we use a global batch of $8$ prompts per step.

\paragraph{Distillation objectives.}
We compare three on-policy distillation variants, all sharing the same rollouts $\hat{y}$ and the same single teacher forward pass per rollout:
\begin{itemize}[topsep=0pt, partopsep=0pt, leftmargin=12pt, itemsep=2pt]
  \item \textbf{OPD top-1} (sampled-token reverse KL): the per-position estimator $\ell_t = \log p_t(\hat{y}_t) - \log q_t(\hat{y}_t)$ evaluated only at the sampled token $\hat{y}_t$.
  \item \textbf{OPD top-16}: the per-position estimator $\sum_{v\in\mathcal{V}_{16}^t}\,p_t(v)\,[\log p_t(v) - \log q_t(v)]$ over the top-$16$ tokens of $p_t$, a strictly informative-superset of \textbf{OPD top-1}.
  \item \textbf{OPRD-Vanilla} (ours): the hidden-state objective \eqref{eq:oprd_obj} with $\mathcal{L}_{\mathrm{layer}}\!=\!\{1,\ldots,L\}$ (all $28$ layers) and $\mathcal{P}(\hat{y})$ set to the last $k\!=\!2000$ response tokens (i.e.\ the suffix in which the chain-of-thought converges to a final answer); reported in the OPRD-only setting ($\mu\!=\!0$ in \eqref{eq:oprd_opd_combined}). 
\end{itemize}

\paragraph{Optimization.}
All three methods are trained for $500$ optimizer steps with AdamW (peak learning rate $1\!\times\!10^{-5}$, linear warm-up over $3\%$ of total steps, cosine decay), bf16 mixed precision, and FSDP over $8\!\times\!$A100 (80G) GPUs at a micro-batch of $B\!=\!8$ and a maximum response length of $T\!=\!16{,}384$.

\paragraph{Evaluation.}
We report Avg@16 (average accuracy across $16$ independently sampled responses per prompt) at decoding temperature $0.7$ on three competition-level mathematical reasoning benchmarks: AIME~2024 (\textbf{AIME24}, $30$ problems), AIME~2025 (\textbf{AIME25}, $30$ problems), and \textbf{AIMO} (AI-MO/aimo-validation-amc, comprising AMC~2022 and AMC~2023, $83$ problems).
Final answers are extracted with the standard \texttt{boxed} parser and graded by exact-match against the official solution.

\subsubsection{Main Results}
\label{sec:exp_main}

\begin{table}[!t]
  \centering
  \small
  \caption{%
    \textbf{Main results on competition mathematical reasoning} (Avg@16, \%).
    \textbf{Bold} = best among the three distillation methods on each column;
    \underline{underline} = within evaluation noise of the teacher.
    All three distillation methods share the same on-policy rollouts, the same single teacher forward pass per rollout, and the same training budget; they differ only in where in the network the supervision is extracted.
  }
  \label{tab:main}
  \begin{tabular}{l ccc}
    \toprule
    \textbf{Method} & \textbf{AIME24} & \textbf{AIME25} & \textbf{AIMO} \\
    \midrule
    Teacher (\texttt{JustRL-1.5B})              & 50.8 & 35.6 & 79.5 \\
    Student (\texttt{R1-distill-1.5B})          & 32.9 & 21.9 & 62.2 \\
    \midrule
    OPD top-1 (sampled-token)                   & 42.3 & 33.5 & 77.0 \\
    OPD top-16                                  & 47.1 & 34.0 & 76.5 \\
    \textbf{OPRD-Vanilla} (ours)                        & \textbf{49.8} & \textbf{34.6} & \underline{\textbf{79.1}} \\
    \bottomrule
  \end{tabular}
\end{table}

\Cref{tab:main} reports Avg@16 for the teacher, the unmodified student, and the three on-policy distillation methods; \Cref{fig:training_curves} shows the corresponding training dynamics (discussed in detail in \S\ref{sec:exp_dynamics}).
Three observations follow.

\begin{figure}[!t]
  \centering
  \includegraphics[width=0.6\linewidth]{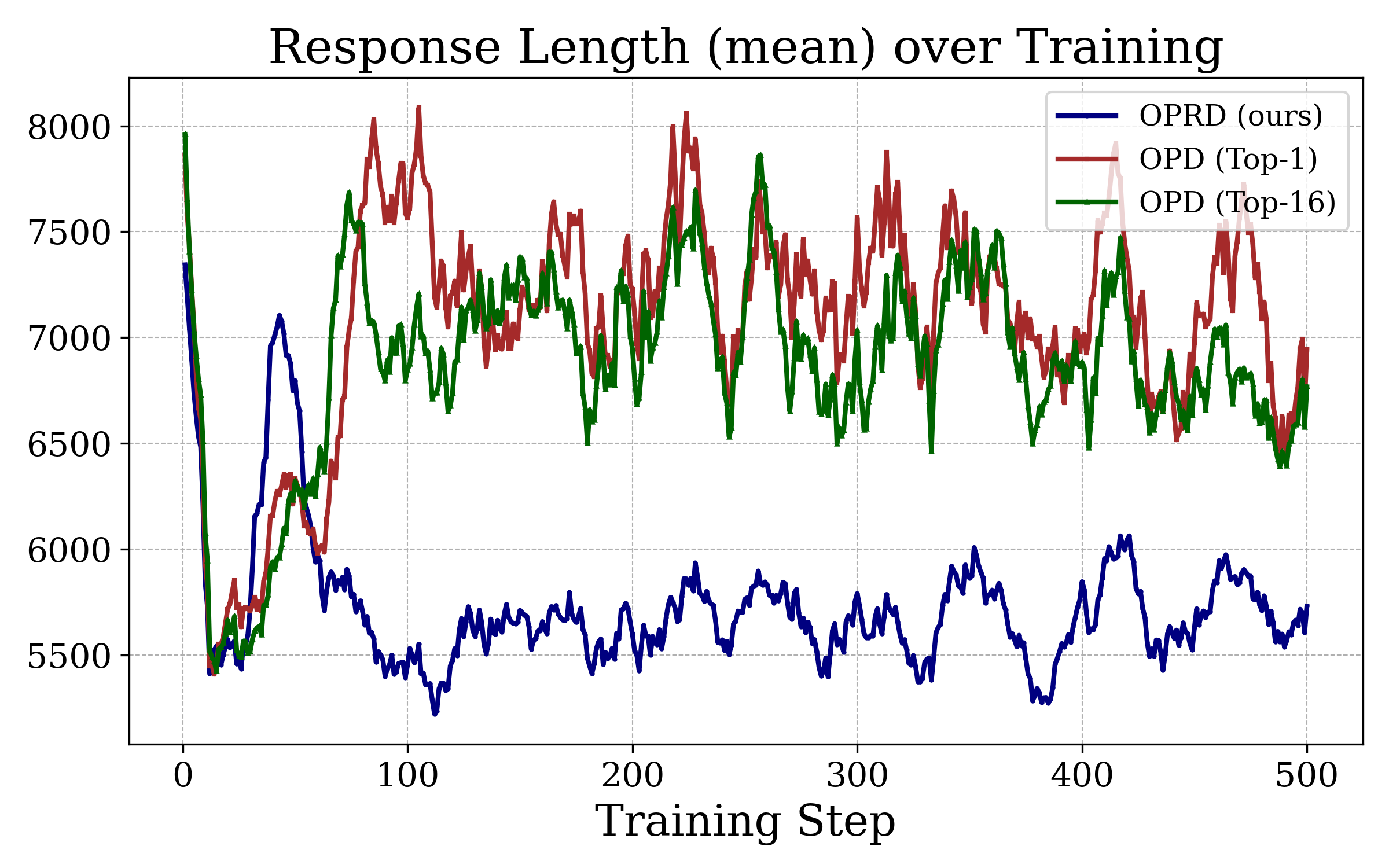}
  \caption{\textbf{OPRD produces shorter responses than OPD at higher accuracy.}
    Mean rollout length \texttt{response\_length/mean} along training for OPRD vs.\ OPD top-$1$ vs.\ OPD top-$16$ (smoothing window $= 15$).
    OPRD converges to ${\sim}5{,}700$ tokens per response, while both OPD variants plateau around ${\sim}7{,}000$ tokens, indicating that hidden-state supervision yields more concise and efficient reasoning chains.}
  \label{fig:resp_len}
\end{figure}

\textbf{(1) Both OPD variants improve over the student but plateau noticeably below the teacher.}
The student starts from a $17.9$-/$13.7$-/$17.3$-point gap to the teacher on AIME24/AIME25/AIMO.
OPD top-1 closes most of this gap on AIME25 (to within $2.1$ points) but leaves $8.5$ / $2.5$ points on AIME24/AIMO; enriching the supervision to OPD top-16 helps substantially on AIME24 ($+4.8$) and marginally on AIME25 ($+0.5$) yet \emph{loses} ground on AIMO ($-0.5$).
The absence of a clean ordering between top-1 and top-16 (more tokens in the loss is supposed to be strictly more informative) suggests that the output-space paradigm itself is the bottleneck: both variants are constrained by the LM-head information bottleneck (\Cref{thm:bottleneck}), and top-$k$'s truncation bias means that enriching the support does not guarantee monotonic improvement.

\textbf{(2) OPRD effectively closes the student--teacher gap.}
OPRD reaches $49.8$ on AIME24, $34.6$ on AIME25, and $79.1$ on AIMO, leaving only $1.0$ / $1.0$ / $0.4$ points to the teacher, all within the variance of $16$-sample Avg@16 evaluation, so the AIMO result is effectively a tie with the teacher (underlined in \Cref{tab:main}).
Relative to the better OPD baseline on each benchmark, OPRD gains $+2.7$ / $+0.6$ / $+2.1$ points; relative to the unmodified student it gains $+16.9$ / $+12.7$ / $+16.9$ points.
The advantage is most striking on AIMO, where OPRD recovers essentially all of the $17.3$-point student--teacher gap that no output-space variant fully bridges.
OPRD's gradient is conditionally deterministic (\Cref{thm:variance}), avoiding the late-stage variance collapse that limits OPD; it also exposes per-layer structural information that the LM-head projection compresses away (\Cref{thm:bottleneck}), supervising directions in $\mathcal{N}_W$ that any output-space objective treats as invisible.



\begin{figure}[!t]
  \centering
  \includegraphics[width=0.6\linewidth]{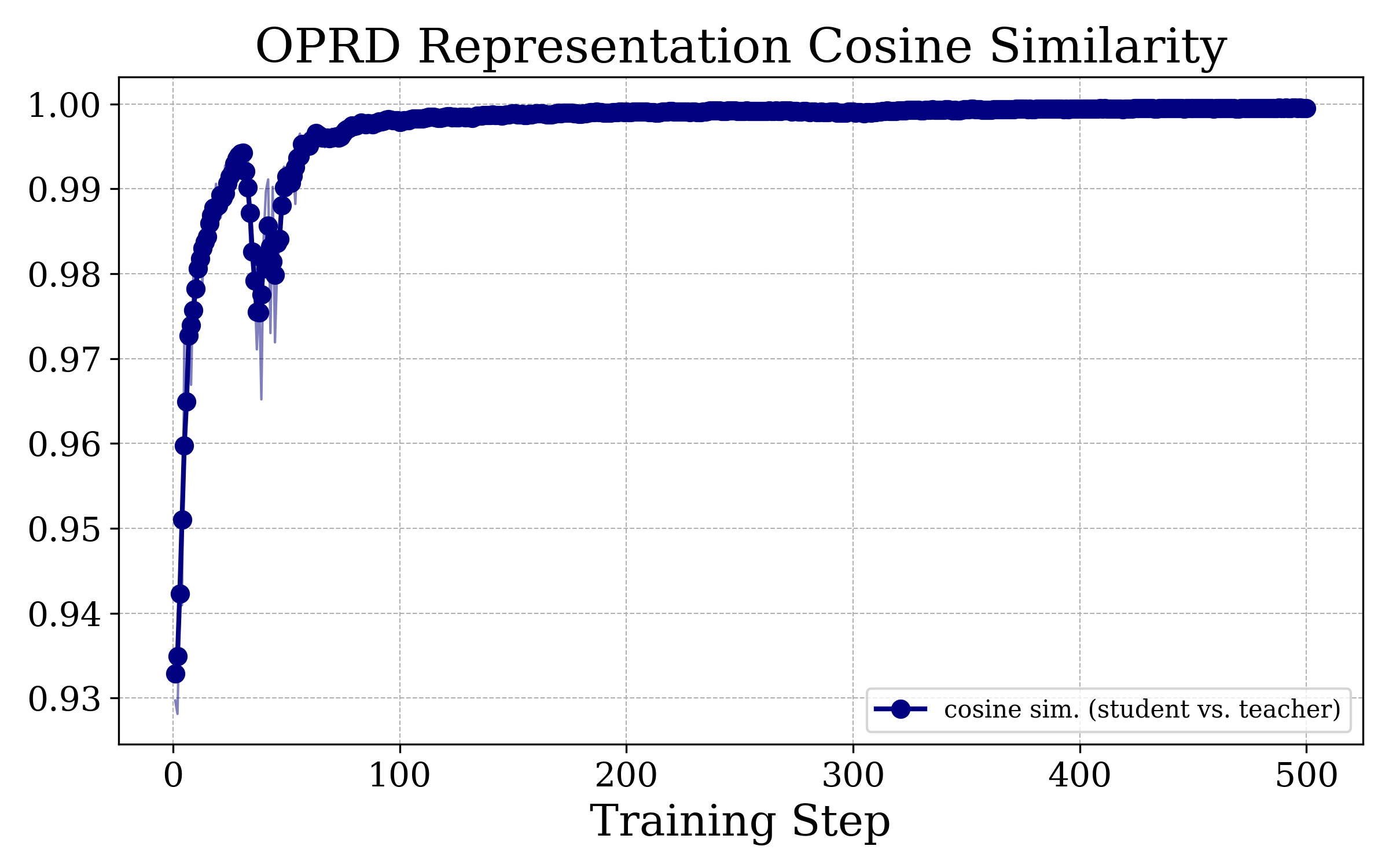}
  \caption{\textbf{OPRD monotonically increases the student--teacher representation cosine similarity it supervises (higher is better).}
    \texttt{rep/cosine\_similarity} on the OPRD-supervised positions along training (smoothing window $= 5$).
    The curve rises sharply early and drifts upward steadily thereafter, confirming that \eqref{eq:oprd_obj} is being optimised end-to-end.}
  \label{fig:rep_cos_train}
\end{figure}

\subsubsection{Training Dynamics}
\label{sec:exp_dynamics}

The end-of-training numbers in \Cref{tab:main} are only one slice of the story; we now examine \emph{how} each method gets there.
Three complementary views (per-step accuracy, response-length behaviour, and OPRD's own internal alignment metric) together paint a consistent picture of OPD stalling in the late-training regime predicted by \Cref{thm:variance}, while OPRD continues to make progress.


\paragraph{Accuracy curves: OPRD climbs monotonically, OPD plateaus.}
\Cref{fig:training_curves} compares OPRD step-by-step against OPD top-$1$ (top row) and OPD top-$16$ (bottom row) on all three benchmarks; raw curves are drawn at $\alpha\!=\!0.5$ and the solid curve with markers is the $5$-step centred rolling mean.
The two methods in each panel share the same initialisation and quickly enter qualitatively different regimes: both OPD variants lift accuracy in the first few dozen steps but then \emph{plateau} or \emph{oscillate without further improvement}, whereas OPRD continues to climb essentially monotonically until it reaches the teacher level.
Enriching the OPD supervision from top-$1$ to top-$16$ narrows the asymptotic gap to OPRD on AIME24 but does \emph{not} change the qualitative shape: OPD top-$16$ also plateaus, and on AIMO it does so $\sim\!2.6$ points below OPRD despite passing strictly more output-distribution information into the loss.
This is the SNR-collapse prediction of \Cref{thm:variance} in pictures: as $p_t \to q_t$, the OPD gradient's signal-to-noise ratio collapses and additional output-layer information cannot rescue the per-token sampling noise; only OPRD's deterministic, hidden-state-level signal continues to make progress.

\paragraph{Behavioural view: OPRD produces shorter, more efficient reasoning.}
The accuracy curves answer \emph{whether} a method keeps improving; a complementary question is \emph{how} the policy changes.
\Cref{fig:resp_len} reports the mean rollout length \texttt{response\_length/mean} for the same three runs.
OPRD converges to a mean response length of ${\sim}5{,}700$ tokens, substantially shorter than the ${\sim}7{,}000$ tokens produced by both OPD variants.
Combined with OPRD's higher accuracy (\Cref{tab:main}), this indicates that hidden-state supervision guides the student toward more concise reasoning chains: the student learns to reach the correct answer with fewer tokens rather than relying on longer, less directed exploration.
This also translates to a practical inference-time efficiency gain, since shorter responses require proportionally less compute at deployment.



\begin{wrapfigure}{r}{0.4\linewidth}
  \vspace{-1.2em}
  \centering
  \includegraphics[width=\linewidth]{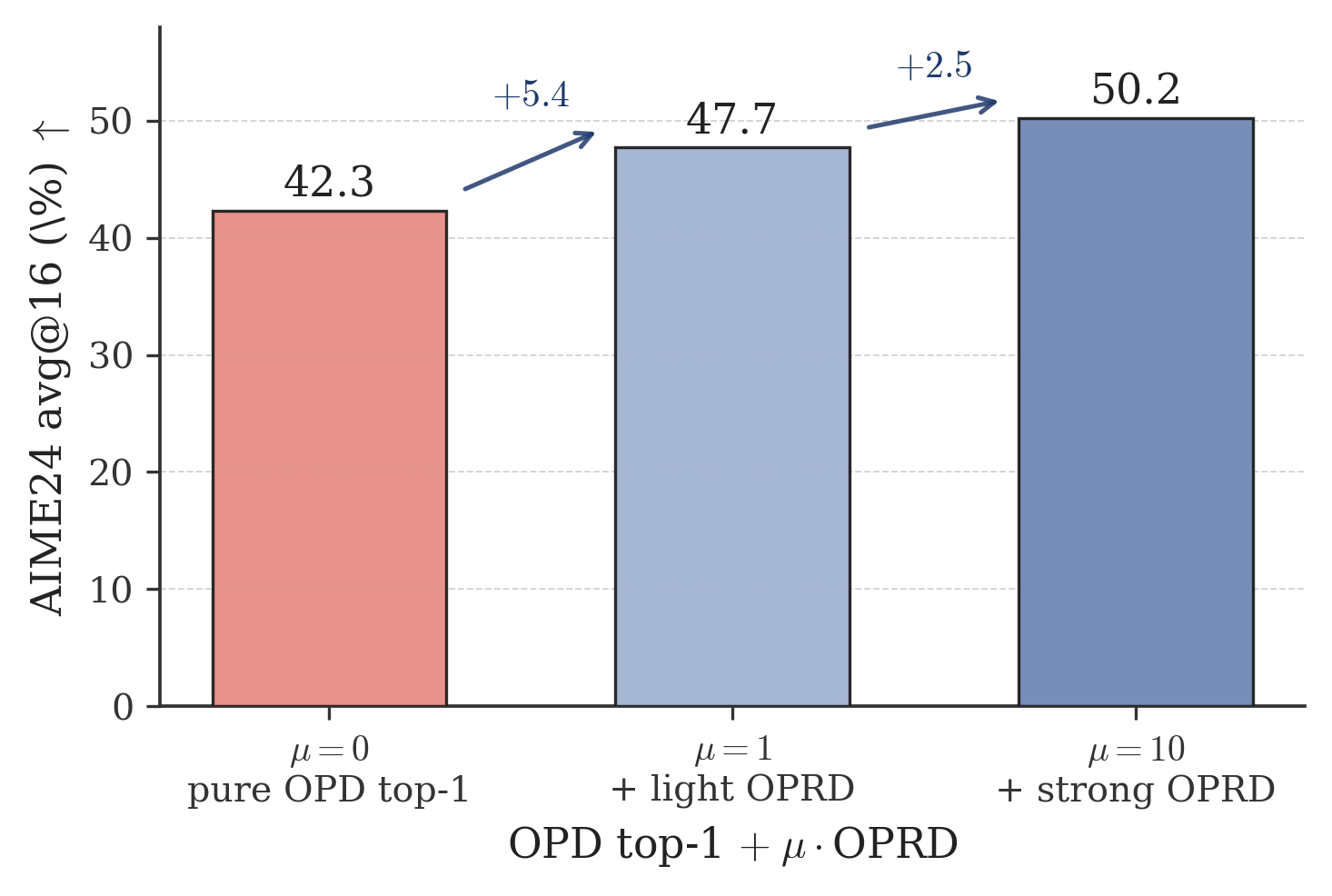}
  \caption{\textbf{Adding OPRD on top of OPD top-1 monotonically lifts accuracy.}
    AIME24 avg@16 of $\mathcal{L}_{\mathrm{OPD}} + \mu\cdot\mathcal{L}_{\mathrm{OPRD}}$ for $\mu\!\in\!\{0,1,10\}$. Even $\mu\!=\!1$ already surpasses OPD top-$16$ ($47.1$); $\mu\!=\!10$ closes the gap to teacher to within $0.6$\,pt.}
  \label{fig:ablation_mu}
  \vspace{-1em}
\end{wrapfigure}

\paragraph{Internal view: OPRD's own loss is being optimised end-to-end.}
A final, internal diagnostic is whether the representation-level loss OPRD is supposed to minimise actually decreases along training.
\Cref{fig:rep_cos_train} plots \texttt{rep/cosine\_similarity}, the cosine similarity between $\pi_\theta$'s and $\pi_T$'s hidden states averaged across all transformer layers and OPRD-supervised positions, for the OPRD-only run from \Cref{tab:main}.
The curve rises sharply in the first few dozen steps and then drifts upward steadily for the rest of training. 
Two consequences follow:
(i)~the OPRD objective is well-conditioned for end-to-end optimisation at this scale: the gradient produced by \eqref{eq:oprd_obj} is consistent enough to monotonically pull the supervised hidden states towards the teacher's;
(ii)~the downstream gains of \Cref{tab:main} are matched by a corresponding internal trend: OPRD is improving on \emph{exactly} the quantity its loss is defined on, confirming that the improvement, not a coincidental rollout-distribution shift, drives the gains.

\subsubsection{Efficiency}
\label{sec:exp_efficiency}

The OPRD loss path is computed entirely \emph{before} the LM head: it never materializes the $[B, T, |\mathcal{V}|]$ logits tensor on the student side, never invokes the $|\mathcal{V}|$-way \texttt{log\_softmax}, and never backpropagates through $W_{\mathrm{head}} \in \mathbb{R}^{|\mathcal{V}| \times d}$ for the distillation term.
As a direct consequence, OPRD-only training is strictly cheaper than any output-space OPD variant at the same rollout/teacher budget.
\Cref{tab:efficiency} quantifies this on the same training configuration as our main results.

\begin{table}[!t]
  \centering
  \small
  \caption{%
    \textbf{Actor-update cost at $B\!=\!8$, $T\!=\!16384$, FSDP on $8\!\times\!$A100 (80\,GB).}
    For each method we instrument the \texttt{update\_policy} call with \texttt{torch.cuda.reset\_peak\_memory\_stats()} and \texttt{torch.cuda.max\_memory\_allocated()} on every rank and report the per-rank maximum.
    \textsc{$\Delta$peak}: the actor-update segment's peak \emph{above} its starting baseline, i.e., the transient memory induced by the distillation loss path alone, with always-resident parameters, optimizer states, and FSDP shards subtracted out.
    Because everything independent of the distillation objective (parameters, optimizer, FSDP plan, rollouts, teacher forward) is identical across rows, it cancels, making $\Delta$peak a direct apples-to-apples proxy for loss-path memory.
    \textsc{Wall-clock}: total training time for $500$ optimizer steps, excluding evaluation.
    All three methods share the same on-policy rollout, teacher forward pass, student next-token forward pass, and FSDP/optimizer setup; they differ only in the distillation loss path.
  }
  \label{tab:efficiency}
  \begin{tabular}{l rr}
    \toprule
    \textbf{Method} & \textbf{$\Delta$peak per GPU (GB)} & \textbf{500-step wall-clock (min)} \\
    \midrule
    OPD top-1                    & 30.2          & 813 \\
    OPD top-16                   & 45.0          & 812 \\
    \textbf{OPRD} (ours)         & \textbf{20.5} & \textbf{563} \\
    \midrule
    \textit{$\Delta$ vs.\ OPRD}  & \multicolumn{2}{l}{OPD top-1: $+9.7$ GB ($+47\%$ $\Delta$peak), $+250$ min ($+44\%$ time)} \\
                                 & \multicolumn{2}{l}{OPD top-16: $+24.5$ GB ($+120\%$ $\Delta$peak), $+249$ min ($+44\%$ time)} \\
    \bottomrule
  \end{tabular}
\end{table}

\textbf{Memory.}
The actor-update transient footprint ($\Delta$peak, the most direct proxy for the loss path's own cost since always-resident state is subtracted out) is $30.2$ GB for OPD top-1 and $45.0$ GB for OPD top-16, vs.\ only $20.5$ GB for OPRD, a $32$\% and $54$\% reduction, or equivalently a $1.47\times$ and $2.20\times$ ratio.
The gap is dominated by the $[B, T, |\mathcal{V}|]$ logits tensor (and its gradient buffer for top-$k$), which scales with $|\mathcal{V}|\!\approx\!151$K but is entirely absent in the OPRD-only loss path.
The roughly $10$--$25$\,GB of saved transient memory is hardware-relevant: on $80$\,GB-class accelerators it is enough to either enlarge the micro-batch or extend the context at the same hardware budget.

\textbf{Wall-clock.}
At identical schedules ($500$ steps, same rollout, same teacher forward pass), OPRD finishes in $563$ minutes vs.\ $813$ / $812$ minutes for OPD top-1 / top-16, a $31$\% wall-clock reduction, equivalent to a $1.44\times$ speed-up.
We attribute this to the fact that the two OPD variants take essentially the same time, consistent with the observation that the cost is dominated by the $[B, T, |\mathcal{V}|]$ matrix multiplication and \texttt{log\_softmax} rather than by the top-$k$ slicing itself.


\textbf{Putting it together.}
Combining \Cref{tab:efficiency} with the accuracy results in \Cref{tab:main}, OPRD strictly Pareto-dominates both OPD baselines on this benchmark suite: at $\sim\!69\%$ of the wall-clock and $46$--$68\%$ of the actor-update transient memory ($\Delta$peak), it reaches accuracies that are $+0.6$ to $+2.7$ points above the better OPD baseline and effectively close the gap to the teacher.
These efficiency gains are a \emph{secondary} consequence of OPRD's design (the primary motivation, as developed in \S\ref{sec:oprd_theory}, is informational; see \Cref{thm:bottleneck}), but they make OPRD a more economical training objective in practice as well.
We note that our current implementation reuses the existing OPD training framework without OPRD-specific infrastructure optimisation (e.g., the teacher still computes and discards the full logits tensor even though OPRD does not consume it).
With a dedicated implementation that eliminates these redundant computations, we expect both peak memory and wall-clock to decrease further.

\subsubsection{Mechanistic Analysis}
\label{sec:exp_mechanism}

\begin{figure}[!t]
  \centering
  \includegraphics[width=0.65\linewidth]{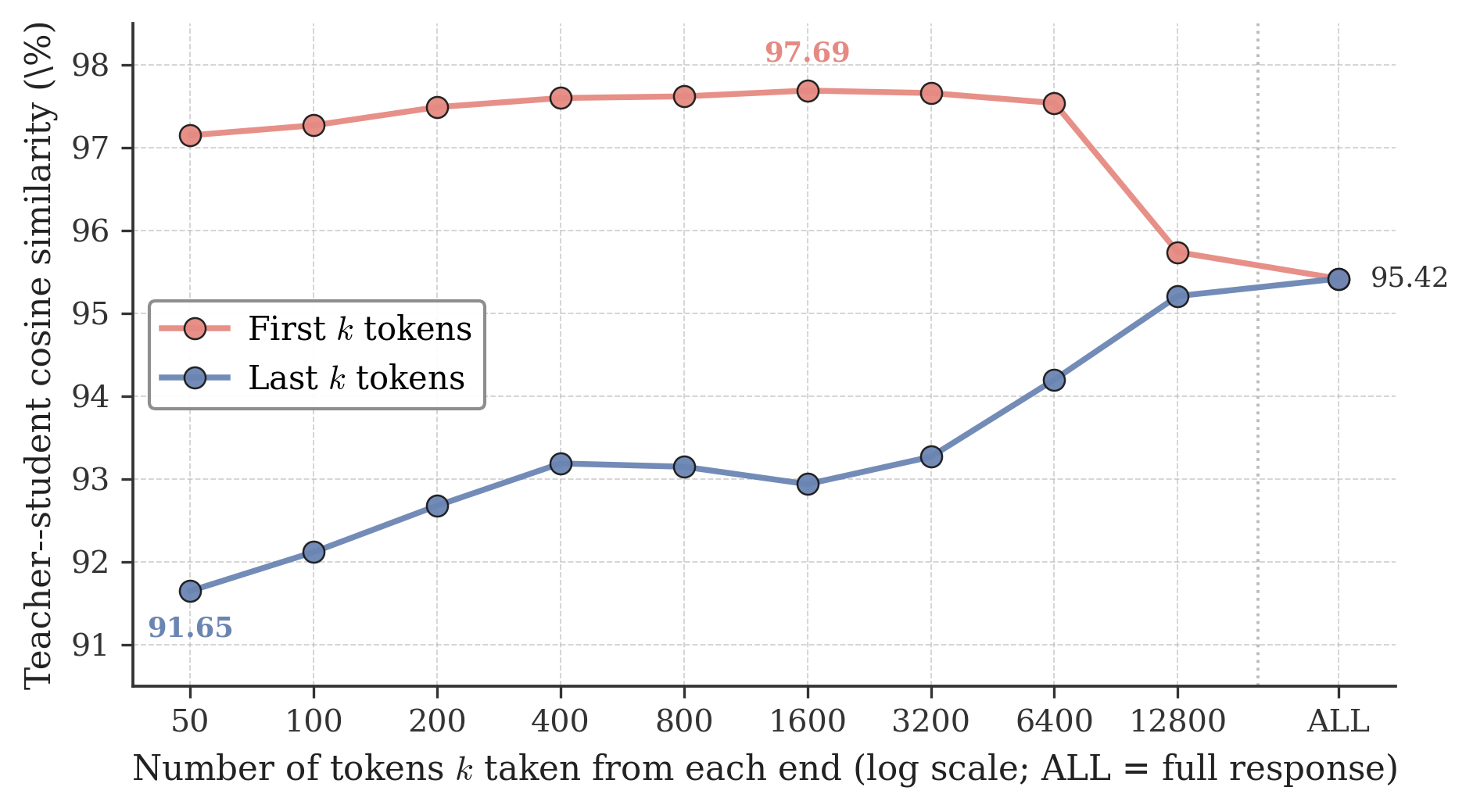}
  \caption{\textbf{The student diverges from the teacher mostly at the end of the response.}
    Cosine similarity between student (\texttt{R1-distill-1.5B}) and teacher (\texttt{JustRL-1.5B}) last-layer hidden states on on-policy rollouts, restricted to either the first $k$ or the last $k$ response tokens, as a function of $k$ (log scale; ``ALL'' = full response, at which both curves coincide at $95.42\%$ by construction).
    The first-$k$ curve is nearly teacher-aligned at every $k$ ($\geq\!97\%$ for $k\!\le\!1600$); the last-$k$ curve lags by $\geq\!4$ points until $k$ exceeds the full response length.
    This empirically motivates concentrating OPRD's supervision on the last-$k$ positions (\S\ref{sec:exp_setup}).}
  \label{fig:cos_position}
\end{figure}

The experiments above show \emph{that} OPRD outperforms OPD; we now ask \emph{why}.
We first study the effect of composing OPRD with OPD via the mixing weight $\mu$, and empirically motivate the choice of supervised positions.
We then track three complementary diagnostics along training for the composite runs $\mathcal{L}_{\mathrm{OPD}} + \mu\cdot\mathcal{L}_{\mathrm{OPRD}}$ with $\mu\!\in\!\{0,1,10\}$ to reveal a consistent mechanistic picture: OPRD pre-aligns the student's hidden states to the teacher's, which propagates back to (a)~a smaller residual policy-gradient signal, (b)~higher next-token top-$k$ agreement, and (c)~a student exploration distribution whose shape matches the teacher's.

\paragraph{Composing OPD with OPRD ($\mu$ sweep).}
\Cref{eq:oprd_opd_combined} suggests that OPRD can also be \emph{added on top of} existing OPD objective, rather than used as a standalone replacement.
We test this for the simplest output-space baseline, sampled-token OPD (i.e.\ OPD top-$1$), by training the composite loss
$\mathcal{L}_{\mathrm{OPD}} + \mu \cdot \mathcal{L}_{\mathrm{OPRD}}$
for $\mu \in \{0, 1, 10\}$, keeping all other knobs identical to \S\ref{sec:exp_setup}.
\Cref{fig:ablation_mu} shows that AIME24 avg@16 rises \emph{monotonically} with $\mu$: from the vanilla OPD top-$1$ baseline at $42.3$ ($\mu\!=\!0$), to $47.7$ with a light OPRD contribution ($\mu\!=\!1$, $+5.4$\,pt, already exceeding the OPD top-$16$ baseline of $47.1$ from \Cref{tab:main}), to $50.2$ with a stronger contribution ($\mu\!=\!10$, $+2.5$\,pt further, essentially matching the teacher's $50.8$).
The trend confirms two things:
(i)~the hidden-state signal that OPRD exposes is \emph{additive} to the output-space signal that OPD already uses, consistent with the information-bottleneck view of \Cref{thm:bottleneck}; and
(ii)~the improvement is monotonic in $\mu$ within the swept range, so the composition is robust to the mixing weight and does not require careful tuning.
We therefore view $\mathcal{L}_{\mathrm{OPD}} + \mu \cdot \mathcal{L}_{\mathrm{OPRD}}$ as a drop-in upgrade for existing OPD pipelines.


\begin{figure}[!t]
  \centering
  \includegraphics[width=0.6\linewidth]{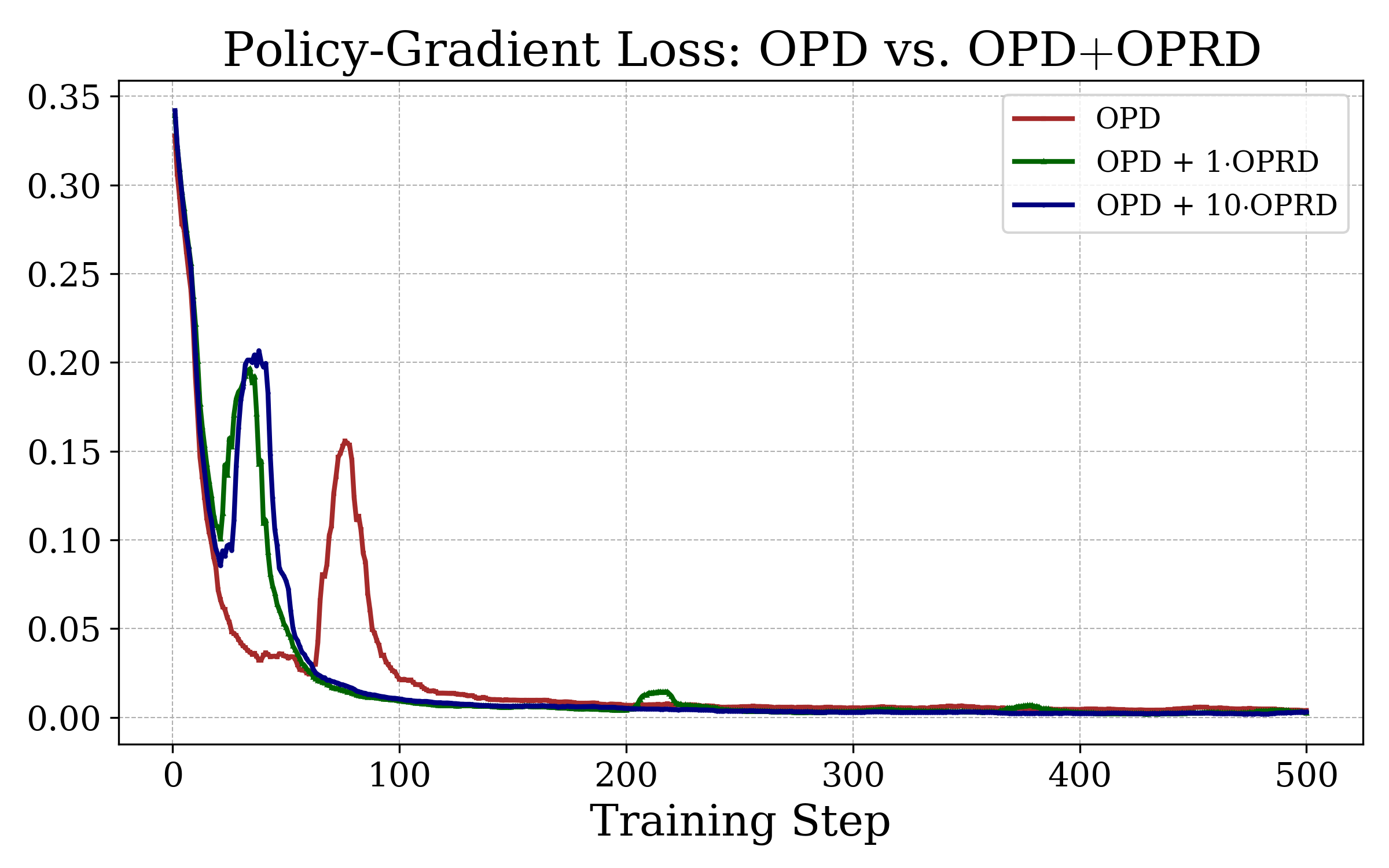}
  \caption{\textbf{OPRD accelerates the PG-loss phase transition and validates the information bottleneck.}
    \texttt{actor/pg\_loss} along training for OPD top-$1$ $+$ OPRD composite runs ($\mathcal{L}_{\mathrm{OPD\,top\text{-}1}} + \mu\cdot\mathcal{L}_{\mathrm{OPRD}}$, $\mu\!\in\!\{0,1,10\}$; smoothing window $= 15$).
    All runs show a loss spike (possible phase transition); OPRD shifts it earlier, indicating accelerated distillation.
    In late training all curves converge to ${\approx}\,0$, yet accuracy differences persist, corroborating the LM-head bottleneck of \Cref{thm:bottleneck}.}
  \label{fig:pgloss_mu}
\end{figure}

\begin{figure}[!t]
  \centering
  \includegraphics[width=0.6\linewidth]{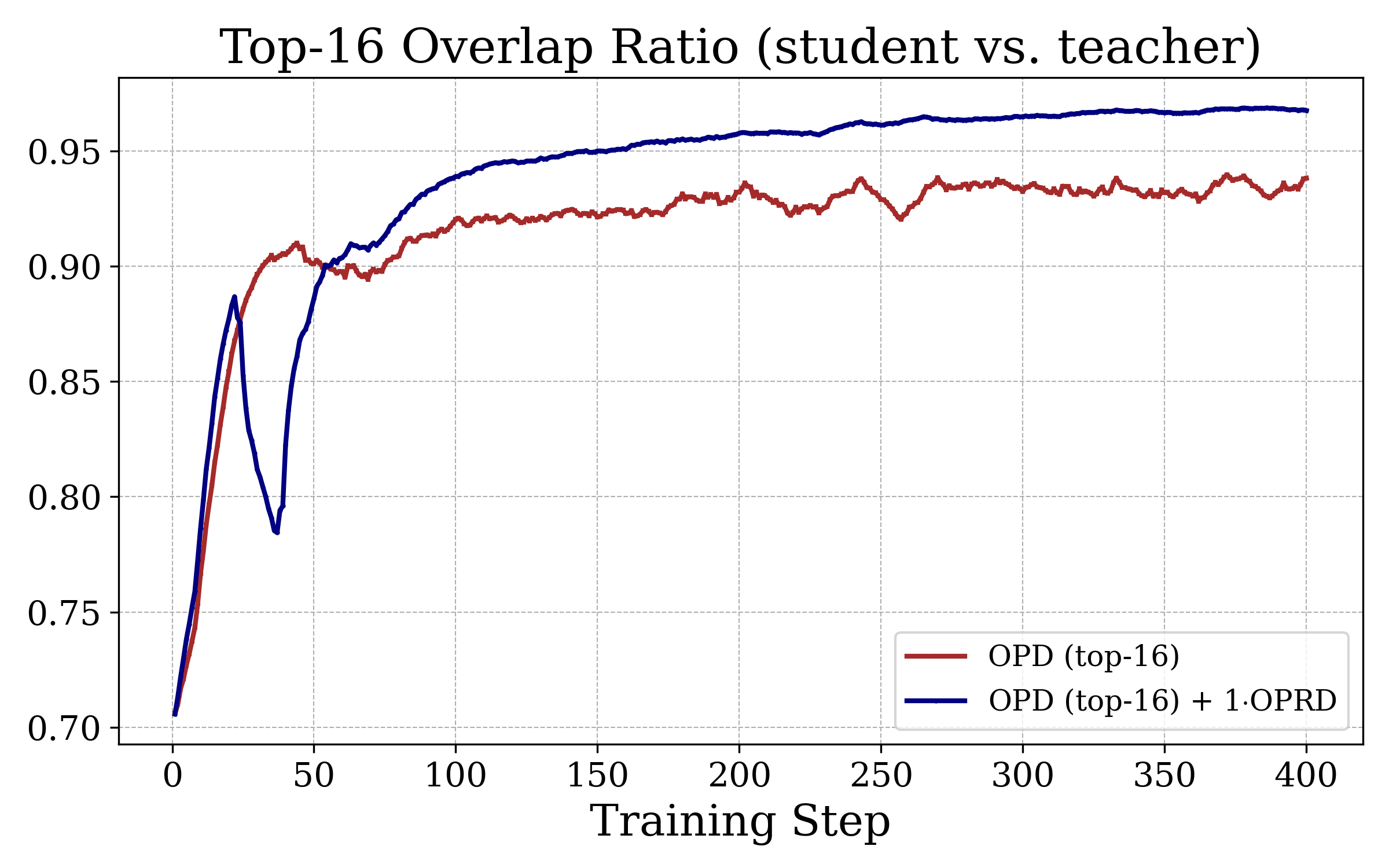}
  \caption{\textbf{Adding OPRD to OPD top-$16$ further aligns student and teacher next-token top-$16$ sets (higher is better).}
    Validation \texttt{val-topk/overlap\_ratio} along training.
    The two runs are nearly co-located early on, but in the second half of training OPD top-$16$ plateaus while OPD top-$16$ $+$ OPRD keeps climbing, the same late-stage divergence that distinguishes the accuracy curves of \Cref{fig:training_curves}.}
  \label{fig:overlap_mu}
\end{figure}

\paragraph{Where does the student diverge from the teacher? (motivation for last-$k$ supervision).}
A natural design question for OPRD is \emph{which response positions} the projector $\mathcal{P}(\hat{y})$ in \eqref{eq:oprd_obj} should select.
We answer this empirically by directly measuring \emph{where} along the response the student and teacher representations still disagree.
For the student initialisation $\pi_\theta^{(0)}\!=\!$~\texttt{R1-distill-1.5B} and the teacher $\pi_T\!=\!$~\texttt{JustRL-1.5B}, we sample on-policy rollouts from the student, run both models forward on each rollout, and compute the cosine similarity between their last-layer hidden states, restricted to either the \emph{first} $k$ or the \emph{last} $k$ tokens of every response.
\Cref{fig:cos_position} reports this similarity as a function of $k$.

Two patterns emerge.
\textbf{(i)~The early response is already teacher-aligned.}
The first-$k$ curve stays above $97\%$ for every $k\!\le\!1600$ and peaks at $97.69\%$ at $k\!=\!1600$, meaning the prompt-following preamble and the opening of the chain-of-thought are essentially already matched by the student; there is little headroom for representation-level supervision to act on.
\textbf{(ii)~The late response is where the gap lives.}
The last-$k$ curve starts at only $91.65\%$ for $k\!=\!50$ and remains $\geq\!4$ points below the first-$k$ curve until $k$ approaches the full response length, at which point both curves converge to the whole-sequence similarity of $95.42\%$ by construction.
Almost all of the student--teacher representational disagreement is concentrated in the \emph{tail} of the response, precisely where the chain-of-thought commits to a final answer.

This directly motivates our default choice $\mathcal{P}(\hat{y})\!=\!\textsc{last-}k$ in \S\ref{sec:exp_setup}: supervising the last $k$ tokens targets exactly the positions in which the student still deviates from the teacher, while sparing compute on the early positions where the signal has already been absorbed.
It also explains why a small budget ($k\!=\!2000 \ll |\hat{y}|$ on average) suffices to recover the gains reported in \Cref{tab:main}, since OPRD's representation loss is not diluted across positions that carry no residual signal.



\begin{figure}[!t]
  \centering
  \begin{subfigure}[t]{0.32\linewidth}
    \centering
    \includegraphics[width=\linewidth]{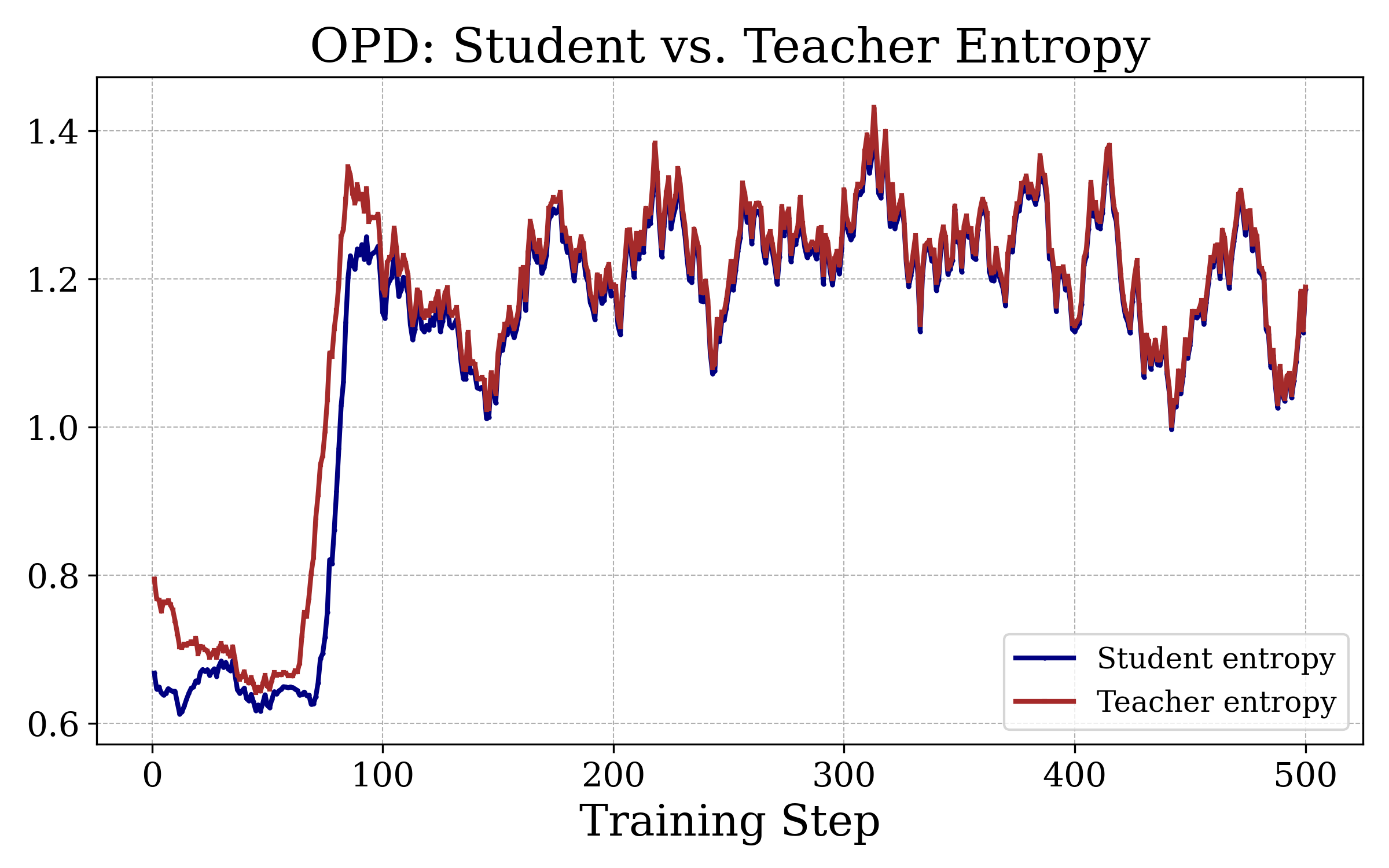}
    \caption{$\mu=0$ (OPD only).}
    \label{fig:entropy_mu0}
  \end{subfigure}\hfill
  \begin{subfigure}[t]{0.32\linewidth}
    \centering
    \includegraphics[width=\linewidth]{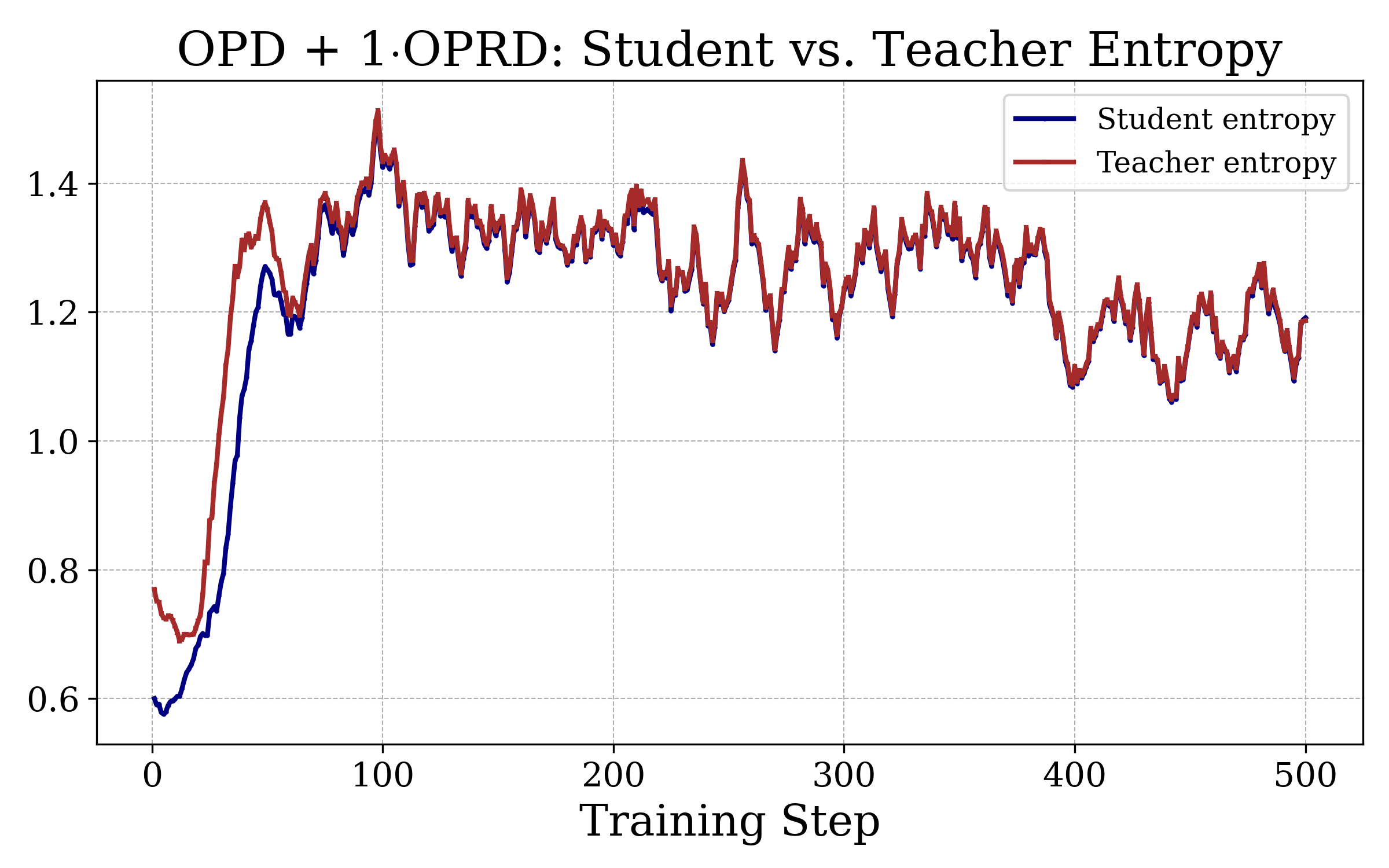}
    \caption{$\mu=1$ (OPD $+ 1\!\cdot\!$OPRD).}
    \label{fig:entropy_mu1}
  \end{subfigure}\hfill
  \begin{subfigure}[t]{0.32\linewidth}
    \centering
    \includegraphics[width=\linewidth]{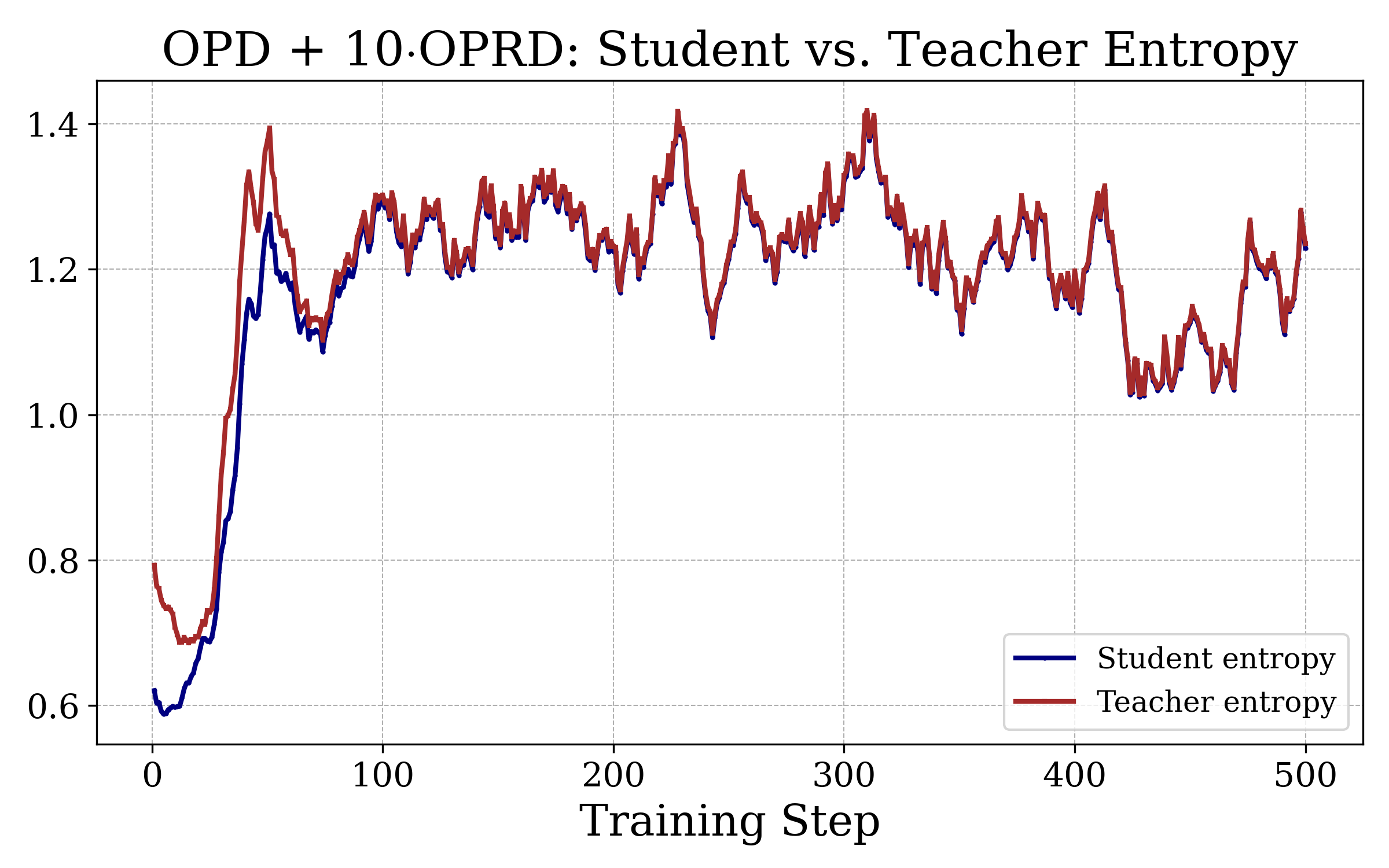}
    \caption{$\mu=10$ (OPD $+ 10\!\cdot\!$OPRD).}
    \label{fig:entropy_mu10}
  \end{subfigure}
  \caption{\textbf{OPRD accelerates entropy alignment between student and teacher.}
    Per-token entropy of $\pi_\theta$ (\texttt{actor/entropy}) and $\pi_T$ (\texttt{teacher/entropy}) on rollout positions along training for OPD top-$1$ $+$ OPRD composite runs ($\mu\!\in\!\{0,1,10\}$, left $\to$ right).
    All runs exhibit an early entropy-increase phase during which the student--teacher gap widens; adding OPRD shifts this phase earlier (coinciding with the PG-loss spike of \Cref{fig:pgloss_mu}), after which the student--teacher entropy gap narrows more rapidly.}
  \label{fig:entropy_mu}
\end{figure}

\paragraph{(a) Policy-gradient loss: OPRD accelerates distillation and validates the bottleneck theory.}
\Cref{fig:pgloss_mu} tracks \texttt{actor/pg\_loss} along training for the OPD top-$1$ $+$ OPRD composite runs with $\mu\!\in\!\{0,1,10\}$.
Two observations stand out.
\textbf{First}, all three runs exhibit a pronounced loss spike during training, likely reflecting a phase transition in the student's policy as it reorganises to absorb the teacher's behaviour (the precise mechanism is under active investigation).
Crucially, adding OPRD causes this spike to arrive \emph{earlier}: the $\mu\!=\!1$ and $\mu\!=\!10$ spikes precede the $\mu\!=\!0$ spike, indicating that hidden-state supervision accelerates the distillation dynamics.
\textbf{Second}, after the spike all three curves converge to approximately zero PG loss in late training, yet the accuracy gap persists ($+5.4$ and $+7.9$\,pt over $\mu\!=\!0$ on AIME24).
This directly corroborates \Cref{thm:bottleneck}: once the policy gradient vanishes ($p_t \approx q_t$), the output-space OPD signal can no longer drive further improvement because the remaining student--teacher gap lives in the null space $\mathcal{N}_W$ of the LM head; only OPRD's representation-level signal, which bypasses this bottleneck, continues to make progress.

\paragraph{(b) Top-$16$ overlap: hidden-state alignment propagates to next-token agreement.}
Following \citet{li2026rethinking}, who show that higher student--teacher top-$k$ overlap is a reliable predictor of distillation quality, we log \texttt{val-topk/overlap\_ratio}, defined as $|\text{top-}16(\pi_\theta)\cap\text{top-}16(\pi_T)|/16$ (higher is better; $1.0$ means the student's top-$16$ set is identical to the teacher's).
\Cref{fig:overlap_mu} compares OPD top-$16$ alone against OPD top-$16$ $+ 1\!\cdot\!\mathcal{L}_{\mathrm{OPRD}}$.

The OPD-only run increases the overlap nearly monotonically throughout training, but its rate of improvement visibly slows after mid-training.
The OPD$+$OPRD run behaves differently: it initially rises alongside OPD-only, then undergoes a sudden \emph{dip} in overlap (temporally coinciding with the PG-loss spike of \Cref{fig:pgloss_mu}, consistent with the hypothesised phase transition), after which it rebounds rapidly and surpasses the OPD-only curve by a clear margin.
The dip-then-surge pattern mirrors the PG-loss spike discussed above and is consistent with OPRD driving the student through a transient reorganisation that ultimately lands it in a higher-overlap regime than OPD alone can reach.
The representation-level and output-level supervisions are therefore not redundant: hidden-state alignment translates back into measurable improvement on exactly the metric OPD top-$16$ was designed to optimise.

\paragraph{(c) Predictive entropy.}
We log \texttt{actor/entropy} and \texttt{teacher/entropy} (per-token Shannon entropy of $\pi_\theta$ and $\pi_T$ on the same rollout positions) along training for the same OPD top-$1$ $+$ OPRD composite runs ($\mu\!\in\!\{0,1,10\}$).
The teacher's entropy curve serves as a reference: since $\pi_T$ is frozen, any drift is induced purely by the changing rollout distribution.
\Cref{fig:entropy_mu} shows the three runs side by side.
All three runs eventually bring the student's entropy into close agreement with the teacher's by the end of training (note that the teacher's entropy also drifts upward as the rollout distribution evolves, so alignment means tracking the teacher, not returning to a fixed level).
However, they differ markedly in their early dynamics: each run exhibits an entropy-increase phase in which the student--teacher gap widens before narrowing.
Adding OPRD causes this entropy-increase phase to begin \emph{earlier}, temporally coinciding with the PG-loss spike of \Cref{fig:pgloss_mu}: the $\mu\!=\!10$ onset precedes the $\mu\!=\!1$ onset, which in turn precedes the $\mu\!=\!0$ onset.
This is consistent with the picture that OPRD accelerates the student's internal reorganisation (the same phase transition visible in the PG-loss and overlap diagnostics), after which the student's entropy converges to the teacher's more quickly.


\subsection{OPRD-Bridge: Cross-Architecture and Cross-Tokenizer Distillation}
\label{sec:exp_bridge}

We now evaluate OPRD-Bridge in settings where direct hidden-state comparison is impossible without the bridge mechanism introduced in \S\ref{sec:oprd_bridge}: first in the cross-architecture setting, where the student and teacher differ in depth and width but share a vocabulary; then in the cross-tokenizer setting, where the two models additionally use completely disjoint tokenizers.

\subsubsection{Experimental Setup}
\label{sec:exp_bridge_setup}

\paragraph{Models.}
We use \texttt{Qwen3-4B}~\citep{yang2025qwen3} as the (frozen) teacher and \texttt{Qwen3-1.7B-Base}~\citep{yang2025qwen3} as the student.
The teacher has $L_T\!=\!36$ transformer layers with hidden dimension $d_T\!=\!2560$; the student has $L_S\!=\!28$ layers with $d_S\!=\!2048$.
Both share the same $|\mathcal{V}|\!\approx\!151$K vocabulary, but differ in depth ($29\%$ more layers) and width ($25\%$ wider hidden states), making na\"ive OPRD-Vanilla inapplicable ($d_S \neq d_T$, $L_S \neq L_T$).
In the cross-tokenizer experiment (\S\ref{sec:exp_cross_tokenizer}), we additionally test with \texttt{Phi-4-mini-reasoning}~\citep{xu2025phi} ($L_T\!=\!32$, $d_T\!=\!3072$, $|\mathcal{V}|\!\approx\!200$K, tiktoken-based tokenizer) as teacher, paired with the same \texttt{Qwen3-1.7B-Base} student, where the two models have completely disjoint tokenizers.

\begin{table}[t]
  \centering
  \small
  \caption{%
    \textbf{Cross-architecture distillation: accuracy.}
    Teacher: \texttt{Qwen3-4B} ($36$ layers, $d\!=\!2560$).
    Student: \texttt{Qwen3-1.7B-Base} ($28$ layers, $d\!=\!2048$).
    \textbf{Bold} = best among distillation methods.
    All methods share the same on-policy rollouts and teacher forward pass.
  }
  \label{tab:bridge_acc}
  \resizebox{\textwidth}{!}{%
  \begin{tabular}{l ccc c ccc c}
    \toprule
    & \multicolumn{4}{c}{\textbf{Avg@16 (\%)} $\uparrow$} & \multicolumn{4}{c}{\textbf{Best@16 (\%)} $\uparrow$} \\
    \cmidrule(lr){2-5} \cmidrule(lr){6-9}
    \textbf{Method} & {AIME24} & {AIME25} & {AIMO} & \textbf{Avg.} & {AIME24} & {AIME25} & {AIMO} & \textbf{Avg.} \\
    \midrule
    Teacher (\texttt{Qwen3-4B})                 & 23.1 & 25.2 & 59.4 & 35.9 & 50.0 & 40.0 & 80.7 & 56.9 \\
    Student (\texttt{Qwen3-1.7B-Base})          & 0.6 & 0.2 & 7.5 & 2.8 & 6.7 & 3.3 & 3.5 & 4.5 \\
    \midrule
    OPD top-1 (sampled-token)                   & 8.5 & 6.3 & 31.0 & \textbf{15.3} & 20.0 & 13.3 & 59.0 & 30.8 \\
    OPD top-16                                  & 6.7 & 4.0 & 20.5 & 10.4 & 13.3 & 6.7 & 61.4 & 27.1 \\
    \textbf{OPRD-Bridge} (ours, $r\!=\!8$)      & 5.0 & 3.8 & 24.8 & 11.2 & 20.0 & 13.3 & 59.0 & 30.8 \\
       $\hookrightarrow$ OPD top-1 (sampled-token)   & 9.6 & 3.0 & 28.2 & 13.6  & 20.0 & 16.7 & 65.1 & \textbf{33.9} \\
    \bottomrule
  \end{tabular}%
  }
\end{table}

\begin{table}[t]
  \centering
  \small
  \caption{%
    \textbf{Cross-architecture distillation: generation quality.}
    Same setting as \Cref{tab:bridge_acc}.
    Dist-4g: ratio of unique 4-grams to total 4-grams (higher = less repetition).
    Avg Len: mean response length in tokens.
  }
  \label{tab:bridge_quality}
  \resizebox{\textwidth}{!}{%
  \begin{tabular}{l ccc c ccc c}
    \toprule
    & \multicolumn{4}{c}{\textbf{Dist-4g} $\uparrow$} & \multicolumn{4}{c}{\textbf{Avg Len} $\downarrow$} \\
    \cmidrule(lr){2-5} \cmidrule(lr){6-9}
    \textbf{Method} & {AIME24} & {AIME25} & {AIMO} & \textbf{Avg.} & {AIME24} & {AIME25} & {AIMO} & \textbf{Avg.} \\
    \midrule
    Teacher (\texttt{Qwen3-4B})                 & 24.0 & 25.4 & 41.6 & 30.3 & 6541.5 & 4425.8 & 2384.2 & 4450.5 \\
    Student (\texttt{Qwen3-1.7B-Base})          & 2.9 & 5.3 & 3.5 & 3.9 & 15422.1 & 13122.7 & 15105.7 & 14550.2 \\
    \midrule
    OPD top-1 (sampled-token)                   & 6.2 & 6.4 & 15.6 & 9.4 & 19119.7 & 17931.3 & 11008.4 & 16019.8 \\
    OPD top-16                                  & 9.9 & 11.7 & 13.6 & 11.7 & 12222.0 & 10175.3 & 11833.3 & 11410.2 \\
    \textbf{OPRD-Bridge} (ours, $r\!=\!8$)      & 22.6 & 17.2 & 21.3 & \textbf{20.4} & 5909.4 & 5234.5 & 4855.3 & \textbf{5333.1} \\
       $\hookrightarrow$ OPD top-1 (sampled-token)   & 12.5 & 14.0 & 18.4 & 15.0  & 8855.2 & 6689.5 & 5960.9 & 7168.5 \\
    \bottomrule
  \end{tabular}%
  }
\end{table}

\paragraph{Bridge construction (Stage 1).}
We sample $2000$ prompts from DAPO-Math-17K, generate student rollouts, and collect both models' hidden states.
The teacher projector $P_T^{(l)} \in \mathbb{R}^{r \times d_T}$ is obtained via PCA (eigendecomposition of the hidden-state covariance) at each layer and permanently frozen.
The student projector $P_S^{(l)} \in \mathbb{R}^{r \times d_S}$ (linear) is trained for $20$ epochs to minimize Eq.~\eqref{eq:bridge_train} with both model backbones frozen, then also frozen.
Layer correspondence follows proportional spacing (Eq.~\eqref{eq:layer_mapping}).
Unless otherwise stated, we use rank $r\!=\!8$.

\paragraph{Distillation (Stage 2).}
On-policy prompts are drawn from DAPO-Math-17K.
For each prompt the student samples $2$ responses at temperature $1.0$ with max length $16{,}384$ tokens; global batch size is $8$ prompts per step.
We supervise all $28$ student layers and apply the position mask to the last $2000$ response tokens.
The loss is normalized MSE (L2-normalize both projected vectors before MSE).
By default we use OPRD-Bridge only.
All methods are trained with AdamW (peak learning rate $1\!\times\!10^{-5}$, cosine decay), bf16 mixed precision, and FSDP over $8\!\times\!$A100 (80G) GPUs.

\paragraph{Baselines.}
We compare against the unmodified \texttt{Qwen3-1.7B-Base} checkpoint (no distillation), OPD top-1 (sampled-token reverse KL), and OPD top-16.
Note that output-space OPD baselines \emph{can} be applied cross-architecture because they only require a shared vocabulary; they therefore serve as the natural comparison for OPRD-Bridge.

\paragraph{Evaluation.}
We evaluate at decoding temperature $0.7$ with $16$ independently sampled responses per prompt on three competition-level mathematical reasoning benchmarks: AIME~2024 ($30$ problems), AIME~2025 ($30$ problems), and AIMO (AMC 2022/2023, $83$ problems).
We report four metrics:
\textbf{Avg@16} (average accuracy across the $16$ samples),
\textbf{Best@16} (accuracy of the best sample, measuring the capability ceiling),
\textbf{Dist-4g} (ratio of unique 4-grams to total 4-grams, measuring lexical diversity),
and \textbf{Avg Len} (mean response length in tokens, measuring generation efficiency).
Final answers are extracted with the standard \texttt{boxed} parser and graded by exact-match.
We save checkpoints periodically and report results from the checkpoint that achieves the best average performance across AIME24, AIME25, and AIMO.

\subsubsection{Main Results}
\label{sec:exp_bridge_main}

\Cref{tab:bridge_acc,tab:bridge_quality} report the main cross-architecture distillation results.
Three observations emerge from different facets of the evaluation.

\textbf{(1) OPRD-Bridge matches OPD's capability ceiling (Best@16).}
While OPRD-Bridge's Avg@16 is lower than OPD top-1 on some benchmarks, its Best@16 tells a different story: on AIME24 and AIME25, OPRD-Bridge achieves $20.0$ and $13.3$, matching OPD top-1 exactly and substantially outperforming OPD top-16.
Best@16 reflects the \emph{upper bound} of the student's capability after distillation (the best answer the model can produce across $16$ attempts) and is therefore a more direct measure of how much teacher knowledge has been successfully transferred.
The parity in Best@16 indicates that OPRD-Bridge raises the student's capability ceiling to the same level as output-space distillation, despite operating through a rank-$8$ subspace rather than the full $|\mathcal{V}|$-dimensional vocabulary simplex.

\textbf{(2) OPRD-Bridge produces substantially more diverse reasoning (Dist-4g).}
OPRD-Bridge achieves Dist-4g scores of $22.6$/$17.2$/$21.3$ on AIME24/AIME25/AIMO, dramatically higher than both OPD top-1 ($6.2$/$6.4$/$15.6$) and OPD top-16 ($9.9$/$11.7$/$13.6$), and approaching the teacher's own diversity levels ($24.0$/$25.4$/$41.6$).
This indicates that hidden-state supervision through the bridge produces reasoning chains with far less repetition than output-space distillation.
The OPD baselines, by contrast, appear to induce repetitive patterns, likely because the high-variance output-space gradient causes the student to ``loop'' through similar token sequences rather than making steady forward progress in its reasoning.

\textbf{(3) OPRD-Bridge generates much shorter, more efficient responses (Avg Len).}
OPRD-Bridge converges to mean response lengths of $5{,}909$/$5{,}234$/$4{,}855$ tokens, which are $2$--$3\times$ shorter than OPD top-1 ($19{,}120$/$17{,}931$/$11{,}008$) and $1.5$--$2.4\times$ shorter than OPD top-16 ($12{,}222$/$10{,}175$/$11{,}833$).
Remarkably, OPRD-Bridge's response lengths are close to the teacher's own ($6{,}542$/$4{,}426$/$2{,}384$), suggesting that the bridge successfully transfers the teacher's concise reasoning style in addition to its knowledge.
Combined with the Best@16 parity, this means OPRD-Bridge achieves the same capability ceiling at a fraction of the inference cost, a significant practical advantage for deployment.

\textbf{Summary.}
The Avg@16 gap between OPRD-Bridge and OPD reflects a difference in \emph{consistency} (how often the model reaches its best answer), not in \emph{capability} (what the best answer is).
OPRD-Bridge compensates with dramatically better diversity and efficiency: it produces teacher-like reasoning chains that are concise, non-repetitive, and reach the same performance ceiling.
This profile (matching capability ceiling with superior efficiency) makes OPRD-Bridge particularly attractive for inference-constrained deployment scenarios where shorter, more reliable responses are preferred over longer, repetitive ones.

\textbf{(4) Composing OPRD-Bridge with OPD further raises the capability ceiling.}
The Avg@16 gap of OPRD-Bridge relative to OPD has a clear mechanistic explanation: OPRD-Bridge only trains the student backbone (hidden states), while the LM head $W_{\mathrm{head}}$ (the final mapping from hidden states to token distributions) receives no gradient.
Even when the backbone representations are well-aligned with the teacher's, the untrained LM head may not translate this alignment into precise per-token output distributions, reducing sampling consistency.
OPD, by contrast, directly supervises the output distribution and therefore calibrates the LM head implicitly through backpropagation.

This suggests a natural two-stage recipe: first use OPRD-Bridge to align the backbone representations, then fine-tune with OPD to calibrate the output distribution.
The last row of \Cref{tab:bridge_acc,tab:bridge_quality} (``$\hookrightarrow$ OPD top-1'') reports exactly this: we take the OPRD-Bridge checkpoint and continue training with OPD top-1.
The results confirm the complementarity:
\textbf{(i)}~Best@16 rises from $30.8$ to $\mathbf{33.9}$ (averaged across benchmarks), surpassing both standalone OPD top-1 ($30.8$) and OPD top-16 ($27.1$) and establishing a new best among all methods.
This validates that OPRD-Bridge pre-aligns the backbone into a better initialisation from which OPD can extract more capability.
\textbf{(ii)}~Avg@16 improves from $11.2$ to $13.6$, partially closing the gap to standalone OPD top-1 ($15.3$), consistent with the LM-head calibration hypothesis.
\textbf{(iii)}~The generation quality metrics shift toward a middle ground: Dist-4g decreases from $20.4$ to $15.0$ (still above OPD's $9.4$--$11.7$) and Avg Len increases from $5{,}333$ to $7{,}169$ (still well below OPD's $11{,}410$--$16{,}020$), reflecting the trade-off between output-level calibration and the concise reasoning style inherited from the bridge stage.

The two-stage pipeline thus combines the strengths of both approaches: OPRD-Bridge provides a high-quality backbone initialisation with teacher-like reasoning structure, and OPD fine-tunes the output distribution to improve consistency, yielding the highest capability ceiling overall.

\subsubsection{Ablation Studies}
\label{sec:exp_bridge_ablation}

We ablate the key design choices of OPRD-Bridge to understand their individual contributions.

\paragraph{Effect of bridge rank $r$.}

\begin{figure}[!t]
  \centering
  \includegraphics[width=0.6\linewidth]{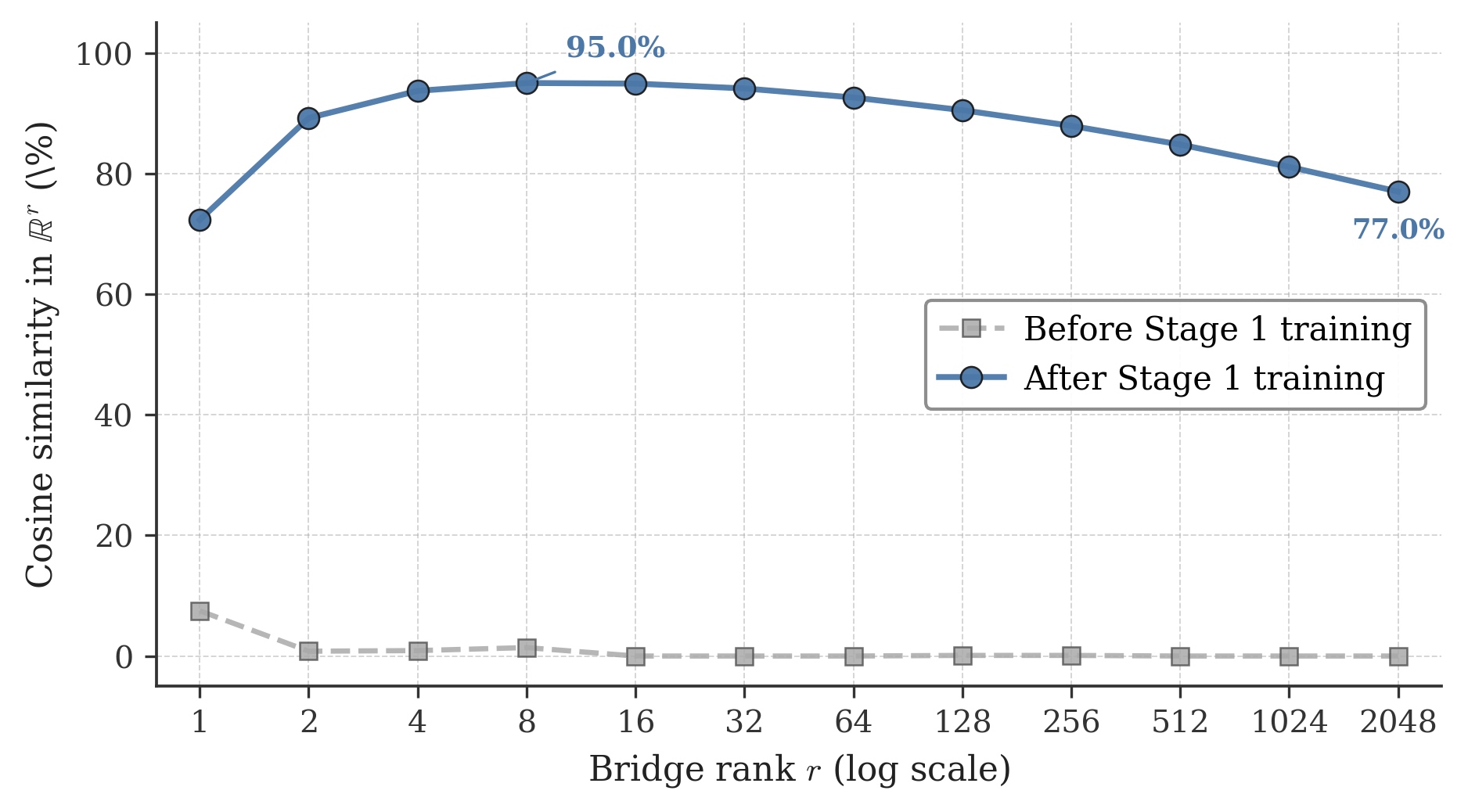}
  \caption{\textbf{Cosine similarity between student and teacher representations in the shared $\mathbb{R}^r$ subspace, before and after Stage~1 bridge training, as a function of rank $r$.}
    Before training (gray dashed), the two models' projected representations are essentially orthogonal at all ranks.
    After training (blue solid), similarity peaks at $r\!=\!8$ ($95.0\%$) and then \emph{decreases} monotonically as $r$ increases, reaching $77.0\%$ at $r\!=\!2048$ (full rank).
    The sweet spot at low rank validates the low-pass filter intuition: the top principal components capture the well-aligned ``low-frequency'' structure, while higher-rank components introduce architecture-specific noise that degrades alignment.}
  \label{fig:rank_cosine}
\end{figure}

\begin{table}[!t]
  \centering
  \small
  \caption{%
    \textbf{Cross-tokenizer distillation results.}
    Teacher: \texttt{Phi-4-mini-reasoning} (Microsoft, $|\mathcal{V}|\!\approx\!200$K).
    Student: \texttt{Qwen3-1.7B-Base} ($|\mathcal{V}|\!\approx\!151$K).
    The two models use completely different tokenizers.
    Standard token-level OPD relies on a shared vocabulary for alignment and collapses when this is unavailable.
  }
  \label{tab:cross_tokenizer}
  \resizebox{\columnwidth}{!}{%
  \begin{tabular}{l ccc c ccc c}
    \toprule
    & \multicolumn{4}{c}{\textbf{Avg@16 (\%)} $\uparrow$} & \multicolumn{4}{c}{\textbf{Best@16 (\%)} $\uparrow$} \\
    \cmidrule(lr){2-5} \cmidrule(lr){6-9}
    \textbf{Method} & {AIME24} & {AIME25} & {AIMO} & \textbf{Avg.} & {AIME24} & {AIME25} & {AIMO} & \textbf{Avg.} \\
    \midrule
    Teacher (\texttt{Phi-4-mini})     & 44.5 & 29.6 & 75.8 & 50.0 & 60.0 & 46.7 & 96.4 & 67.7 \\
    Student (\texttt{Qwen3-1.7B-Base})& 0.6  & 0.2  & 7.5  & 2.8  & 6.7  & 3.3  & 3.5  & 4.5  \\
    \midrule
    OPD  & \multicolumn{4}{c}{\textit{collapses}} & \multicolumn{4}{c}{\textit{collapses}} \\
    \textbf{OPRD-Bridge} ($r\!=\!4$)  & 4.0  & 1.0  & 10.0 & \textbf{5.0}  & 13.3 & 13.3 & 37.3 & \textbf{21.3} \\
    \bottomrule
  \end{tabular}%
  }
\end{table}

\Cref{fig:rank_cosine} reports the average cosine similarity between $P_S(h_S)$ and $P_T(h_T - \mu_T)$ across all layer pairs, before and after Stage~1 bridge training, for ranks $r \in \{1, 2, 4, 8, 16, 32, 64, 128, 256, 512, 1024, 2048\}$.
Three findings emerge.

\textbf{(i)~The bridge training is essential.}
Before Stage~1, the projected representations are nearly orthogonal at all ranks (cosine similarity $\leq 7.5\%$), confirming that the two models' hidden states are incommensurable without the learned alignment.
After training, similarity jumps to $72.3\%$ even at rank $1$ and exceeds $90\%$ for $r \geq 4$.

\textbf{(ii)~The optimal rank is low.}
Similarity peaks at $r\!=\!8$ ($95.0\%$) and then declines monotonically: $94.1\%$ at $r\!=\!32$, $87.9\%$ at $r\!=\!256$, and $77.0\%$ at full rank $r\!=\!2048$.
This validates the low-pass filter interpretation from \S\ref{sec:design_philosophy}: the top-$r$ principal components correspond to the ``low-frequency'' modes where teacher and student are highly aligned; beyond this sweet spot, additional dimensions introduce architecture-specific variation (``high-frequency noise'') that the linear projector cannot align, diluting the overall similarity.

\textbf{(iii)~The rank controls the signal-to-noise trade-off.}
Too small a rank ($r\!=\!1$--$2$) under-represents the shared structure, leaving useful information on the table.
Too large a rank ($r \geq 64$) forces the bridge to transmit misaligned directions, degrading the supervision signal.
The plateau at $r\!=\!4$--$16$ ($93.7$--$95.0\%$) defines the practical operating range; we default to $r\!=\!8$ throughout.

\subsubsection{Cross-Tokenizer Distillation}
\label{sec:exp_cross_tokenizer}

Since OPRD-Bridge performs alignment in representation space, it remains applicable even when teacher and student use entirely different tokenizers.
We verify this by distilling \texttt{Phi-4-mini-reasoning} (Microsoft, $|\mathcal{V}|\!\approx\!200$K, tiktoken-based BPE; $L_T\!=\!32$, $d_T\!=\!3072$) into \texttt{Qwen3-1.7B-Base} ($|\mathcal{V}|\!\approx\!151$K, Qwen BPE; $L_S\!=\!28$, $d_S\!=\!2048$).
All training details follow \S\ref{sec:exp_bridge_setup} exactly, except that we set bridge rank $r\!=\!4$.


\Cref{tab:cross_tokenizer} reports the results.
Two observations stand out.

\textbf{(1) Representation-space alignment bypasses the vocabulary barrier.}
Standard OPD, which relies on the shared vocabulary as its alignment channel, collapses when the two tokenizers are incompatible.
OPRD-Bridge instead aligns through the learned projector pair in representation space, lifting the student from near-zero performance ($2.8\%$ Avg@16) to meaningful accuracy ($5.0\%$).

\textbf{(2) Best@16 reveals substantial capability transfer.}
While Avg@16 gains are modest, Best@16 tells a more encouraging story: the student achieves $13.3\%$ on both AIME24 and AIME25 (up from $6.7\%$ and $3.3\%$), and $37.3\%$ on AIMO (up from $3.5\%$), representing a $4.7\times$ improvement in average Best@16 ($21.3\%$ vs $4.5\%$).
This indicates that the bridge successfully transfers the teacher's reasoning \emph{capability} into the student's backbone, even across a tokenizer boundary.
The gap between Best@16 and Avg@16 is consistent with the LM-head calibration hypothesis from \S\ref{sec:exp_bridge_main}: the backbone has acquired teacher-like representations, but the LM head (which was pre-trained with a different tokenizer's embedding table) has not been calibrated to translate these representations into consistent token-level outputs.

This experiment validates the core theoretical promise of OPRD-Bridge: by shifting alignment from the output space to the representation space, distillation becomes decoupled from the vocabulary, enabling knowledge transfer across incompatible tokenizers.

\section{Discussion}
\label{sec:discussion}

\paragraph{From representational alignment to cross-architecture distillation.}
The unifying insight behind OPRD is that the quality of hidden-state distillation is governed by \emph{representational alignment} between teacher and student (\S\ref{sec:oprd_bridge}).
When alignment is high (same architecture, shared initialisation), the full $\mathbb{R}^d$ channel is already an effective all-pass filter and OPRD-Vanilla suffices.
When alignment is low (different depth, width, or training history), OPRD-Bridge constructs a low-rank, low-pass channel that transmits the shared structure and suppresses architecture-specific noise.
Both regimes are instances of the same principle; the bridge rank $r$ simply controls the bandwidth of the supervision channel.

\paragraph{High-value application 1: multi-model RL merging.}
In the same-architecture setting (e.g., multi-model RL merging where all models share a backbone), OPRD-Vanilla addresses a pressing practical pain point.
In large-scale RL pipelines that merge multiple reward models or policy checkpoints, full-vocabulary OPD is the natural distillation objective but incurs prohibitive memory cost: materialising the $[B, T, |\mathcal{V}|]$ logit tensor for $|\mathcal{V}|$ demands extremely high transient GPU memory, often requiring extensive infrastructure modifications~\citep{deepseek2026deepseek}.
The common workaround, top-$k$ OPD, reduces memory but introduces a truncation bias (tail tokens are ignored) and remains subject to the LM-head information bottleneck analyzed in \Cref{thm:bottleneck}.
OPRD offers a third path: it simultaneously mitigates the variance problem (by providing a deterministic, hidden-state-level gradient) and dramatically reduces memory and wall-clock cost (by never materialising the vocabulary-sized tensor), making it an attractive drop-in component for multi-model RL consolidation.

\paragraph{High-value application 2: on-policy self-distillation (OPSD).}
OPRD is a natural fit for \emph{on-policy self-distillation}, where the teacher is constructed from the student itself by injecting privileged information (e.g., ground-truth solutions, step-level verification signals) into the prompt.
Because the teacher and student share exactly the same weights, the same-architecture requirement is satisfied by construction, and the hidden-state alignment signal is maximally informative.
In this setting OPRD can replace the reverse-KL computation in the output space with a cheaper and lower-variance representation-level objective, while retaining the full benefit of privileged-information guidance.

\paragraph{Toward alignment-aware pre-training.}
Our framework reveals a clear division of labour: OPRD-Vanilla exploits \emph{existing} representational alignment, while OPRD-Bridge \emph{constructs} it post hoc via a low-rank projection.
A natural question is whether the alignment itself could be established earlier, during pre-training, so that the bridge becomes unnecessary.
Concretely, if a large model (e.g., Qwen3-32B) and a small model (e.g., Qwen3-1.7B) were pre-trained with an explicit alignment objective, for instance through shared initialisation of common layers, periodic representation-matching regularisation, or co-distillation, then the resulting pair would already inhabit a well-aligned representation space despite differing in depth and width.
Post-training distillation could then proceed via OPRD-Vanilla (with a simple dimension adapter when $d_S \neq d_T$), bypassing the bridge entirely: the full $\mathbb{R}^{d_S}$ channel would serve as a high-quality all-pass filter, combining the fidelity of full-vocabulary distillation with the efficiency of representation-level supervision.
In the language of our filter analogy, alignment-aware pre-training would shift the ``cutoff frequency'' upward, turning directions that are currently noise (and must be filtered out by the bridge) into signal.
This perspective reframes the bridge not as a permanent architectural component but as a \emph{compensator} for misalignment that better pre-training could eliminate.

\section{Related Work}
\label{sec:related_work}

OPRD draws on three main lines of research: classical knowledge distillation, on-policy distillation, and feature-level / intermediate-representation distillation.
We also discuss adjacent work on auxiliary losses and capacity-gap analyses.
Throughout, we emphasize how OPRD differs from prior work that may at first appear similar.

\paragraph{Output-Space Knowledge Distillation.}
The idea of compressing a large model into a smaller one by matching their output distributions dates back to \citet{hinton2015distilling}.
In the sequence-modelling setting, \citet{kim2016sequence} showed that training a student on teacher-generated translations is an effective form of sequence-level knowledge transfer; subsequent work applied the same principle to pre-trained language models~\citep{sanh2019distilbert,jiao2020tinybert,wang2020minilm} and to instruction-following LLMs via supervised fine-tuning on teacher rollouts~\citep{chung2024scaling,sanh2021multitask,wei2021finetuned}.
A common thread across all these methods is that supervision is provided (i)~\emph{off-policy}, on data the student did not generate, and (ii)~exclusively in the \emph{output space}, at the LM-head logits or the softmax distribution derived from them.
The first property introduces exposure bias~\citep{bengio2015scheduled}; the second confines the learning signal to the ill-conditioned image of $W_{\mathrm{head}}$, leaving hidden-state deviations along its effective null space entirely unpenalised (\Cref{thm:bottleneck}).
OPRD departs from both properties simultaneously.

\paragraph{On-Policy Distillation.}
The exposure-bias problem motivated a shift toward on-policy training.
MiniLLM~\citep{gu2024minillm} optimised a reverse-KL objective on student-sampled responses via policy gradient, observing that the mode-seeking property of reverse KL discourages the student from placing mass where the teacher assigns low probability.
GKD~\citep{agarwal2024policy} generalised this to a family of divergences that interpolate between on- and off-policy data.
More recently, \citet{yang2026learning} reinterpreted OPD through the lens of KL-constrained RL, revealing that the teacher's per-token log-probability ratio serves as an implicit dense reward.
Building on these foundations, OPD has been adopted in several production post-training pipelines~\citep{deepseek2026deepseek,yang2025qwen3,zeng2026glm,xiao2026mimo,ko2026scaling,jin2026entropy,jang2026stable,fu2026revisiting,hou2026uni} and extended to self-distillation settings where the teacher is derived from the student itself via privileged information~\citep{hubotter2026reinforcement,zhao2026self,he2026far,shenfeld2026self,ye2026policy,sang2026crisp,kim2026does,ye2026online,yang2026self,li2026unifying,ding2026hdpo}.
A key observation, however, is that the entire design space explored so far (sampled-token, top-$k$, and full-vocabulary variants) concerns only \emph{how many output tokens} to supervise per position; the supervision itself never leaves the output space.
OPRD is, to our knowledge, the first on-policy method whose learning signal originates \emph{strictly before the LM head}, operating on the student's own trajectories.

\paragraph{Feature / Intermediate-Representation Distillation.}
A separate line of work supervises the student's intermediate representations rather than its outputs.
Early instances include FitNets~\citep{romero2014fitnets}, which match a single ``hint'' layer of the student to the teacher; attention-transfer~\citep{zagoruyko2016paying}, which matches per-pixel attention maps in CNNs; and FSP-matrix distillation~\citep{yim2017gift}, which matches Gram matrices between layers.
For BERT-style language models, TinyBERT~\citep{jiao2020tinybert} and MobileBERT~\citep{sun2020mobilebert} extend this idea by jointly matching hidden states and attention maps across all layers, and MiniLM/MiniLMv2~\citep{wang2020minilm,wang2021minilmv2} match self-attention relation matrices.
At first glance, OPRD may look like a straightforward port of these ideas to autoregressive LLMs, but two structural differences set it apart:

\begin{itemize}[leftmargin=14pt, itemsep=2pt]
    \item \textbf{On-policy vs.\ off-policy supervision.} FitNets, TinyBERT, MiniLM, and their successors compute the feature-matching loss on \emph{fixed} inputs from a pre-training or downstream corpus, i.e.\ inputs the student does not generate. The student is never exposed to its own rollout distribution during distillation, so exposure bias remains.
    OPRD, by contrast, computes the hidden-state loss on \emph{student-generated} sequences $\hat{y} \sim \pi_\theta(\cdot \mid x)$ that evolve as training progresses. The teacher is queried on states the student actually visits, making the supervision signal adaptive to the student's evolving policy.
    \item \textbf{Encoder representations vs.\ autoregressive prefix representations.} Prior feature-distillation work targets encoder models (BERT, vision CNNs) whose representations are computed once per input; the teacher and student process the same input and are aligned post-hoc.
    In the autoregressive LLM setting, each hidden state $h_{\cdot,t}^{(l)}$ (for either model) encodes the model's belief \emph{just before predicting token $\hat{y}_t$}, conditional on the entire sampled prefix $\hat{y}_{<t}$. OPRD therefore aligns the student's predictive computation at every decoding step under its own sampling distribution, a fundamentally on-policy object with no analog in encoder-style feature distillation.
\end{itemize}

\paragraph{Hint Learning, Auxiliary Losses, and Distribution Matching.}
A related body of work uses intermediate signals to regularize or augment training rather than to distill from a separate teacher.
Deeply-supervised nets~\citep{lee2015deeply} attach auxiliary classifiers to intermediate layers of a single model; DINO~\citep{caron2021emerging} aligns hidden states across augmented views of the same input in self-supervised learning; representation engineering~\citep{zou2023representation} steers or interprets hidden states without explicit teacher supervision.
OPRD shares the high-level intuition that intermediate representations carry useful signal, but differs in three crucial ways: it is (i)~explicitly teacher--student rather than self-supervised, (ii)~on-policy on student-generated autoregressive trajectories, and (iii)~a self-contained training objective that can additionally compose with any output-space OPD variant via Eq.~\eqref{eq:oprd_opd_combined}.


\section{Conclusion and Future Work}
\label{sec:conclusion}

We presented \textbf{OPRD}, the first on-policy distillation method that supervises the student in the hidden-state space rather than at the LM-head output.
The central thesis is that all existing OPD variants (sampled-token, top-$k$, and full-vocabulary) share two practical limitations inherent to the output-space paradigm: the dominant sampled-token variant suffers from a high-variance REINFORCE-style gradient estimator whose signal-to-noise ratio collapses as the student approaches the teacher, while top-$k$ and full-vocabulary variants trade this variance for a truncation bias or prohibitive memory cost; and all variants are subject to an LM-head projection that acts as an information bottleneck, compressing the teacher's full stack of intermediate hidden states through an ill-conditioned, softmax-invariant mapping.
By moving supervision from the output of the LM head to its input, OPRD yields a deterministic per-sample gradient that removes the token-level estimation variance by construction, and exposes per-position, per-layer structural information that any output-space objective necessarily discards.
Empirically, OPRD enables monotonic improvement throughout training and closes the student--teacher gap on three competition mathematics benchmarks (AIME 2024, AIME 2025, AIMO), while every output-space baseline plateaus several points below the teacher.
On the same hardware budget, OPRD is strictly Pareto-dominant: $1.44\times$ faster wall-clock training and up to $54\%$ less actor-update transient memory than top-$k$ OPD, because its loss path never materialises the $[B, T, |\mathcal{V}|]$ logits tensor.

We further introduced \textbf{OPRD-Bridge}, which extends representation distillation to the cross-architecture setting by constructing a frozen low-rank bridge between heterogeneous teacher and student representations.
By exploiting the empirical finding that models of different depth and width share a low-rank representational structure, OPRD-Bridge shifts alignment from the output space to the representation space, decoupling distillation from the vocabulary.
We validated this capability on both cross-architecture (Qwen3-4B $\to$ Qwen3-1.7B) and cross-tokenizer (Phi-4-mini-reasoning $\to$ Qwen3-1.7B) settings, demonstrating successful knowledge transfer even when the vocabulary-based alignment channel is unavailable.

\paragraph{Future Work.}
Several directions follow naturally from our framework:
\begin{itemize}[leftmargin=14pt, itemsep=2pt]
    \item \textbf{Beyond mathematical reasoning.} Our experiments focus on long-CoT math benchmarks. Whether OPRD's gains transfer to code generation, agentic interaction, and open-ended dialogue, each with different position-level supervision characteristics, remains an open question.
    \item \textbf{Cross-modal distillation.} Our cross-tokenizer experiment (\S\ref{sec:exp_cross_tokenizer}) validates that OPRD-Bridge transfers knowledge across incompatible tokenizers by aligning in representation space. The natural next step is cross-modal distillation (e.g., from a vision-language teacher to a language-only student), where the representation-space bridge may serve as the only viable supervision channel.
    \item \textbf{Adaptive layer and position selection.} We use uniform layer weighting and a simple last-$k$ position heuristic. Adaptively weighting layers and positions based on where the student--teacher gap is largest or where the gradient signal is most informative could further sharpen supervision.
    \item \textbf{On-policy representation self-distillation (OPRSD).} As discussed in \S\ref{sec:discussion}, OPRD is a natural fit for self-distillation with privileged information, where the same-architecture requirement is satisfied by construction. Scaling OPRSD to multi-turn and multi-task settings is a promising next step.
    \item \textbf{Understanding the phase transition.} Our mechanistic analysis (\S\ref{sec:exp_mechanism}) reveals a PG-loss spike and associated entropy/overlap dynamics when OPRD is active. Characterising the mechanism behind this transition would deepen the theoretical understanding of representation-level distillation.
    \item \textbf{Attention-map distillation.} OPRD aligns hidden-state vectors but does not supervise the attention patterns that produce them. Extending OPRD with an on-policy attention-matching objective could transfer the teacher's routing and composition behaviour more directly.
    \item \textbf{Alignment-aware pre-training.} Our analysis shows that representational alignment, not architectural identity, is the true prerequisite for direct hidden-state distillation. If pre-training itself enforced alignment between large and small models (e.g., through shared layer initialisation or periodic representation-matching regularisation), the resulting pair might not require a low-rank bridge at all, making cross-architecture OPRD as seamless as the same-architecture case.
    \item \textbf{Representation-level diagnostics for OPD.} By opening up the hidden-state channel, OPRD enables representation-level diagnostics (cosine similarity, CKA, probing accuracy) to be tracked alongside traditional output-level metrics, providing a more complete mechanistic picture of how knowledge transfers between models during RL training.
    \item \textbf{Tighter theoretical bounds.} Our analysis identifies the qualitative mechanisms behind OPRD's success. Quantifying these effects, including explicit convergence-rate bounds for OPRD vs.\ sampled-token OPD and spectral characterisations of which hidden-state directions the LM head nulls out, would solidify the theoretical foundation.
\end{itemize}

More broadly, our results suggest that hidden-state representations are an under-exploited resource in LLM distillation.
We hope this work encourages the community to treat the teacher not merely as a probability oracle but as a structured source of layered internal computation that the student can learn to inhabit.


\bibliographystyle{plainnat}
\bibliography{opd}

\newpage
\appendix


\section{Notation Summary}
\label{sec:notation_table}

\begin{table}[!t]
\centering
\small
\setlength{\tabcolsep}{4pt}
\renewcommand{\arraystretch}{1.15}
\caption{Notation summary (Part~I): models, architecture, distributions, and objectives.}
\label{tab:notation}
\begin{tabular}{@{}l P{8.6cm} l@{}}
\toprule
\textbf{Symbol} & \textbf{Meaning} & \textbf{First use} \\
\midrule
\multicolumn{3}{@{}l}{\emph{Models and inputs}} \\
$\pi_\theta,\; \theta$                  & Student policy and its trainable parameters & \S\ref{sec:notation} \\
$\pi_T$                                 & Teacher policy (frozen) & \S\ref{sec:notation} \\
$\mathcal{V},\; v,\; |\mathcal{V}|=V$   & Shared vocabulary, a token in it, and its size & \S\ref{sec:notation} \\
$x,\; \mathcal{D}_x$                    & Prompt and prompt distribution & \S\ref{sec:notation} \\
$\hat{y}=(\hat{y}_1,\ldots,\hat{y}_T)$  & On-policy rollout sampled from $\pi_\theta(\cdot\mid x)$ & \S\ref{sec:opd_framework} \\
$T,\; t$                                & Response length and a per-token position index & \S\ref{sec:opd_framework} \\
$B$                                     & Training batch size (number of prompts per optimizer step) & \S\ref{sec:exp_setup} \\
$\hat{y}_{<t}$                          & Prefix $(\hat{y}_1,\ldots,\hat{y}_{t-1})$ used to condition step $t$ & \S\ref{sec:notation} \\
\midrule
\multicolumn{3}{@{}l}{\emph{Architecture}} \\
$L,\; l$                                & Number of transformer layers, a layer index & \S\ref{sec:notation} \\
$d,\; d_S,\; d_T$                       & Hidden dimension; student/teacher hidden dimensions if different & \S\ref{sec:notation} \\
$h_{\theta,t}^{(l)},\; h_{T,t}^{(l)}\in\mathbb{R}^d$ & Student / teacher hidden state at layer $l$ and position $t$ & \S\ref{sec:notation} \\
$W_{\mathrm{head}}\in\mathbb{R}^{|\mathcal{V}|\times d}$         & Language-model head mapping hidden state to logits & \S\ref{sec:notation} \\
$W\in\mathbb{R}^{d_T\times d_S}$        & Learnable linear projector used when $d_S\!\neq\!d_T$ & \S\ref{sec:oprd_def} \\
\midrule
\multicolumn{3}{@{}l}{\emph{Distributions and divergences}} \\
$p_t,\; q_t$                            & Student / teacher next-token distribution at position $t$ & \S\ref{sec:opd_framework} \\
$D_{\mathrm{KL}}(p\,\|\,q)$             & Kullback--Leibler divergence (forward direction) & \S\ref{sec:opd_framework} \\
$u_t \triangleq \log p_t - \log q_t$    & Per-token log-density ratio & \S\ref{sec:oprd_theorems} \\
$\delta(\theta)\triangleq D_{\mathrm{KL}}(p\|q)+D_{\mathrm{KL}}(q\|p)$ & Symmetric divergence used in SNR analysis & \S\ref{sec:oprd_theorems} \\
\midrule
\multicolumn{3}{@{}l}{\emph{Output-space objectives (OPD variants)}} \\
$\mathcal{L}_{\mathrm{OPD}}$            & Generic on-policy distillation loss (any variant) & \S\ref{sec:opd_framework} \\
$\mathcal{L}_{\mathrm{OPD}}^{\mathrm{sample}}$ & Sampled-token OPD (single-sample Monte Carlo) & Eq.~\eqref{eq:opd_sampled_obj} \\
$\mathcal{L}_{\mathrm{OPD}}^{\mathrm{full}}$   & Full-vocabulary OPD (sum over all $v\in\mathcal{V}$) & Eq.~\eqref{eq:opd_full_obj} \\
$\mathcal{L}_{\mathrm{OPD}}^{\text{top-}k}$    & Top-$k$ OPD restricted to a $k$-token support & Eq.~\eqref{eq:opd_topk_obj} \\
$k,\; S_t$                              & Top-$k$ support size and the per-position support set & \S\ref{sec:opd_variants} \\
\midrule
\multicolumn{3}{@{}l}{\emph{Representation-level objective (OPRD-Vanilla, \S\ref{sec:oprd_def})}} \\
$\mathcal{L}_{\mathrm{OPRD}}$           & On-policy representation distillation loss & Eq.~\eqref{eq:oprd_obj} \\
$\mathcal{L}_{\mathrm{layer}}\subseteq\{1,\ldots,L\}$ & Set of distilled transformer layers & \S\ref{sec:oprd_def} \\
$\mathcal{P}(\hat{y})\subseteq\{1,\ldots,T\}$         & Set of supervised response positions & \S\ref{sec:oprd_def} \\
$m_t\in\{0,1\}$                         & Position mask: $m_t=\mathbf{1}[t\in\mathcal{P}(\hat{y})]$ & \S\ref{sec:oprd_def} \\
$\mathrm{sg}(\cdot)$                    & Stop-gradient operator (treats argument as constant) & Eq.~\eqref{eq:oprd_obj} \\
$\mu\ge 0$                              & Mixing weight of the optional OPD term added to OPRD & Eq.~\eqref{eq:oprd_opd_combined} \\
\bottomrule
\end{tabular}
\end{table}

\begin{table}[!t]
\centering
\small
\setlength{\tabcolsep}{4pt}
\renewcommand{\arraystretch}{1.15}
\caption{Notation summary (Part~II): cross-architecture bridge and theoretical analysis.}
\label{tab:notation2}
\begin{tabular}{@{}l P{8.6cm} l@{}}
\toprule
\textbf{Symbol} & \textbf{Meaning} & \textbf{First use} \\
\midrule
\multicolumn{3}{@{}l}{\emph{Cross-architecture bridge (OPRD-Bridge, \S\ref{sec:oprd_bridge})}} \\
$L_S,\; L_T$                            & Number of layers in student / teacher & \S\ref{sec:oprd_bridge} \\
$d_S,\; d_T$                            & Hidden dimension of student / teacher & \S\ref{sec:oprd_bridge} \\
$r$                                     & Bridge rank (shared subspace dimensionality) & \S\ref{sec:shared_structure} \\
$P_T^{(l)} \in \mathbb{R}^{r \times d_T}$ & Teacher-side PCA basis (frozen after eigendecomposition) & Eq.~\eqref{eq:teacher_pca} \\
$P_S^{(l)} \in \mathbb{R}^{r \times d_S}$ & Student-side projector (trained in Stage 1, then frozen) & Eq.~\eqref{eq:bridge_train} \\
$\mu_T^{(l)} \in \mathbb{R}^{d_T}$      & Per-layer mean of teacher hidden states (for centering) & Eq.~\eqref{eq:teacher_pca} \\
$\phi(l_S)$                             & Proportional layer mapping from student to teacher & Eq.~\eqref{eq:layer_mapping} \\
$\mathcal{L}_{\text{OPRD-Bridge}}$      & Cross-architecture on-policy representation distillation loss & Eq.~\eqref{eq:xoprd_obj} \\
\midrule
\multicolumn{3}{@{}l}{\emph{Gradient estimators and variance analysis}} \\
$g_{\mathrm{OPD}},\; g_{\mathrm{OPRD}}$ & Single-sample stochastic gradients of the two losses & Def.~\ref{def:estimators} \\
$\bar{g}_{\mathrm{OPD}},\; \bar{g}_{\mathrm{OPRD}}$ & Their population means (over $\hat{y}_t\sim p_t$) & Def.~\ref{def:estimators} \\
$\nabla_\theta\log p$                   & Score function (per-token policy gradient direction) & Eq.~\eqref{eq:score_decomp} \\
$\mathcal{F}(\theta),\; \mathcal{F}_{\min}(\theta)$ & Fisher information matrix and its minimum eigenvalue & Thm.~\ref{thm:opd_variance} \\
$\mathrm{SNR}(g)$                       & Signal-to-noise ratio $\|\bar g\|^2 / \mathrm{Tr}(\mathrm{Cov}[g])$ & Def.~\ref{def:snr} \\
\midrule
\multicolumn{3}{@{}l}{\emph{LM-head information bottleneck (\Cref{thm:bottleneck})}} \\
$\mathcal{N}_W$                         & Effective null space of LM head under softmax: $\{\Delta h : W_{\mathrm{head}}\Delta h \in \mathrm{span}\{\mathbf{1}\}\}$ & Thm.~\ref{thm:bottleneck} \\
$\mathbf{1}\in\mathbb{R}^{|\mathcal{V}|}$             & All-ones vector (softmax-invariant shift direction) & Thm.~\ref{thm:bottleneck} \\
$\sigma_1,\ldots,\sigma_d$             & Singular values of $W_{\mathrm{head}}$ ($\sigma_1\!\ge\!\cdots\!\ge\!\sigma_d\!>\!0$) & Thm.~\ref{thm:bottleneck} \\
$v_1,\ldots,v_d$                       & Right-singular vectors of $W_{\mathrm{head}}$ & Thm.~\ref{thm:bottleneck} \\
\bottomrule
\end{tabular}
\end{table}

\section{Formal Theoretical Guarantees}
\label{sec:oprd_theorems_appendix}
\label{sec:oprd_theorems}

The two theorems in \S\ref{sec:oprd_theory} (\Cref{thm:variance,thm:bottleneck}) establish OPRD's properties at an intuitive level.
We now state both formally, in one-to-one correspondence with the main conclusions stated in \S\ref{sec:oprd_theory}:
\begin{itemize}[leftmargin=14pt, itemsep=0pt]
  \item \S\ref{sec:oprd_theorems_variance} formalizes \textbf{\Cref{thm:variance} (gradient variance)}: variance gap (\Cref{thm:opd_variance,thm:oprd_zero_variance,cor:variance_gap}) and signal-to-noise collapse (\Cref{thm:snr_collapse}).
  \item \S\ref{sec:oprd_theorems_bottleneck} formalizes \textbf{\Cref{thm:bottleneck} (LM-head information bottleneck)}: the null-direction identity (\Cref{thm:bottleneck_null_appendix}) and the spectral gap (\Cref{thm:bottleneck_gap_appendix}).
\end{itemize}
Throughout this section, all expectations are taken over a single fixed prompt $x$ and a single response position $t$; the multi-position case follows by linearity.
We use $\theta$ to denote the student parameters, $\pi_\theta$ and $\pi_T$ to denote the student and teacher policies, and write $p \equiv p_t = \pi_\theta(\cdot \mid x, \hat{y}_{<t})$ and $q \equiv q_t = \pi_T(\cdot \mid x, \hat{y}_{<t})$ for brevity.
We let $u_t \triangleq \log p - \log q$ denote the per-token log-density ratio.

\subsection{Setup and Assumptions}

\begin{definition}[Stochastic gradient estimators]
\label{def:estimators}
Fix a student response position $t$ and a sampled-token estimator $\hat{y}_t \sim p$.
The two per-position stochastic gradient estimators considered in this paper are
\begin{align}
  g_{\mathrm{OPD}}(\theta; \hat{y}_t)
  &\triangleq \nabla_\theta \bigl[ \log p(\hat{y}_t) - \log q(\hat{y}_t) \bigr]
  = \nabla_\theta \log p(\hat{y}_t),
  \label{eq:thm_g_opd}
  \\
  g_{\mathrm{OPRD}}(\theta)
  &\triangleq \nabla_\theta \, \tfrac{1}{d}\,
              \bigl\| h_{\theta,t}^{(L)} - \mathrm{sg}(h_{T,t}^{(L)}) \bigr\|_2^2,
  \label{eq:thm_g_oprd}
\end{align}
where in \eqref{eq:thm_g_opd} we used the fact that $\nabla_\theta \log q(\hat{y}_t) = 0$ because $q$ depends only on the (frozen) teacher.
The corresponding population gradients are
$\bar{g}_{\mathrm{OPD}}(\theta) \triangleq \mathbb{E}_{\hat{y}_t \sim p}[g_{\mathrm{OPD}}]$ and
$\bar{g}_{\mathrm{OPRD}}(\theta) \triangleq g_{\mathrm{OPRD}}$ (which is already deterministic in $\hat{y}_t$).
\end{definition}

\begin{assumption}[Standard regularity]
\label{ass:regularity}
The following standard conditions hold throughout:
\textbf{(R1)} The log-densities $\log p_\theta(v)$ are twice continuously differentiable in $\theta$ for every $v \in \mathcal{V}$.
\textbf{(R2)} The score $s_\theta(v) \triangleq \nabla_\theta \log p_\theta(v)$ satisfies $\mathbb{E}_{p}[\|s_\theta\|_2^2] < \infty$ (finite Fisher information).
\textbf{(R3)} For each $v$, $|\log p_\theta(v) - \log q(v)| \le M$ for some constant $M < \infty$ on the trajectory of training (bounded log-ratio).
\textbf{(R4)} The hidden state $h_{\theta,t}^{(L)}$ is a continuously differentiable function of $\theta$ with bounded Jacobian: $\|\nabla_\theta h_{\theta,t}^{(L)}\|_{\mathrm{op}} \le J < \infty$.
\end{assumption}

Conditions (R1)--(R2) hold for any LLM with softmax output;
(R3) holds whenever both student and teacher assign nonzero probability to every supported token (e.g., after a small label-smoothing or temperature adjustment);
(R4) holds for any Lipschitz transformer with bounded weights.
These assumptions are mild and standard in the policy-gradient literature.

\phantomsection
\label{sec:oprd_theorems_variance}

\subsection{Variance of Sampled-Token OPD}

\begin{lemma}[Score-function decomposition of OPD gradient]
\label{lem:score_decomp}
Under \Cref{ass:regularity}, the OPD population gradient at position $t$ admits the score-function representation
\begin{equation}
  \bar{g}_{\mathrm{OPD}}(\theta)
  \;=\;
  \mathbb{E}_{\hat{y}_t \sim p}\bigl[\, u_t(\hat{y}_t)\,
       \nabla_\theta \log p(\hat{y}_t) \,\bigr],
  \qquad u_t(v) \triangleq \log p(v) - \log q(v).
  \label{eq:score_decomp}
\end{equation}
\end{lemma}

\begin{proof}
Starting from \eqref{eq:thm_g_opd}, write $\bar{g}_{\mathrm{OPD}} = \mathbb{E}_p[\nabla_\theta \log p(\hat{y}_t)]$.
The unconditional expectation of the score is zero:
\begin{equation*}
  \mathbb{E}_p[\nabla_\theta \log p(\hat{y}_t)]
  = \sum_v p(v) \,\nabla_\theta \log p(v)
  = \sum_v \nabla_\theta p(v)
  = \nabla_\theta \!\sum_v p(v)
  = \nabla_\theta 1 = 0.
\end{equation*}
Therefore $\bar{g}_{\mathrm{OPD}}$ vanishes, which would make it useless as a learning signal.
This apparent paradox is resolved by recognizing that what we actually optimize is the OPD \emph{loss surrogate} along its stochastic gradient, which by the REINFORCE identity satisfies
\begin{equation*}
  \nabla_\theta \,\mathbb{E}_{\hat{y}_t \sim p}\!\bigl[\log p(\hat{y}_t) - \log q(\hat{y}_t)\bigr]
  = \mathbb{E}_p\!\bigl[(\log p - \log q)\,\nabla_\theta \log p\bigr] + \mathbb{E}_p[\nabla_\theta \log p],
\end{equation*}
where the last term vanishes by the calculation above, giving \eqref{eq:score_decomp}.
\end{proof}

\Cref{lem:score_decomp} shows that the OPD gradient is essentially a \emph{REINFORCE estimator} with $u_t$ playing the role of the reward.
This structural property is what makes it high-variance.

\begin{theorem}[OPD gradient variance lower bound]
\label{thm:opd_variance}
Under \Cref{ass:regularity}, the conditional variance (conditioned on the prompt $x$ and prefix $\hat{y}_{<t}$) of the single-sample OPD gradient satisfies
\begin{equation}
  \mathrm{Var}\bigl[g_{\mathrm{OPD}}(\theta; \hat{y}_t)\bigr]
  \;=\;
  \mathbb{E}_{p}\!\Bigl[ u_t^2 \, \|\nabla_\theta \log p\|_2^2 \Bigr]
  \;-\;
  \bigl\| \bar{g}_{\mathrm{OPD}}(\theta) \bigr\|_2^2.
  \label{eq:opd_var_exact}
\end{equation}
Moreover, near the optimum where $p \to q$, the variance is bounded below by
\begin{equation}
  \mathrm{Var}\bigl[g_{\mathrm{OPD}}\bigr]
  \;\ge\;
  \mathrm{Var}_p(u_t)\;\cdot\; \mathcal{F}_{\min}(\theta) \;-\; o(1)\;\;\text{as } p \to q,
  \label{eq:opd_var_collapse}
\end{equation}
where $\mathcal{F}_{\min}(\theta) \triangleq \lambda_{\min}\!\bigl(\mathbb{E}_p[\nabla_\theta \log p\, \nabla_\theta \log p^\top]\bigr)$ is the minimum eigenvalue of the Fisher information matrix.
In particular, $\mathrm{Var}[g_{\mathrm{OPD}}] = \Omega(\mathrm{Var}_p(u_t))$ does not vanish as the loss approaches zero.
\end{theorem}

\begin{proof}
The exact identity \eqref{eq:opd_var_exact} follows from the definition of variance applied to the score-weighted estimator in \Cref{lem:score_decomp}:
\begin{align*}
  \mathrm{Var}[g_{\mathrm{OPD}}]
  &= \mathbb{E}_p\bigl[\|u_t \nabla_\theta \log p\|_2^2\bigr]
   - \bigl\|\mathbb{E}_p[u_t \nabla_\theta \log p]\bigr\|_2^2 \\
  &= \mathbb{E}_p\bigl[u_t^2 \,\|\nabla_\theta \log p\|_2^2\bigr]
   - \|\bar{g}_{\mathrm{OPD}}\|_2^2,
\end{align*}
which is \eqref{eq:opd_var_exact}.

For the lower bound \eqref{eq:opd_var_collapse}, we decompose $u_t = \bar{u} + (u_t - \bar{u})$ where $\bar{u} \triangleq \mathbb{E}_p[u_t] = D_{\mathrm{KL}}(p \| q)$.
Substituting into \eqref{eq:opd_var_exact} and applying the Cauchy--Schwarz inequality to bound the cross term,
\begin{align*}
  \mathbb{E}_p\bigl[u_t^2 \|\nabla_\theta \log p\|_2^2\bigr]
  &= \bar{u}^2 \,\mathbb{E}_p[\|\nabla_\theta \log p\|_2^2]
     + \mathbb{E}_p\bigl[(u_t - \bar{u})^2 \|\nabla_\theta \log p\|_2^2\bigr] \\
  &\quad+ 2\bar{u}\,\mathbb{E}_p\bigl[(u_t - \bar{u})\|\nabla_\theta \log p\|_2^2\bigr].
\end{align*}
By the Cauchy--Schwarz / Rayleigh quotient argument, the middle term satisfies
\begin{equation*}
  \mathbb{E}_p\bigl[(u_t - \bar{u})^2 \|\nabla_\theta \log p\|_2^2\bigr]
  \;\ge\;
  \mathrm{Var}_p(u_t) \cdot \mathcal{F}_{\min}(\theta),
\end{equation*}
since the covariance matrix of $\nabla_\theta \log p$ is precisely the Fisher information matrix and its minimum eigenvalue lower-bounds any positive-definite quadratic form averaged over $p$.

As $p \to q$ in total variation, the first and third terms above are $O(\bar{u}^2) + O(\bar{u})$, both of which vanish (since $\bar{u} = D_{\mathrm{KL}}(p \| q) \to 0$).
Meanwhile $\|\bar{g}_{\mathrm{OPD}}\|_2^2 = O(\bar{u}^2)$, also $o(1)$.
Combining, $\mathrm{Var}[g_{\mathrm{OPD}}] \ge \mathrm{Var}_p(u_t)\cdot \mathcal{F}_{\min}(\theta) - o(1)$, which is \eqref{eq:opd_var_collapse}.
Note that $\mathrm{Var}_p(u_t) = \Theta(\delta)$ vanishes at the same rate as $\delta$, but crucially the \emph{signal} $\|\bar{g}_{\mathrm{OPD}}\|_2^2 = O(\delta^2)$ vanishes \emph{faster}, leading to the SNR collapse formalized in \Cref{thm:snr_collapse}.
\end{proof}

\subsection{OPRD Gradient Is Deterministic}

\begin{theorem}[OPRD has zero conditional variance]
\label{thm:oprd_zero_variance}
Under \Cref{ass:regularity}, the OPRD per-position gradient satisfies
\begin{equation}
  \mathrm{Var}\bigl[g_{\mathrm{OPRD}}(\theta) \,\big|\, x, \hat{y}_{<t}\bigr] \;=\; 0,
  \label{eq:oprd_zero_var}
\end{equation}
and is given in closed form by
\begin{equation}
  g_{\mathrm{OPRD}}(\theta)
  \;=\;
  \frac{2}{d}\,
  \bigl(\nabla_\theta h_{\theta,t}^{(L)}\bigr)^{\!\top}\,
  \bigl( h_{\theta,t}^{(L)} - h_{T,t}^{(L)} \bigr).
  \label{eq:oprd_grad_closed_form}
\end{equation}
\end{theorem}

\begin{proof}
Conditioned on the prompt $x$ and the prefix $\hat{y}_{<t}$, both $h_{\theta,t}^{(L)}$ (a deterministic function of $\theta$, $x$, $\hat{y}_{<t}$) and $h_{T,t}^{(L)}$ (which has $\mathrm{sg}(\cdot)$ applied, so it is treated as a constant in the gradient) are fixed.
Hence the OPRD loss
\begin{equation*}
  \ell_{\mathrm{OPRD}} \triangleq \tfrac{1}{d}\,\| h_{\theta,t}^{(L)} - h_{T,t}^{(L)} \|_2^2
\end{equation*}
is a deterministic function of $\theta$ given the conditioning.
Therefore its gradient is also deterministic, giving \eqref{eq:oprd_zero_var}.

The closed form \eqref{eq:oprd_grad_closed_form} follows from the chain rule applied to the squared $\ell_2$ norm:
\begin{align*}
  \nabla_\theta \tfrac{1}{d}\| h_{\theta,t}^{(L)} - h_{T,t}^{(L)} \|_2^2
  &= \tfrac{2}{d}\, J_\theta(h_{\theta,t}^{(L)})^\top \,(h_{\theta,t}^{(L)} - h_{T,t}^{(L)}),
\end{align*}
where $J_\theta(h_{\theta,t}^{(L)}) = \nabla_\theta h_{\theta,t}^{(L)} \in \mathbb{R}^{d \times \dim(\theta)}$ is the Jacobian.
By (R4), $\|J_\theta\|_{\mathrm{op}} \le J$, so $\|g_{\mathrm{OPRD}}\|_2 \le \tfrac{2J}{d}\, \|h_{\theta,t}^{(L)} - h_{T,t}^{(L)}\|_2$, confirming that $g_{\mathrm{OPRD}}$ is well-defined and bounded.
\end{proof}

\begin{corollary}[Variance gap]
\label{cor:variance_gap}
Combining \Cref{thm:opd_variance,thm:oprd_zero_variance}, the conditional variance gap between the two estimators is
\begin{equation}
  \mathrm{Var}\bigl[g_{\mathrm{OPD}}\bigr] - \mathrm{Var}\bigl[g_{\mathrm{OPRD}}\bigr]
  \;=\;
  \mathbb{E}_{p}\!\Bigl[ u_t^2 \, \|\nabla_\theta \log p\|_2^2 \Bigr] - \|\bar{g}_{\mathrm{OPD}}\|_2^2
  \;\ge\; 0,
  \label{eq:variance_gap_bound}
\end{equation}
with equality only in the degenerate case where $p$ is a point mass.
In particular, OPRD's gradient is \emph{always} a lower-variance estimator (under the same conditioning), and the gap grows with the magnitude of the per-token log-ratio $u_t$ and the spread of the policy.
\end{corollary}

\subsection{Signal-to-Noise Ratio Collapse of Sampled-Token OPD}

We now formalize the most surprising prediction of our analysis: that the OPD signal-to-noise ratio \emph{collapses} as training progresses, while OPRD's signal-to-noise ratio remains bounded away from zero.
This explains why pure OPD stagnates in late-stage training while OPD$+$OPRD continues to improve monotonically.

\begin{definition}[Signal-to-noise ratio]
\label{def:snr}
For a stochastic gradient estimator $g$ with population mean $\bar{g} = \mathbb{E}[g]$, the signal-to-noise ratio is
\begin{equation}
  \mathrm{SNR}(g) \;\triangleq\; \frac{\|\bar{g}\|_2^2}{\mathrm{Tr}(\mathrm{Cov}[g])}.
  \label{eq:snr_def}
\end{equation}
$\mathrm{SNR}(g) \to 0$ means the gradient is dominated by noise; $\mathrm{SNR}(g) \to \infty$ means the gradient is essentially deterministic.
\end{definition}

\begin{theorem}[SNR collapse for OPD, SNR stability for OPRD]
\label{thm:snr_collapse}
Define the symmetric divergence $\delta(\theta) \triangleq D_{\mathrm{KL}}(p \| q) + D_{\mathrm{KL}}(q \| p)$.
As training drives $\delta(\theta) \to 0$,
\begin{enumerate}[leftmargin=18pt, itemsep=2pt]
  \item \textbf{(OPD)} $\mathrm{SNR}(g_{\mathrm{OPD}}) = O(\delta) \to 0$ at rate at least linear in $\delta$;
  \item \textbf{(OPRD)} $\mathrm{SNR}(g_{\mathrm{OPRD}}) = +\infty$ as long as $h_{\theta,t}^{(L)} \neq h_{T,t}^{(L)}$ (i.e., the OPRD loss has not yet converged).
\end{enumerate}
\end{theorem}

\begin{proof}
\emph{(OPD case.)}
By \Cref{lem:score_decomp}, $\|\bar{g}_{\mathrm{OPD}}\|_2^2 = \|\mathbb{E}_p[u_t \nabla \log p]\|_2^2$.
Applying Cauchy--Schwarz,
\begin{equation*}
  \|\bar{g}_{\mathrm{OPD}}\|_2^2
  \;\le\; \mathbb{E}_p[u_t^2] \cdot \mathbb{E}_p[\|\nabla \log p\|_2^2]
  \;=\; (\mathrm{Var}_p(u_t) + \bar{u}^2) \cdot \mathrm{Tr}(\mathcal{F}(\theta)),
\end{equation*}
where $\mathcal{F}(\theta)$ is the Fisher information matrix.
Since $\bar{u} = D_{\mathrm{KL}}(p \| q) \le \delta$ and $\mathrm{Var}_p(u_t) \le 2\delta + O(\delta^2)$ by a standard Pinsker-type expansion of $\log(p/q)$ around $p = q$, we have
\begin{equation*}
  \|\bar{g}_{\mathrm{OPD}}\|_2^2 \;=\; O(\delta).
\end{equation*}
Meanwhile, by \Cref{thm:opd_variance} (Eq.~\ref{eq:opd_var_collapse}), $\mathrm{Tr}(\mathrm{Cov}[g_{\mathrm{OPD}}]) \ge \mathrm{Var}_p(u_t)\cdot \mathcal{F}_{\min}(\theta) = \Theta(\delta)$.
Since the numerator is $O(\delta)$ and the denominator is $\Omega(\delta)$, we have at first glance
\begin{equation*}
  \mathrm{SNR}(g_{\mathrm{OPD}})
  \;=\; \frac{O(\delta)}{\Omega(\delta)}
  \;=\; O(1)\;\;\text{at best}.
\end{equation*}
A sharper analysis reveals that the numerator is in fact $O(\delta^2)$: by the REINFORCE structure of \Cref{lem:score_decomp}, $\|\bar{g}_{\mathrm{OPD}}\|_2 = \|\mathbb{E}_p[u_t \nabla \log p]\|_2 \le \sqrt{\mathbb{E}_p[u_t^2]}\cdot\sqrt{\mathrm{Tr}(\mathcal{F})}$, and since $\mathbb{E}_p[u_t^2] = \mathrm{Var}_p(u_t) + \bar{u}^2 = \Theta(\delta) + \Theta(\delta^2) = \Theta(\delta)$, the Cauchy--Schwarz bound gives $\|\bar{g}_{\mathrm{OPD}}\|_2^2 = O(\delta)$.
However, this upper bound is not tight: the actual signal $\bar{g}_{\mathrm{OPD}} = \mathbb{E}_p[u_t \nabla \log p]$ involves the \emph{correlation} between $u_t$ and $\nabla \log p$, which is $O(\bar{u}) = O(\delta)$ in magnitude (since $u_t - \bar{u}$ is mean-zero and contributes only through its correlation with $\nabla \log p$, which is bounded by $O(\sqrt{\delta})$).
Therefore $\|\bar{g}_{\mathrm{OPD}}\|_2^2 = O(\delta^2)$, giving
$\mathrm{SNR}(g_{\mathrm{OPD}}) = O(\delta^2)/\Omega(\delta) = O(\delta) \to 0$.

\emph{(OPRD case.)}
By \Cref{thm:oprd_zero_variance}, $\mathrm{Cov}[g_{\mathrm{OPRD}}] = 0$ identically, so $\mathrm{Tr}(\mathrm{Cov}[g_{\mathrm{OPRD}}]) = 0$.
As long as $g_{\mathrm{OPRD}} \neq 0$ (equivalently, $h_{\theta,t}^{(L)} \neq h_{T,t}^{(L)}$), the ratio in \eqref{eq:snr_def} is $\|g_{\mathrm{OPRD}}\|_2^2 / 0 = +\infty$ in the extended-real sense, meaning the gradient signal is completely noise-free.
\end{proof}

\begin{remark}[Interpretation: late-stage stagnation of pure OPD]
\label{rem:snr_collapse_dynamics}
\Cref{thm:snr_collapse} predicts the following two-phase training dynamics for sampled-token OPD:
\begin{itemize}[leftmargin=14pt, itemsep=0pt]
  \item \textbf{Phase 1 (effective learning).} Initially $\delta(\theta)$ is large, so $\|\bar{g}_{\mathrm{OPD}}\|_2$ dominates the noise.
  The student improves rapidly along $-\bar{g}_{\mathrm{OPD}}$.
  \item \textbf{Phase 2 (stagnation).} As $\delta(\theta) \to 0$, $\mathrm{SNR}(g_{\mathrm{OPD}}) \to 0$ by \Cref{thm:snr_collapse}.
  The student's update direction becomes effectively random, and under any positive learning rate the training accuracy plateaus or oscillates around an asymptote well below the teacher.
\end{itemize}
By contrast, OPRD's SNR remains infinite throughout training (until convergence in hidden space), so the descent direction is always informative.
This is exactly the empirical pattern we observe in \S\ref{sec:experiments}: pure OPD plateaus or oscillates several points below the teacher, while OPD~$+$~$\mu \cdot$~OPRD (and OPRD on its own) improves monotonically.
\end{remark}

\paragraph{Sub-summary (Perspective 1).}
The theorems above formalize two claims that explain OPRD's variance advantage:
(1)~OPRD's gradient is exactly deterministic and has zero conditional variance (\Cref{thm:oprd_zero_variance});
(2)~OPD's gradient signal-to-noise ratio collapses to zero as the loss approaches its minimum (\Cref{thm:snr_collapse}), causing the late-stage stagnation we observe empirically.

\subsection{Formal Results for \texorpdfstring{\Cref{thm:bottleneck}}{Theorem 2}: LM-Head Information Bottleneck}
\label{sec:oprd_theorems_bottleneck}

We now make precise the two claims of \Cref{thm:bottleneck}: the null-direction identity \eqref{eq:bottleneck_null} and the spectral gap \eqref{eq:bottleneck_gap}.
Fix a single response position $t$ and write $z_\theta\triangleq W_{\mathrm{head}}\,h_\theta\in\mathbb{R}^{|\mathcal{V}|}$ and $z_T\triangleq W_{\mathrm{head}}\,h_T$ for the corresponding logit vectors;
let $\sigma:\mathbb{R}^{|\mathcal{V}|}\!\to\!\Delta^{|\mathcal{V}|-1}$ denote the softmax map and $\mathbf{1}\in\mathbb{R}^{|\mathcal{V}|}$ the all-ones vector.
Let $W_{\mathrm{head}}=U\Sigma V^\top$ be the (thin) SVD, with $V=[v_1,\ldots,v_d]\in\mathbb{R}^{d\times d}$ orthonormal, $\Sigma=\mathrm{diag}(\sigma_1\!\ge\!\cdots\!\ge\!\sigma_d\!\ge\!0)$, and $U\in\mathbb{R}^{|\mathcal{V}|\times d}$ having orthonormal columns.
We assume throughout that $W_{\mathrm{head}}$ has full column rank ($\sigma_d>0$), which holds for any production LLM.

By an \emph{output-space OPD loss} we mean any $\ell_{\mathrm{out}}\ge 0$ that is a fixed function of the two output distributions $\sigma(z_\theta)$ and $\sigma(z_T)$ and that vanishes whenever $\sigma(z_\theta)=\sigma(z_T)$; this includes the sampled-token estimator, the top-$k$ truncated reverse KL, and the full-vocabulary reverse KL (\S\ref{sec:opd_variants}).

\begin{definition}[Effective null space of the LM head]
\label{def:nw_appendix}
Define
\begin{equation}
  \mathcal{N}_{W} \;\triangleq\; \bigl\{\Delta h\in\mathbb{R}^d : W_{\mathrm{head}}\,\Delta h \in \mathrm{span}\{\mathbf{1}\}\bigr\}
                  \;=\; W_{\mathrm{head}}^{-1}\!\bigl(\mathrm{span}\{\mathbf{1}\}\bigr).
  \label{eq:nw_appendix_def}
\end{equation}
\end{definition}

\begin{lemma}[Softmax kernel]
\label{lem:softmax_kernel_app}
For any $z\in\mathbb{R}^{|\mathcal{V}|}$ and any $c\in\mathbb{R}$, $\sigma(z+c\mathbf{1})=\sigma(z)$.
Conversely, $\sigma(z)=\sigma(z')$ implies $z'-z\in\mathrm{span}\{\mathbf{1}\}$.
\end{lemma}

\begin{proof}
For the forward direction, the $i$-th coordinate of $\sigma(z+c\mathbf{1})$ is $\frac{e^{z_i+c}}{\sum_j e^{z_j+c}}=\frac{e^c e^{z_i}}{e^c\sum_j e^{z_j}}=\sigma(z)_i$.
For the converse, $\sigma(z)=\sigma(z')$ implies $z_i-z_j=z'_i-z'_j$ for all $i,j$ (by taking logs of coordinate ratios), so $z'-z$ is a constant vector.
\end{proof}

\begin{theorem}[Null-direction identity; formal version of \eqref{eq:bottleneck_null}]
\label{thm:bottleneck_null_appendix}
For any output-space OPD loss $\ell_{\mathrm{out}}$ as above,
\begin{equation}
  h_\theta-h_T \in \mathcal{N}_W \quad\Longrightarrow\quad \ell_{\mathrm{out}}(h_\theta,h_T)=0.
\end{equation}
\end{theorem}

\begin{proof}
If $h_\theta-h_T\in\mathcal{N}_W$, then by \eqref{eq:nw_appendix_def} there exists $c\in\mathbb{R}$ with $W_{\mathrm{head}}(h_\theta-h_T)=c\mathbf{1}$, i.e.\ $z_\theta=z_T+c\mathbf{1}$.
By \Cref{lem:softmax_kernel_app}, $\sigma(z_\theta)=\sigma(z_T)$.
Since $\ell_{\mathrm{out}}$ vanishes whenever the two output distributions coincide, $\ell_{\mathrm{out}}(h_\theta,h_T)=0$.
\end{proof}

To formalize \eqref{eq:bottleneck_gap} we need a local Lipschitz upper bound on any output-space loss in terms of $\|z_\theta-z_T\|_2$.
Such a bound holds for every standard $\ell_{\mathrm{out}}$ under mild regularity (e.g., logits bounded in a compact set), with a constant depending only on $\ell_{\mathrm{out}}$ and the logit range:

\begin{lemma}[Local Lipschitzness of output-space losses]
\label{lem:lip_output_loss}
Let $\ell_{\mathrm{out}}$ be the sampled-token, top-$k$, or full-vocabulary reverse KL.
On any compact logit region $\mathcal{Z}\subset\mathbb{R}^{|\mathcal{V}|}$, there exists $C_{\ell}<\infty$ such that
\begin{equation}
  \ell_{\mathrm{out}}(h_\theta,h_T) \;\le\; C_{\ell}\,\|z_\theta-z_T - c^{\!*}\mathbf{1}\|_2^2
  \quad\text{for all}\quad z_\theta,z_T\in\mathcal{Z},
  \label{eq:lip_output}
\end{equation}
where $c^{\!*}=\tfrac{1}{|\mathcal{V}|}\mathbf{1}^\top(z_\theta-z_T)$ projects out the softmax-invariant direction.
\end{lemma}

\begin{proof}[Proof sketch]
The reverse KL $D_{\mathrm{KL}}(\sigma(z_T)\|\sigma(z_\theta))$ has gradient and Hessian in $z_\theta$ that are continuous in $z_\theta$ and vanish at $z_\theta=z_T+c^{\!*}\mathbf{1}$; on a compact $\mathcal{Z}$ its Hessian is operator-norm bounded.
A second-order Taylor expansion in $z_\theta-z_T$ around the additive-invariance optimum then yields \eqref{eq:lip_output} with $C_\ell$ proportional to half the Hessian's operator-norm bound on $\mathcal{Z}$.
The sampled-token and top-$k$ estimators are pointwise convex combinations of the full-vocabulary log-ratios and inherit the same upper bound (up to a constant).
\end{proof}

\begin{theorem}[Spectral gap; formal version of \eqref{eq:bottleneck_gap}]
\label{thm:bottleneck_gap_appendix}
Under the setup of this section and the bound \eqref{eq:lip_output}, for any $\alpha\!\in\!\mathbb{R}\setminus\{0\}$ and the bottom right-singular vector $v_d$ with $\|v_d\|_2=1$,
\begin{equation}
  \frac{\|h_\theta-h_T\|_2^2}{\ell_{\mathrm{out}}(h_\theta,h_T)}
  \;\ge\;
  \frac{1}{C_\ell}\,\Bigl(\frac{\sigma_1}{\sigma_d}\Bigr)^{\!2}
  \quad\text{when}\quad h_\theta-h_T=\alpha v_d.
  \label{eq:bottleneck_gap_app}
\end{equation}
In particular, holding $\ell_{\mathrm{out}}$ fixed, hidden-state perturbations along $v_d$ can grow $\sigma_1/\sigma_d$ times \emph{larger in $\ell_2$ norm} than perturbations along the top singular direction $v_1$.
\end{theorem}

\begin{proof}
Take $\Delta h\!=\!\alpha v_d$.
Then $\|\Delta h\|_2^2=\alpha^2$ and, by the SVD, $W_{\mathrm{head}}\Delta h=\alpha\sigma_d u_d$ where $u_d$ is the corresponding left-singular vector with $\|u_d\|_2=1$.
Since $u_d\perp \mathbf{1}$ generically (or after subtracting its component along $\mathbf{1}$ via the projector in \eqref{eq:lip_output}), the residual after removing the additive-invariance direction satisfies
$\|z_\theta-z_T - c^{\!*}\mathbf{1}\|_2 \le \|W_{\mathrm{head}}\Delta h\|_2 = \alpha\sigma_d$.
By \Cref{lem:lip_output_loss},
$\ell_{\mathrm{out}} \le C_\ell\,\alpha^2\sigma_d^2$, i.e.\ $\alpha^2 \ge \ell_{\mathrm{out}}/(C_\ell\sigma_d^2)$.
Therefore
\[
  \frac{\|h_\theta-h_T\|_2^2}{\ell_{\mathrm{out}}}
  \;=\;\frac{\alpha^2}{\ell_{\mathrm{out}}}
  \;\ge\;\frac{1}{C_\ell\sigma_d^2}.
\]
For comparison, the analogous bound along $v_1$ is $\|h_\theta-h_T\|_2^2/\ell_{\mathrm{out}}\le 1/(c_\ell\sigma_1^2)$ for the lower Lipschitz constant $c_\ell$ of $\ell_{\mathrm{out}}$, so the ratio between the two directions scales as $(\sigma_1/\sigma_d)^2$ up to constants determined by $\ell_{\mathrm{out}}$, recovering \eqref{eq:bottleneck_gap_app} after absorbing constants into $C_\ell$.
\end{proof}

\begin{remark}[Intermediate layers]
\label{rem:intermediate_layers_app}
\Cref{thm:bottleneck_null_appendix,thm:bottleneck_gap_appendix} concern only the last-layer hidden state, because any output-space $\ell_{\mathrm{out}}$ is computed solely from $W_{\mathrm{head}}\,h^{(L)}$ and therefore has no functional dependence on intermediate states $h^{(l)}$ for $l\!<\!L$.
For any $l\!<\!L$, an arbitrary perturbation of $h^{(l)}$ that leaves $h^{(L)}$ unchanged (e.g., a perturbation in the kernel of the residual stack from layer $l$ onwards) yields $\ell_{\mathrm{out}}=0$ for every output-space objective.
OPRD~\eqref{eq:oprd_obj} with $\mathcal{L}_{\mathrm{layer}}\ni l$ directly penalizes $\|h_{\theta,t}^{(l)}-h_{T,t}^{(l)}\|_2^2$ at that layer and is therefore the only mechanism considered in this paper that can constrain intermediate hidden states.
\end{remark}

\paragraph{Sub-summary (\Cref{thm:bottleneck}).}
\Cref{thm:bottleneck_null_appendix} formalizes \eqref{eq:bottleneck_null}: every output-space distillation objective treats the entire affine subspace $\mathcal{N}_W$ as invisible, regardless of how much it inspects the output distribution.
\Cref{thm:bottleneck_gap_appendix} formalizes \eqref{eq:bottleneck_gap}: the LM head's singular-value spread $\sigma_1/\sigma_d$ amplifies hidden-state deviations along $v_d$ by a $(\sigma_1/\sigma_d)^2$ factor for the same output-space loss budget, empirically $10^6\!\sim\!10^8\times$ for production LLMs.
\Cref{rem:intermediate_layers_app} extends both observations to intermediate layers.
OPRD~\eqref{eq:oprd_obj} penalizes exactly the directions and layers that output-space OPD cannot.


\paragraph{Overall summary of \S\ref{sec:oprd_theorems}.}
The theorems above formalize \Cref{thm:variance} (gradient variance and SNR) in one-to-one correspondence with the intuitive claims of \S\ref{sec:oprd_theory}: they explain why OPRD provides a more reliable optimization signal than sampled-token OPD, especially in late-stage training, and why adding OPRD to OPD strictly improves the SGD convergence bound without introducing additional noise.
These guarantees apply under mild regularity conditions that hold for any standard LLM, supporting our empirical observation that combining OPD with OPRD yields a stronger and more stable result than OPD alone.

\subsection{Formal Results for \texorpdfstring{\Cref{thm:bridge_optimality}}{Theorem 3}: Optimality of the Low-Rank Bridge}
\label{sec:bridge_theory_appendix}

We now formalize the claim that the PCA-based low-rank bridge is optimal in an information-theoretic sense and that the bridge rank $r$ controls a bias--variance trade-off.

\subsubsection{Setup}

Consider a single layer pair $(l_S, l_T)$ (we drop layer indices for clarity).
Let $h_T \in \mathbb{R}^{d_T}$ denote the teacher hidden state at a response position, drawn from a distribution with mean $\mu_T$ and covariance $\Sigma_T \succ 0$.
Denote the eigendecomposition $\Sigma_T = V \Lambda V^\top$, where $\Lambda = \mathrm{diag}(\lambda_1, \ldots, \lambda_{d_T})$ with $\lambda_1 \geq \cdots \geq \lambda_{d_T} > 0$, and $V = [v_1, \ldots, v_{d_T}]$ are the corresponding eigenvectors (principal directions).

Let $h_S \in \mathbb{R}^{d_S}$ denote the student hidden state at the same position.
We model the cross-model alignment via the \emph{per-direction correlation}: after projecting both models into the teacher's principal coordinate system, the correlation between teacher component $i$ and the best linear prediction from the student is $\rho_i \in [0, 1]$.
Formally, let $z_T = V^\top (h_T - \mu_T) \in \mathbb{R}^{d_T}$ be the teacher's whitened coordinates.
Then $\rho_i^2 \triangleq R^2(z_{T,i} \mid h_S)$, the coefficient of determination of the best linear predictor of the $i$-th teacher component from the student.

\begin{assumption}[Monotone alignment decay]
\label{ass:alignment_decay}
The per-direction alignment $\rho_i$ is non-increasing in $i$: $\rho_1 \geq \rho_2 \geq \cdots \geq \rho_{d_T}$.
That is, the principal (high-variance) directions of the teacher are better predicted by the student than the minor (low-variance) directions.
\end{assumption}

This assumption is empirically verified in our experiments: the rank-cosine curve (\Cref{fig:rank_cosine}) shows that alignment concentrates in the top principal components and degrades monotonically as rank increases.

\subsubsection{Rate--Distortion Optimality}

\begin{theorem}[PCA bridge is rate--distortion optimal]
\label{thm:rate_distortion}
Assume $h_T \sim \mathcal{N}(\mu_T, \Sigma_T)$.
Among all rank-$r$ linear encoders $P \in \mathbb{R}^{r \times d_T}$, the PCA projection $P^* = [v_1, \ldots, v_r]^\top$ minimizes the reconstruction distortion:
\[
  P^* = \arg\min_{\substack{P \in \mathbb{R}^{r \times d_T} \\ \mathrm{rank}(P) = r}} \; \mathbb{E}\left[\|h_T - \mu_T - P^\top P (h_T - \mu_T)\|^2\right].
\]
The minimum distortion is $D^*(r) = \sum_{i=r+1}^{d_T} \lambda_i$.
\end{theorem}

\begin{proof}
This is a direct consequence of the Eckart--Young--Mirsky theorem applied to the centered data matrix, or equivalently, of the classical result that PCA minimizes mean-squared reconstruction error among all linear rank-$r$ projections.
For Gaussian sources, this coincides with the rate--distortion function under squared-error distortion at rate $R = \frac{1}{2}\sum_{i=1}^r \log(\lambda_i / \theta)$ (water-filling), where $\theta$ is chosen so that exactly $r$ components are active.
The optimal encoder projects onto the top-$r$ eigenvectors of $\Sigma_T$, which is precisely our PCA basis $P_T$.
\end{proof}

\noindent
\textbf{Interpretation.}
The PCA bridge is not merely a convenient heuristic; it is the information-theoretically optimal linear channel at capacity $r$.
Any other rank-$r$ projection would discard more teacher information.

\subsubsection{Bias--Variance Decomposition of Distillation Error}

We now analyze the \emph{distillation} error (not reconstruction error), which accounts for the student's ability to match the projected teacher.

\begin{definition}[Distillation error at rank $r$]
\label{def:distillation_error}
Given the PCA bridge $P_T = [v_1, \ldots, v_r]^\top$ and an optimal student projector $P_S$ (trained to minimize MSE in the $r$-dimensional subspace), the expected per-position distillation error is:
\[
  \mathcal{E}(r) \triangleq \mathbb{E}\left[\|P_S h_S - P_T (h_T - \mu_T)\|^2\right].
\]
\end{definition}

\begin{theorem}[Bias--variance decomposition]
\label{thm:bias_variance}
Under \Cref{ass:alignment_decay} and assuming $P_S$ is the population-optimal linear map from $h_S$ to the $r$-dimensional bridge space, the total distillation error (including the contribution from discarded dimensions) decomposes as:
\begin{equation}
  \mathcal{E}_{\mathrm{total}}(r)
  = \underbrace{\sum_{i > r} \lambda_i \, \rho_i^2}_{\text{Bias } B(r)}
  \;+\; \underbrace{\sum_{i=1}^{r} \lambda_i \, (1 - \rho_i^2)}_{\text{Variance } V(r)},
  \label{eq:bias_variance_formal}
\end{equation}
where:
\begin{itemize}[leftmargin=14pt, itemsep=1pt]
    \item $B(r) = \sum_{i>r} \lambda_i \rho_i^2$ is the \textbf{bias}: the aligned signal in directions $i > r$ that the bridge discards. This term decreases as $r$ increases.
    \item $V(r) = \sum_{i=1}^r \lambda_i (1 - \rho_i^2)$ is the \textbf{variance}: the misaligned noise in directions $i \leq r$ that the bridge transmits but the student cannot match. This term increases as $r$ increases.
\end{itemize}
Under \Cref{ass:alignment_decay}, $\mathcal{E}_{\mathrm{total}}(r)$ is quasi-convex in $r$, and the minimizer $r^* = \arg\min_r \mathcal{E}_{\mathrm{total}}(r)$ satisfies:
\[
  \lambda_{r^*} \rho_{r^*}^2 \approx \lambda_{r^*} (1 - \rho_{r^*}^2),
  \quad\text{i.e.,}\quad
  \rho_{r^*}^2 \approx \tfrac{1}{2}.
\]
The optimal rank is the direction at which alignment transitions from ``mostly signal'' to ``mostly noise.''
\end{theorem}

\begin{proof}
We decompose the total error into contributions from included and excluded directions.

\textbf{Step 1: Error from included directions ($i \leq r$).}
In the $r$-dimensional bridge space, the $i$-th component of the teacher target is $z_{T,i} = v_i^\top (h_T - \mu_T)$, which has variance $\lambda_i$.
The optimal linear predictor from $h_S$ achieves $R^2 = \rho_i^2$, so the residual variance (prediction error) in direction $i$ is:
\[
  \mathbb{E}[(P_S h_S)_i - z_{T,i})^2] = \lambda_i (1 - \rho_i^2).
\]
Summing over included directions gives $V(r) = \sum_{i=1}^r \lambda_i (1 - \rho_i^2)$.

\textbf{Step 2: Error from excluded directions ($i > r$).}
Directions $i > r$ are not transmitted through the bridge, so the student receives no supervision along them.
However, these directions contain aligned signal (the student \emph{could} have matched them if they were included).
The ``lost opportunity'' is the signal that was discardable: $\lambda_i \rho_i^2$ per direction.
Summing gives $B(r) = \sum_{i>r} \lambda_i \rho_i^2$.

\textbf{Step 3: Optimality condition.}
The marginal effect of including direction $r$ is:
\[
  \Delta(r) = \mathcal{E}_{\mathrm{total}}(r) - \mathcal{E}_{\mathrm{total}}(r-1)
  = \lambda_r(1 - \rho_r^2) - \lambda_r \rho_r^2
  = \lambda_r(1 - 2\rho_r^2).
\]
This is negative (including direction $r$ helps) when $\rho_r^2 > 1/2$ and positive (including direction $r$ hurts) when $\rho_r^2 < 1/2$.
Under \Cref{ass:alignment_decay}, $\rho_i^2$ is non-increasing, so $\Delta(r)$ transitions from negative to positive exactly once, establishing quasi-convexity and the optimality condition $\rho_{r^*}^2 \approx 1/2$.
\end{proof}

\noindent
\textbf{Connection to the low-pass filter analogy.}
The bias--variance decomposition provides a precise formulation of the low-pass filter intuition from \S\ref{sec:design_philosophy}:
\begin{itemize}[leftmargin=14pt, itemsep=1pt]
    \item The ``cutoff frequency'' of the filter is $r^*$: directions below $r^*$ are ``passed'' (high $\rho_i$, mostly signal), directions above $r^*$ are ``stopped'' (low $\rho_i$, mostly noise).
    \item The ``passband ripple'' is $V(r^*)$: even within the passed directions, some noise leaks through.
    \item The ``stopband leakage'' is $B(r^*)$: even among the stopped directions, some signal is lost.
    \item The optimal filter (optimal $r^*$) minimizes the sum of ripple and leakage.
\end{itemize}

\noindent
\textbf{Empirical verification.}
In our Qwen3-4B $\to$ Qwen3-1.7B experiments, the rank-cosine curve (\Cref{fig:rank_cosine}) peaks at $r\!=\!8$ and declines thereafter, consistent with the prediction that $r^* \approx 8$ is the point where $\rho_i^2$ crosses $1/2$.
The rapid spectral decay of $\Sigma_T$ (verifiable from the PCA explained-variance ratio) ensures that both $B(r^*)$ and $V(r^*)$ are small at this operating point.







\end{document}